\documentclass[conference]{IEEEtran}
\IEEEoverridecommandlockouts
\usepackage{cite}
\usepackage{amsmath,amssymb,amsfonts}
\usepackage{algorithmic}
\usepackage{graphicx}
\usepackage{textcomp}
\usepackage{xcolor}

\usepackage{booktabs}
\usepackage{epsfig}
\usepackage{algorithm}
\usepackage{ntheorem}
\usepackage{multirow}

\usepackage{hyperref}
\usepackage{subfigure}

\theoremstyle{plain}
\newtheorem{theorem}{Theorem}[section]

\theoremstyle{definition}
\newtheorem{definition}[theorem]{Definition}
\newtheorem{assumption}[theorem]{Assumption}
\theoremstyle{remark}

\newtheorem*{proof}{\bf Proof}

\def\BibTeX{{\rm B\kern-.05em{\sc i\kern-.025em b}\kern-.08em
    T\kern-.1667em\lower.7ex\hbox{E}\kern-.125emX}}
\begin{document}

\title{Learning Subspace-Preserving Sparse Attention Graphs from Heterogeneous Multiview Data}

\author{\IEEEauthorblockN{1\textsuperscript{st} Jie Chen}
\IEEEauthorblockA{\textit{College of Computer Science} \\
\textit{Sichuan University}\\
Chengdu, China \\
chenjie2010@scu.edu.cn}
\and
\IEEEauthorblockN{2\textsuperscript{nd} Yuanbiao Gou}
\IEEEauthorblockA{\textit{College of Computer Science} \\
\textit{Sichuan University}\\
Chengdu, China \\
gouyuanbiao@gmail.com}
\and
\IEEEauthorblockN{3\textsuperscript{rd} Chuanbin Liu}
\IEEEauthorblockA{\textit{School of Economics and Management} \\ \textit{China University of Petroleum (Beijing)}\\
Beijing, China \\
liuchuanbin@cutech.edu.cn}
\and
\IEEEauthorblockN{4\textsuperscript{th} Zhu~Wang}
\IEEEauthorblockA{\textit{Law School} \\
\textit{Sichuan University}\\
Chengdu, China \\
wangzhu@@scu.edu.cn}
\and
\IEEEauthorblockN{5\textsuperscript{th} Xi Peng}
\IEEEauthorblockA{\textit{ School of Artificial Intelligence} \\
\textit{Sichuan University}\\
Chengdu, China \\
pengx.gm@gmail.com}
}

\maketitle

\begin{abstract}
The high-dimensional features extracted from large-scale unlabeled data via various pretrained models with diverse architectures are referred to as heterogeneous multiview data. Most existing unsupervised transfer learning methods fail to faithfully recover intrinsic subspace structures when exploiting complementary information across multiple views. Therefore, a fundamental challenge involves constructing sparse similarity graphs that preserve these underlying subspace structures for achieving semantic alignment across heterogeneous views. In this paper, we propose a sparse attention graph learning (SAGL) method that learns subspace-preserving sparse attention graphs from heterogeneous multiview data. Specifically, we introduce a bilinear attention factorization scheme to capture asymmetric similarities among the high-dimensional features, which breaks the symmetry bottleneck that is inherent in the traditional representation learning techniques. A dynamic sparsity gating mechanism then predicts a feature-specific compression factor for adaptively controlling the topological contributions of neighbors. Furthermore, we employ a structured sparse projection via $\alpha$-entmax to generate subspace-preserving sparse attention graphs for individual views. SAGL leverages these view-specific graphs to conduct sparse information aggregation, yielding discriminative representations for multiview learning tasks. In addition, we provide a rigorous theoretical analysis that bridges differentiable sparse attention and probability simplex constraints. Extensive experiments conducted on multiple benchmark datasets demonstrate that SAGL consistently outperforms the state-of-the-art unsupervised transfer learning approaches.
\end{abstract}


\section{Introduction}
\label{sec:Introduction}
By leveraging transfer learning with pretrained models possessing diverse architectures, high-dimensional features can be obtained from large-scale unlabeled data, yielding heterogeneous multiview data \cite{Alkin2025MIM, Chu2024IC}. These features typically lie close to low-dimensional structures corresponding to latent semantic categories \cite{Brbic20LRSSC}. Unsupervised knowledge transfer is aimed at bridging the discrepancy between such pretrained features and the intrinsic subspace structures for multiview learning.

Recent advances in unsupervised knowledge transfer techniques have fundamentally reshaped the representation learning process for downstream tasks \cite{Ng2025GRA, Wang2025ELM, Chen2024UTRL}. By transferring knowledge from pretrained models, it has become possible to extract transferable features from large-scale unlabeled visual data across different domains. Numerous unsupervised transfer learning approaches have been developed to enforce the semantic alignment of features across multiple views, ensuring that their representations remain discriminative without the need for manual annotations \cite{Chen2026MSRL, Liu2025HGS, Gadetsky2024UT}. By taking advantage of the extensive semantic priors of foundation vision-language models, these unsupervised transfer methods consistently produce results that are competitive with those of the traditional supervised fine-tuning approaches \cite{Fu2026DTL, Shang2025VPT}. When samples approximately lie in a union of low-dimensional subspaces, the samples derived from the same subspace tend to exhibit stronger intrinsic correlations with each other than with samples obtained from different subspaces. However, most existing unsupervised transfer learning approaches fail to effectively capture such subspace structures when completing unsupervised knowledge transfer tasks.

Multiview self-representation learning has achieved remarkable progress in terms of capturing the subspace structures that are embedded in multiview data \cite{Wang2026DSS, Li2026PLS, Jiang2025CSF, Chen2025OAGL, Zhang2025AAGR, Fei2025DMVC}. Most existing self-representation learning-based methods can be broadly categorized into graph representation learning \cite{Jiang2025CSF, Chen2025OAGL, Zhang2025AAGR} and self-supervised graph learning techniques \cite{Wang2026DSS, Li2026PLS, Deng2025CGC, Liu2025LMC}. These methods focus primarily on exploiting complementary information by recovering subspace structures while simultaneously enforcing a shared consensus graph across multiple views. However, when $\ell_0$-norm or low-rank relaxations are integrated into shallow learning architectures, graph representation learning approaches require computationally prohibitive iterative solvers \cite{Chen2025OAGL, Brbic20LRSSC, Elhamifar13SSC}. As a result, this severely restricts their scalability to large-scale unlabeled data. In addition, the traditional contrastive learning architectures adopted by self-supervised graph representation learning approaches limit their ability to capture rich semantic knowledge from heterogeneous multiview data. The $\alpha$-entmax transformation was proposed as a formal generalization of the standard softmax function for inducing structured sparsity in attention mechanisms \cite{Peters2019SSS, Martins2016SM}. The theoretical relationship between $\alpha$-entmax projection and $\ell_0$-norm sparse self-representation across multiple views deserves further investigations in heterogeneous multiview learning scenarios. Therefore, a fundamental challenge lies in constructing sparse similarity graphs that faithfully capture the intrinsic subspace structures embedded in heterogeneous multiview data.

In this paper, we propose a sparse attention graph learning (SAGL) method that learns subspace-preserving sparse attention graphs from heterogeneous multiview data. By reformulating the $\ell_0$-norm sparse self-representation as a simplex-constrained differentiable optimization, we present an adaptive sparse self-representation learning framework that performs sparse information aggregation to produce discriminative representations. Specifically, we first introduce a bilinear attention factorization scheme to capture asymmetric similarities among features. This breaks the symmetry bottleneck inherent in traditional representation learning techniques. Then, we present a dynamic sparsity gating mechanism that adaptively controls the topological contribution of neighbors to facilitate rigorous intrasubspace isolation. Furthermore, we employ a structured sparse projection via $\alpha$-entmax to generate the subspace-preserving sparse attention graphs for individual views. Compared to existing multiview graph learning methods, SAGL enables fully unsupervised transfer learning without computationally expensive iterative solvers. In addition, we provide a rigorous theoretical analysis bridging differentiable sparse attention and probability simplex constraints, demonstrating how structured sparsity enforces subspace-preserving representations. Extensive experiments demonstrate that SAGL consistently outperforms state-of-the-art unsupervised transfer learning approaches.

The key contributions of this work are summarized as follows.
\begin{itemize}
\item{By reformulating sparse self-representation under probability simplex constraints, we present a sparse self-representation framework that learns subspace-preserving sparse attention graphs from heterogeneous multiview data. These graphs are employed to perform sparse information aggregation, yielding discriminative representations for performing unsupervised transfer learning without requiring computationally expensive iterative solvers.}
\item{We introduce a bilinear attention factorization scheme to capture the asymmetric similarities among features. This scheme is enhanced by a dynamic sparsity-based gating mechanism that predicts a feature-specific compression factor to adaptively control the topological contributions of neighbors for implementing rigorous intrasubspace isolation.}
\item{We provide rigorous theoretical guarantees that bridge differentiable sparse attention and probability simplex constraints, demonstrating how structured sparsity enforces subspace-preserving representations in sparse self-representation learning.}
\end{itemize}

\section{Related Work}
\label{sec:work}

\subsection{Multiview Self-Representation Learning}
The traditional self-representation learning approaches are aimed at uncovering the subspace structures embedded in high-dimensional data \cite{Chen2025HSRC, Huang2022MVM, Brbic20LRSSC}. They typically assume that such data lie near low-dimensional subspaces corresponding to latent semantic categories \cite{Fu2022YLRT, Brbic20LRSSC}. Inspired by these classic self-representation learning approaches, multiview self-representation learning methods have been developed to exploit complementary subspace information across multiple views \cite{Wen2026PMI, Yang2026MVC, Xu2026BONN, Chen2025OAGL}. For example, Chen \textit{et al.} \cite{Chen2025OAGL} proposed a low-rank tensor graph learning approach that enforces the alignment of spectral block-diagonal representations for high-dimensional multiview data. However, the relaxations developed based on the nuclear norm require computationally prohibitive iterative solvers \cite{Fu2022YLRT, Brbic20LRSSC, Elhamifar13SSC}. Xu \textit{et al.} \cite{Xu2026BONN} introduced a self-representation learning framework (SRLF) that mitigates cross-modal discrepancies by deriving high-level semantic representations from low-level representations. However, this approach struggles to align heterogeneous feature distributions across multiple views. These studies indicate that integrating self-representation learning with view-specific structural constraints is essential for capturing the subspace structures of high-dimensional data.

\subsection{Unsupervised Knowledge Transfer Learning}
Transfer learning provides a fundamental paradigm for leveraging the semantic knowledge encoded in diverse pretrained models to enhance the generalizability of representation learning techniques. Recent advances in unsupervised knowledge transfer have been aimed at learning discriminative representations from large-scale unlabeled data without manual supervision \cite{Chen2026MSRL, Gadetsky2024UT, Liu2025ECP}. For example, Gadetsky \textit{et al.} \cite{Gadetsky2024UT} proposed an unsupervised transfer learning method that exploits transferable features to learn linearly separable representations without manual annotation. Chen \textit{et al.} \cite{Chen2026MSRL} introduced a multiview self-representation learning (MSRL) method that captures view-invariant representations by exploiting the self-representation property of features across heterogeneous views. These studies demonstrate that unsupervised knowledge transfer improves the general applicability of transfer learning in real-world scenarios without the need for manual annotation.

\section{Sparse Attention Graph Learning}
\label{sec:sscd}

\subsection{Adaptive Sparse Self-Representation Learning}
Let $\mathcal{D}_{tr} = \left\{ {\mathbf{x}_1,\mathbf{x}_2,...,\mathbf{x}_n} \right\}$ be a training dataset consisting of $n$ unlabeled samples, with each sample $\mathbf{x}_i \in {\mathbb{R}^{m}}$. By utilizing $L$ distinct pretrained models $\left\{ {{\phi _l}} \right\}_{l = 1}^L$,  a multiview dataset $\mathcal{H} = \{ \mathbf{H}^{(1)}, \mathbf{H}^{(2)}, \dots, \mathbf{H}^{(L)} \} $ is generated through transfer learning, where $ \mathbf{H}^{(l)} = \left[ {\mathbf{h}_1^{\left( l \right)},\mathbf{h}_2^{\left( l \right)},...,\mathbf{h}_n^{\left( l \right)}} \right]$ $\left( {1 \le l \le L} \right)$, and ${\mathbf{h}_i^{\left( l \right)}} \in {\mathbb{R}^{d_l}}$ denotes the feature of sample $\mathbf{x}_i$, i.e.,
\begin{equation}\label{eq:feat}
{\mathbf{h}_i^{\left( l \right)}} = {\phi _l} \left(  \mathbf{x}_i  \right).
\end{equation}

When at least two of these $L$ pretrained models adopt different pretraining objectives, $\mathcal{H}$ is referred to as heterogeneous multiview data. For simplicity, we utilize two pretrained models, namely, ${\phi _{\mu}} \left(  \cdot  \right)$ and ${\phi _{\nu}} \left(  \cdot  \right)$. Considering two linear models ${\delta^{\left( \mu \right)}}\left( { \cdot \ ; \ \cdot } \right)$ and ${\delta ^{\left( \nu \right)}}\left( { \cdot \ ; \ \cdot } \right)$, the corresponding features $\mathbf{z}_i^{(\mu)} \in {\mathbb{R}^{C}}$ and $\mathbf{z}_i^{(\nu)} \in {\mathbb{R}^{C}}$ can be produced by
\begin{equation}\label{eq:linearfeat}
\begin{split}
\mathbf{z}_i^{(\mu)} = \delta^{(\mu)}(\mathbf{h}_i^{(\mu)};\mathbf{W}_{\mu}), \qquad \mathbf{z}_i^{(\nu)} = \delta^{(\nu)}(\mathbf{h}_i^{(\nu)};\mathbf{W}_{\nu})
\end{split}
\end{equation}
where $C$ represents the number of categories, and ${\mathbf{W}_{\mu}} \in {\mathbb{R}^{{d_{\mu}}\times{C}}}$ and ${\mathbf{W}_{\nu}} \in {\mathbb{R}^{{d_{\nu}}\times{C}}}$ denote the parameters of ${\delta^{\left( i \right)}}\left( { \cdot \ ; \ \cdot } \right)$ and ${\delta ^{\left( j \right)}}\left( { \cdot \ ; \ \cdot } \right)$, respectively.

For the $l$th view, let $\mathbf{Z}^{(l)} = \left[ {\mathbf{z}_1^{(l)}, \mathbf{z}_2^{(l)}, \dots, \mathbf{z}_n^{(l)}} \right]^\top \in \mathbb{R}^{n \times C}$ denote the matrix of features. Each feature can be represented as a sparse linear combination of other features belonging to the same latent subspace. Specifically, the optimal sparse coefficient vector $\mathbf{a}_i^{\left( l \right)}$ can be obtained by solving the following $\ell_0$-norm minimization problem:
\begin{equation}\label{eq:sr1}
\widetilde{\mathbf{a}}_i^{\left( l \right)} = \mathop{\arg\min}\limits_{\mathbf{a}_i^{\left( l \right)}} \left\| {\mathbf{z}_i^{\left( l \right)} - {\left( \mathbf{Z}^{\left( l \right)} \right)^T}  \mathbf{a}_i^{\left( l \right)}} \right\|_2^2 + \lambda {\left\| \mathbf{a}_i^{\left( l \right)} \right\|_0}
\end{equation}
where $\lambda > 0$ is a trade-off parameter and ${\left\| \cdot \right\|_0}$ denotes the number of nonzero entries contained in the vector. As the optimization problem in Eq.~\eqref{eq:sr1} is nonconvex and NP-hard, it is typically relaxed to a sparsity-inducing $\ell_1$-norm minimization problem. However, relying on traditional iterative solvers for this convex relaxation is computationally prohibitive when applied to large-scale multiview data.

To integrate the $\ell_0$-norm self-representation strategy into an end-to-end differentiable learning framework, we reformulate the sparse representation problem as a simplex-projected attention scheme. Let $\Delta^{n-1} := \left\{ \mathbf{a}_i^{(l)} \mid \mathbf{1}^\top \mathbf{a}_i^{(l)} = 1, \ \mathbf{a}_i^{(l)} \ge \mathbf{0} \right\}$ denote the $(n-1)$-dimensional probability simplex. We formulate the simplex-constrained sparse self-representation problem as follows:
\begin{equation}
\label{eq:lzsimplex}
\begin{aligned}
& \min_{\mathbf{a}_i^{(l)} \in \Delta^{n-1}} \|\mathbf{a}_i^{(l)}\|_0  + \lambda {\left\|  \mathbf{z}_i^{(l)} - {\left( \mathbf{Z}^{\left( l \right)} \right)^T} \mathbf{a}_i^{(l)} \right\|_2^2} \\
& \quad \text{s.t.} \quad \mathbf{z}_i^{(l)} = \mathbf{Z}^{(l)} \mathbf{a}_i^{(l)} + \mathbf{e}_i^{(l)}
\end{aligned}
\end{equation}
where $\mathbf{e}_i^{(l)} \in \mathbb{R}^{d_l}$ denotes the sparse error term. Instead of algebraic optimization solutions, our goal is to adaptively learn subspace-preserving sparse attention graphs $\left\{ {\mathbf{A}^{(l)} \in \mathbb{R}^{n \times n}} \right\}_{l = 1}^L$ for the multiview dataset $\mathcal{H}$, where $ \mathbf{A}^{(l)} = \left\{ {\mathbf{a}_1^{\left( l \right)},\mathbf{a}_2^{\left( l \right)},...,\mathbf{a}_n^{\left( l \right)}} \right\}$, and an attention coefficient ${a_{ij}^{(l)}} \in \mathbf{A}^{(l)} $ $\left( {1 \le i, j \le n} \right)$ denotes the similarity between the features ${\mathbf{h}_i^{(l)}}$ and ${\mathbf{h}_j^{(l)}}$ in the $l$th view.

To align the heterogeneous distributions across multiple views, we present an adaptive sparse self-representation learning framework. Specifically, we formulate the process of constructing subspace-preserving sparse attention graphs as an adaptive sparse representation learning problem:
\begin{equation}\label{eq:obj}
\begin{aligned}
& \min_{\mathbf{W}_{\mu}, \mathbf{W}_{\nu}, \mathbf{A}^{(\mu)}, \mathbf{A}^{(\nu)}} \sum_{i=1}^n \mathcal{L} \left( \mathbf{p}_i^{(\mu)}, \mathbf{p}_i^{(\nu)} \right) + \lambda \sum_{l \in \{\mu, \nu\}} \Omega(\mathbf{A}^{(l)}) \\
& \text{s.t.} \  \left(\mathbf{A}^{(l)}\right)^T \mathbf{1} = \mathbf{1}, \mathbf{A}^{(l)} \ge \mathbf{0}
\end{aligned}
\end{equation}
where $\mathcal{L}$ denotes a self-supervised alignment loss, and $ \Omega \left( \cdot \right)$ denotes the structural sparsity constraint imposed on the attention graphs. The fundamental challenge encountered when minimizing the objective function in Eq.~\eqref{eq:obj} lies in efficiently optimizing the network weights to learn the subspace-preserving sparse attention graphs $\mathbf{A}^{(l)}$ while satisfying the imposed sparse structural constraints.

\subsection{Learning Subspace-Preserving Sparse Attention Graphs}
According to Eq.~\eqref{eq:lzsimplex}, subspace-preserving sparse attention graphs determine the topological neighborhoods for feature-specific aggregation. These graphs can be constructed by evaluating the similarities between the features. However, the standard attention mechanisms inevitably result in densely connected graphs. Such dense connectivity leads to heavily entangled representations across different latent subspaces. We take advantage of the underlying subspace structures embedded in high-dimensional features to learn  subspace-preserving sparse attention graphs.

\subsubsection{Bilinear Attention Factorization}
We introduce a bilinear attention factorization scheme that projects the encoded features $\mathbf{Z}^{(l)}$ into two latent semantic subspaces via two learnable projection matrices $\mathbf{U}^{(l)}, \mathbf{V}^{(l)} \in \mathbb{R}^{C \times C}$. The raw similarity matrix $\mathbf{S}^{(l)} \in \mathbb{R}^{n \times n}$ is computed as follows:
\begin{equation}\label{eq:lrs}
\mathbf{S}^{(l)} = \frac{(\mathbf{Z}^{(l)}\mathbf{U}^{(l)})(\mathbf{Z}^{(l)}\mathbf{V}^{(l)})^\top}{\sqrt{C}}
\end{equation}
where the scaling factor $\frac{1}{\sqrt{C}}$ prevents the dot products from growing excessively large in high-dimensional feature spaces. The similarity matrix $\mathbf{S}^{(l)}$ captures the global pairwise semantic structure among all the samples in the $l$th view. Unlike the standard inner products that inherently produce symmetric undirected graphs, the introduction of two distinct projection matrices $\mathbf{U}^{(l)}$ and $\mathbf{V}^{(l)}$ generates an asymmetric similarity matrix $\mathbf{S}^{(l)}$. This fundamental asymmetry effectively facilitates discriminative information aggregation among the features for self-representation learning.

\subsubsection{Dynamic Sparsity Gate}
Since the features derived from distinct latent subspaces exhibit varying degrees of topological uncertainty, applying an isotropic threshold across all features inevitably compromises the structural integrity of the learned graphs. To achieve precise intrasubspace isolation, it is imperative to dynamically adapt the sparsity level for each individual feature. To this end, we propose a dynamic sparsity gate that is formulated as a lightweight multilayer perceptron (MLP). For each feature $\mathbf{z}_i^{(l)}$, the gate predicts a feature-specific compression factor $\omega_i^{(l)} \in [0, 1]$:
\begin{equation}\label{eq:gate}
\omega_i^{(l)} = \sigma \left( \mathbf{W}_2 \text{ReLU} \left( \mathbf{W}_1 \mathbf{z}_i^{(l)} \right) \right)
\end{equation}
where $\sigma(\cdot)$ denotes the sigmoid activation function, and $\mathbf{W}_1$ and $\mathbf{W}_2$ are learnable weight matrices. The directed similarities are subsequently adapted by this adaptive gate as follows:
\begin{equation}\label{eq:sim}
\tilde{\mathbf{s}}_{i}^{(l)} = \mathbf{s}_{i}^{(l)} \cdot \left( 1 - \omega_i^{(l)} \right)
\end{equation}
where $ \mathbf{s}_{i}^{(l)}$ denotes the $i$th column of $\mathbf{S}^{(l)}$. A larger $\omega_i^{(l)}$ imposes a stronger penalty on the similarities, effectively forcing the model to concentrate on a narrower set of highly confident neighbors in the same latent subspace. As the feature-specific compression factor $\omega_i^{(l)}$ is dynamically learned during the optimization, we also present an alternative scheme to compute the directed similarities as follows:
\begin{equation}\label{eq:alsim}
\tilde{s}_{ij}^{(l)} = \frac{s_{ij}^{(l)}}{1 - \omega_i^{(l)} + \epsilon}
\end{equation}
where $\epsilon$ is a small constant, e.g., $\epsilon=10^{-6}$. This scheme explicitly breaks the structural symmetry, i.e., $\tilde{s}_{ij}^{(l)} \neq \tilde{s}_{ji}^{(l)}$. Consequently, the factor $\omega_i^{(l)}$ adaptively controls the topological contributions of the neighbors during the feature aggregation.

\subsubsection{Structured Sparse Projection}
The bilinear attention factorization scheme can capture the global pairwise semantic information among all the features. However, it inevitably assigns nonzero similarities to all feature pairs regardless of their underlying subspace affiliations. This indicates that the similarity matrix $\mathbf{S}^{(l)}$ remains densely connected, which violates the independent subspace assumption underlying the sparse self-representation learning. Structural sparsity can be enforced for the similarity matrix $\mathbf{S}^{(l)}$ by eliminating any spurious connections between features derived from different latent subspaces. Specifically, we project each column of the similarity matrix $\mathbf{S}^{(l)}$ onto the $(n-1)$-dimensional probability simplex $\Delta^{n-1}$ to obtain a subspace-preserving sparse attention graph $\mathbf{A}^{(l)}$. The standard softmax function is intrinsically incapable of producing absolute zeros. Instead of utilizing the standard softmax function, the subspace-preserving sparse attention graph $\mathbf{A}^{(l)}$ is generated by employing the $\alpha$-entmax transformation \cite{Peters2019SSS, Martins2016SM} : $\mathbf{a}_i^{(l)} = \text{entmax}_{\alpha} \left( \tilde{\mathbf{s}}_{i}^{(l)} \right)$. For any $\alpha > 1$, the entmax mapping process induces sparse solutions via a thresholding mechanism:
\begin{equation}\label{eq:entmax}
\mathbf{a}_i^{(l)} = \left[(\alpha - 1)\left(\tilde {\mathbf{s}}_i^{(l)} - \tau_i^{(l)} \mathbf{1}\right)\right]_+^{\frac{1}{\alpha - 1}}
\end{equation}
where $ \tau_i^{(l)} \in \mathbb{R}$ is a unique threshold. By setting $\alpha=1.5$, the transformation acts as a strict continuous thresholding operator. This operator forces the marginal inter-subspace similarities to exactly zero. By adaptively adjusting the factor $\omega_i^{(l)}$, the projection assigns stronger inhibition to features with higher topological uncertainty while preserving the meaningful intrasubspace connections for features with distinct subspace affiliations. Consequently, it produces a subspace-preserving sparse attention graph $\mathbf{A}^{(l)}$ that reveals the block-diagonal structures of the underlying subspaces.

\textbf{Discussion}. The strict thresholding property of $\alpha$-entmax with $\alpha = 1.5$ provides a principled mechanism for recovering the block-diagonal subspace structures from the given features. Specifically, the $\alpha$-entmax projection inherently maps low-confidence similarities to exact zeros via Tsallis entropy regularization. The structural sparsity constraint $\Omega(\mathbf{A}^{(\mu)}, \mathbf{A}^{(\nu)})$ in Eq.~\eqref{eq:obj} is implicitly implemented through this differentiable thresholding mechanism, which produces sparse attention graphs with $\mathbf{A}^{(l)} \ne \mathbf{I}$. Consequently, this projection ensures that the proposed SAGL model consistently produces directed sparse attention graphs without additional regularization terms.

\subsubsection{Sparse Information Aggregation}
To obtain view-specific representations, we introduce a sparse information aggregation strategy that integrates sparse representation learning with a residual connection. Given a subspace-preserving sparse attention graph $\mathbf{A}^{(l)}$ $\left( {l \in \left\{ {\mu, \nu} \right\}} \right)$,  the sparse information aggregation operation performed on $\mathbf{z}_i^{(l)}$, denoted as $ f\left( {\mathbf{z}_i^{(l)}}; \mathbf{a}_i^{(l)} \right)$, is defined as follows:
\begin{equation}\label{eq:rep}
\begin{split}
\mathbf{p}_i^{(l)} & = f\left( {\mathbf{z}_i^{(l)}}; \mathbf{a}_i^{(l)} \right)  = \sum\limits_{j = 1}^n {{a_{ij}^{(l)}}{\mathbf{z}_j^{(l)}} + } {\mathbf{z}_i^{(l)}}.
\end{split}
\end{equation}
where $\mathbf{a}_i^{(l)}$ represents the $i$th column vector of $\mathbf{A}^{(l)}$, and ${\mathbf{p}_i} \in {\mathbb{R}^{C}}$ denotes the representation corresponding to sample ${\mathbf{x}_i}$. Unlike aggregating the information derived from all features, the sparse aggregation operation leverages the highly sparse nature of $\mathbf{a}_i^{(l)}$. This ensures that $\mathbf{p}_i^{(l)}$ is represented as a sparse linear combination of spatially proximate neighbors belonging to the same latent subspace. By strictly restricting the aggregation procedure to the intrinsic subspace, this operation filters out inter-subspace noise. This significantly enhances the discriminative capabilities of representations.

The probability simplex constraint $\mathbf{A}^{(l)}\mathbf{1} = \mathbf{1}$, $\mathbf{A}^{(l)} \ge \mathbf{0}$ directly corresponds to the sparse self-representation constraint in Eq.~\eqref{eq:obj}. The sparse information aggregation uses the sparsity of $\mathbf{A}^{(l)}$ to refine the representations in Eq.~\eqref{eq:rep}. By restricting the aggregation stage to intrinsic subspaces, the mechanism filters out intersubspace noise and produces more discriminative representations. For the $l$th view, the assignment probability distribution of the representation $\mathbf{p}_i^{(l)}$, i.e., $\mathbf{q}_i^{(l)} \in \Delta^{C-1}$, corresponding to sample $\mathbf{x}_i$ $\left( {1 \le i \le n} \right)$, is calculated as follows:
\begin{equation}\label{eq:prob}
\begin{split}
\mathbf{q}_i^{(l)} = \rho \left( \mathbf{p}_i^{(l)} \right)
\end{split}
\end{equation}
where $\rho \left( { \cdot } \right)$ is a softmax activation function and $\Delta^{C-1}$ denotes the probability simplex. The complete optimization objective of SAGL is provided in the appendix.

\subsection{Theoretical Justification}

\subsubsection{Asymmetry of Bilinear Attention Factorization}
The underlying relationship among high-dimensional features are typically directed. The directional dependency implies that the similarity relationship from $\mathbf{z}_j^{(l)}$ to $\mathbf{z}_i^{(l)}$ is inherently asymmetric. We provide a theoretical analysis of the asymmetric nature of bilinear attention factorization. The assumptions and proofs of all presented theorems are provided in the appendix.

\begin{theorem}
\label{thm:capacity}
[\textbf{Universal Representation Capacity for Asymmetric Subspaces}]
Let a linear feature matrix $\mathbf{Z} \in \mathbb{R}^{n \times c}$ have full column rank, and let $\mathbf{A}^* \in \mathbb{R}^{n \times n}$ be an arbitrary asymmetric adjacency matrix defining the directed similarity relationships among different subspaces. If $\mathbf{A}^*$ resides within the bilinear span of the features, i.e., $\mathbf{A}^* = \mathbf{Z}\mathbf{W}^*\mathbf{Z}^\top$, for some latent structural matrix $\mathbf{W}^* \in \mathbb{R}^{c \times c}$,  the bilinear attention factorization scheme implemented via projection matrices $\mathbf{U}, \mathbf{V} \in \mathbb{R}^{c \times c}$ is guaranteed to possess universal expressive power for the exact recovery of $\mathbf{A}^*$.
\end{theorem}

Theorem~\ref{thm:capacity} provides a rigorous theoretical justification for employing two distinct projection matrices, i.e., $\mathbf{U}^{(l)}$ and $\mathbf{V}^{(l)}$, in Eq.~\eqref{eq:lrs}. Specifically, it demonstrates that bilinear attention factorization has universal expressive power for pursuing the arbitrary directed similarity relationships of latent subspaces. The traditional representation learning techniques often evaluate similarity relationships via symmetric projections, which forces the underlying structural matrix $\mathbf{W}^*$ to be strictly symmetric. In contrast, the proposed bilinear attention factorization scheme successfully breaks the symmetry bottleneck of standard inner product operations. This enables the proposed SAGL model to faithfully capture the directed relationships inherent in high-dimensional features.

\subsubsection{Rethinking $\ell_0$-Norm Self-Representation as Simplex-Projected Attention}
We establish a theoretical bridge between the $\ell_0$-norm self-representation and the differentiable probability simplex-projected attention. The probability simplex constraint provides a principled convex relaxation that preserves the subspace-preserving property of the original $\ell_0$-norm solution.%We consider features drawn from a union of $K$ subspaces: $\mathcal{S} = \bigcup_{k=1}^{K} \mathcal{S}_k, \quad z_i \in \mathcal{S}_{y_i}$.

\begin{definition}
\label{def:spr}
[\textbf{Subspace-Preserving Representation}]
A sparse coefficient vector $\mathbf{a}_i^{(l)} \in \mathbb{R}^n$ is said to be \emph{subspace-preserving} for a feature $\mathbf{z}_i^{(l)} \in \mathcal{S}_k$ if $a_{ij}^{(l)} \ne 0$ implies that $\mathbf{z}_j^{(l)} \in \mathcal{S}_k$; i.e., all nonzero coefficients index features from the same latent subspace.
\end{definition}

\begin{theorem}
\label{thm:relaxation}
[\textbf{$\ell_0$-Norm-to-Simplex Relaxation}]
Let the features $\{\mathbf{z}_i^{(l)}\}_{i=1}^n$ be drawn from a union of linear subspaces $\{\mathcal{S}_k\}_{k=1}^K$. Under the independent subspace, sufficient intrasubspace sampling, and similarity separation assumptions (Assumptions~\ref{ass:indsub},~\ref{ass:sufsamp}, and~\ref{ass:simsep}, respectively), the following subspace-preserving properties hold:

\textbf{(1) Subspace-Preserving Property of the $\ell_0$-Norm Solution:}
Consider the $\ell_0$-norm sparse self-representation problem:
\begin{equation}\label{eq:l0}
\widetilde{\mathbf{a}}_i^{(l)} = \arg\min_{\mathbf{a}_i^{(l)}} \left\| \mathbf{a}_i^{(l)} \right\|_0 \quad \text{s.t.} \quad \mathbf{z}_i^{(l)} = {\left( \mathbf{Z}^{\left( l \right)} \right)^T} \mathbf{a}_i^{(l)}, \quad a_{ii}^{(l)} = 0.
\end{equation}
Then the solution $\widetilde{\mathbf{a}}_i^{(l)}$ is guaranteed to be subspace-preserving, i.e., $\widetilde{a}_{ij}^{(l)} \neq 0 \implies \mathbf{z}_j^{(l)} \in \mathcal{S}_{y_i}$.

\textbf{(2) Exact Subspace Preservation via Simplex Projection:}
Let $\widetilde{\mathbf{a}}_i^{(l)} = \text{entmax}_\alpha(\tilde {\mathbf{s}}_i^{(l)}) $ for $\alpha > 1$, where $\tilde {\mathbf{s}}_i^{(l)}$ is the similarity vector. Consider the entmax-regularized attention projection:
\begin{equation}
\widetilde{\mathbf{a}}_i^{(l)} = \arg\max_{\mathbf{a}_i^{(l)} \in \Delta^{n-1}} \langle \mathbf{a}^{(l)}, \tilde {\mathbf{s}}_i^{(l)} \rangle - \Omega_\alpha(\mathbf{a}_i^{(l)}),
\end{equation}
where $\Omega_\alpha(\mathbf{a}^{(l)})$ is the Tsallis entropy for $\alpha > 1$. If the directed similarities satisfy the similarity separation condition (Assumption~\ref{ass:simsep}) with a sufficiently large margin $\epsilon > 0$, then the resulting sparse attention is exactly subspace-preserving, i.e.,
\begin{equation}
\widetilde{a}_{ij}^{(l)} > 0 \implies \mathbf{z}_j^{(l)} \in \mathcal{S}_{y_i}.
\end{equation}
\end{theorem}

Theorem~\ref{thm:relaxation} establishes the fact that the probability simplex constraint $\mathbf{A}^{(l)}\mathbf{1} = \mathbf{1}$, $\mathbf{A}^{(l)} \ge \mathbf{0}$ in Eq.~\eqref{eq:obj} provides a principled convex relaxation of the $\ell_0$-norm self-representation constraint. Under strict separation between intra- and intersubspace similarities, this thresholding mechanism guarantees that only intrasubspace connections remain active, thereby preserving the subspace structure.  By employing $\alpha$-entmax with $\alpha = 1.5$, we obtain a continuous analog of the $\ell_0$ sparsity operator. This formulation facilitates the end-to-end differentiable optimization process without relying on iterative algebraic solvers. Consequently, the subspace-preserving property established for simplex projection inherently extends to entmax-based attention under analogous separation conditions.

\section{Experiments}
\label{sec:exp}

We conduct extensive experiments to evaluate the effectiveness of the proposed SAGL method on eight publicly available datasets: Pets, KITTI, Flowers, Caltech101, EuroSAT, SUN397, Food101 and ImageNet-1K \cite{Paszke2019Pytorch}. Specifically, we consider three transfer learning tasks: self-supervised, zero-shot, and supervised transfer learning. For self-supervised transfer learning, three standard metrics are employed to evaluate the clustering performance of all competing methods, i.e., clustering accuracy (ACC), normalized mutual information (NMI) and adjusted rand index (ARI) \cite{Chen2026MSRL}. We compare SAGL against representative state-of-the-art methods, including SRLF \cite{Xu2026BONN}, principle of rate reduction (PRR) \cite{Chu2024IC}, masked image modeling (MIM) \cite{Alkin2025MIM}, MSRL \cite{Chen2026MSRL} and TURTLE \cite{Gadetsky2024UT}. For zero-shot transfer learning tasks, CLIP zero-shot transfer \cite{Gadetsky2024UT} and LaFTer \cite{Mirza2023FT} are included as baselines. For supervised transfer learning, we employ a supervised linear probe as the baseline. We utilize two representative pretrained backbones, i.e., DINOv3 \cite{Simeoni2025Dinov3} and SigLIP 2 \cite{Tschannen2025LIP}, as backbones. The source codes of all the baselines are provided by their respective authors. The source code of SAGL is publicly available at \url{https://github.com/chenjie20/SAGL}.

\begin{table*}[!htbp]
\tiny
\setlength{\tabcolsep}{0.8pt}
\centering
\caption{Clustering performance of competing methods on eight datasets.}
\label{tb:ssl:results}
\begin{tabular}{c|ccc|ccc|ccc|ccc|ccc|ccc|ccc|ccc}
\hline
\multirow{2}*{Methods} & \multicolumn{3}{c|}{Pets} & \multicolumn{3}{c|}{KITTI}  & \multicolumn{3}{c|}{Flowers} & \multicolumn{3}{c|}{Caltech101} & \multicolumn{3}{c|}{EuroSAT} & \multicolumn{3}{c|}{SUN397} & \multicolumn{3}{c|}{Food101}  & \multicolumn{3}{c}{ImageNet-1K} \\
\cline{2-25}
~ & ACC & NMI & ARI & ACC & NMI & ARI & ACC & NMI & ARI & ACC & NMI & ARI & ACC & NMI & ARI & ACC & NMI & ARI & ACC & NMI & ARI & ACC & NMI & ARI \\
\hline
SRLF & 95.45  & 95.76 & 91.97 & 42.18 & 13.64 & 8.60 & 94.06 & 98.48 & 95.74 & 65.27 & 87.69 & 46.27 & 95.68 & 90.25 & 90.76 & 63.68 & 82.89 & 53.59 & 91.53 & 92.20 & 85.72 & 69.41 & 85.07 & 56.09 \\
PRR & 85.28 & 80.99 & 78.87 & 45.06  & 16.14 & 11.42 & 94.59 & 93.17 & 90.92 & 78.28 & 72.95 & 64.59 & 80.86 & 56.41 & 59.50 & 64.37 & 80.88 & 51.18 &  92.81 & 92.48 & 88.28 & 72.97 & 68.53 & 58.92 \\
MIM & 83.89 & 87.89 & 76.49 & 45.30 & 15.46 & 10.61 & 97.50 & 96.77 & 95.33& 86.90 & 83.62 & 78.66 & 87.66 & 56.34 & 61.34 & 66.40 & 75.32 & 57.13 & 92.20 & 79.62 & 80.61 & 63.24 & 84.72 & 40.18 \\
TURTLE & 96.05 & 96.14 & 92.79 & 45.59 & 16.51 & 11.88 & 99.59 & 99.66 & 99.40 & 90.14 & 94.59 &  \underline{88.16} & 96.38  & 91.61 & 92.17 & \underline{71.18} & \underline{85.13} & \underline{59.01} &  95.02 & 94.97 & 90.95 & 71.98 & 88.12 & 61.72 \\
MSRL & \underline{97.38} & \underline{97.18} & \underline{95.00} & \underline{53.74} & \underline{19.39} & \underline{16.94} & \underline{99.72} & \underline{99.76} & \underline{99.59} & \underline{90.38} & \underline{94.69} & 86.92 &  \underline{99.88} &  \underline{99.67} &  \underline{99.73} &  70.26 & 85.33 & 59.61 & \underline{95.07} & \underline{95.13} & \underline{91.09} & \underline{74.65} & \underline{89.17} & \underline{65.14} \\
SAGL & \textbf{97.57} & \textbf{97.33} & \textbf{95.35} &  \textbf{56.95} & \textbf{24.73} & \textbf{26.12} & \textbf{99.76} & \textbf{99.80} & \textbf{99.65} & \textbf{ 91.54} & \textbf{95.35} & \textbf{89.01} & \textbf{99.92} & \textbf{99.80} & \textbf{99.82} &\textbf{73.34} & \textbf{86.29} & \textbf{62.27} & \textbf{97.68} & \textbf{97.42} & \textbf{95.53} & \textbf{79.34} &  \textbf{90.98} &  \textbf{70.50} \\
\hline
\end{tabular}
\end{table*}

\begin{table*}[!htbp]
\tiny
\setlength{\tabcolsep}{4pt}
\centering
\caption{A comparison of performance between SAGL and the supervised LR baselines.}
\label{tb:stl:results}
\begin{tabular}{c|c|c|c|c|c|c|c|c}
\hline
Methods & Pets & KITTI & Flowers & Caltech101 & EuroSAT & SUN397 & Food101 & ImageNet-1K \\
\hline
LR+SigLIP 2& 96.76 & \underline{73.33} & 99.66 & \textbf{97.22} & 96.14 & \textbf{83.52} & \underline{96.91} &  \textbf{87.04} \\
LR+DINOv3 & \underline{97.08} & \textbf{76.60} & \underline{99.72} & \underline{94.51} & \underline{97.40} & \underline{79.80} & 95.81 & \underline{86.87} \\
SAGL & \textbf{97.57} & 56.95 & \textbf{99.76} & 91.54 & \textbf{99.92} & 73.34 & \textbf{97.68} &79.34 \\
\hline
\end{tabular}
\end{table*}

\begin{table*}[!htbp]
\tiny
\setlength{\tabcolsep}{0.8pt}
\centering
\caption{Clustering performance of SAGL with different pretrained model combinations.}
\label{tb:cbn:results}
\begin{tabular}{c|ccc|ccc|ccc|ccc|ccc|ccc|ccc|ccc}
\hline
\multirow{2}*{Methods} & \multicolumn{3}{c|}{Pets} & \multicolumn{3}{c|}{KITTI}  & \multicolumn{3}{c|}{Flowers} & \multicolumn{3}{c|}{Caltech101} & \multicolumn{3}{c|}{EuroSAT} & \multicolumn{3}{c|}{SUN397} & \multicolumn{3}{c|}{Food101}  & \multicolumn{3}{c}{ImageNet-1K} \\
\cline{2-25}
~ & ACC & NMI & ARI & ACC & NMI & ARI & ACC & NMI & ARI & ACC & NMI & ARI & ACC & NMI & ARI & ACC & NMI & ARI & ACC & NMI & ARI & ACC & NMI & ARI \\
\hline
(1) & \underline{94.63} & \underline{95.82} & \underline{91.73} & 52.14& 21.01 & \underline{20.01} &  \underline{99.71} & \underline{99.75} & \underline{99.56} & 85.39 & 92.46 & 83.72 & 99.62 & 99.00 & 99.16 &  63.49  & 82.09 & \underline{52.39} &  92.11 & 93.52 & 87.36 & 72.45 & 87.97 & 62.07 \\
(2) & 93.19 & 95.07 & 89.79 & \underline{52.41} & 18.60 & 15.34 &  69.78 &87.47 & 65.28 & 77.38 & 87.44 & 82.38 & 92.60 & 88.35 & 85.48 & 59.88 & 81.48 & 49.93 & \textbf{98.32} & \textbf{98.08} & \textbf{96.72} & 71.36 & 87.59 & 61.60 \\
(3) & 90.79 & 94.15 & 85.72 & 48.93 & \underline{24.53} & 14.84 & 84.55 & 93.41 & 45.75 & 88.53 & 94.21 & 82.46 & 88.76  & 81.66 & 77.22 & \underline{69.81} & \underline{84.17} & 32.59 & 94.02 & 94.67 & 88.45 & \underline{75.85} & 89.36 & 47.75 \\
(4) & 92.31 & 95.58 & 89.95 & 52.21 & 22.16 & 17.85 & 88.24 & 95.10 & 59.99 & \underline{89.74} & \underline{94.22} & \underline{85.44} & \underline{99.72} & \underline{99.32} & \underline{99.38} & 65.71 & 81.58 & 11.83 & 97.60 & 97.59 & 95.51 & 75.75 & \underline{89.72} & \underline{65.50} \\
SAGL & \textbf{97.57} & \textbf{97.33} & \textbf{95.35} &  \textbf{56.95} & \textbf{24.73} & \textbf{26.12} & \textbf{99.76} & \textbf{99.80} & \textbf{99.65} & \textbf{ 91.54} & \textbf{95.35} & \textbf{89.01} & \textbf{99.92} & \textbf{99.80} & \textbf{99.82} &\textbf{73.34} & \textbf{86.29} & \textbf{62.27} & \underline{97.68} & \underline{97.42} & \underline{95.53} & \textbf{79.34} &  \textbf{90.98} &  \textbf{70.50} \\
\hline
\end{tabular}
\end{table*}

\begin{table*}[!htbp]
\tiny
\setlength{\tabcolsep}{0.8pt}
\centering
\caption{Ablation study on the key components of the SAGL model.}
\label{tb:abl:results}
\begin{tabular}{c|ccc|ccc|ccc|ccc|ccc|ccc|ccc|ccc}
\hline
\multirow{2}*{Methods} & \multicolumn{3}{c|}{Pets} & \multicolumn{3}{c|}{KITTI}  & \multicolumn{3}{c|}{Flowers} & \multicolumn{3}{c|}{Caltech101} & \multicolumn{3}{c|}{EuroSAT} & \multicolumn{3}{c|}{SUN397} & \multicolumn{3}{c|}{Food101}  & \multicolumn{3}{c}{ImageNet-1K} \\
\cline{2-25}
~ & ACC & NMI & ARI & ACC & NMI & ARI & ACC & NMI & ARI & ACC & NMI & ARI & ACC & NMI & ARI & ACC & NMI & ARI & ACC & NMI & ARI & ACC & NMI & ARI \\
\hline
SAGL$_{i}$ & 96.18 & 96.07 & 92.95 & 50.53 & 17.57 & 15.56 & 99.07 & 99.54 & 97.98 & 90.19 & 94.14 & 87.28 & 96.40 & 91.72 & 92.22 & 70.21 & 84.53 & 58.02 & 95.13  & 94.99 & 91.11 & 73.97 & 88.79 & 64.28 \\
SAGL$_{b}$ & 63.61 & 72.98 & 50.04 & 52.47 & 21.85 & 19.36 & 99.32 & 99.66 & 99.41 & 80.08 & 88.20 & 70.45 & 86.34 & 76.86 & 73.43 & 31.99 & 57.39 & 2.70 & 67.72 & 77.91 & 54.66 & 18.86 & 47.04 & 0.24 \\
SAGL$_{w}$ & \underline{97.22} & \underline{97.34} & \underline{94.96} & \underline{55.08} & \underline{21.71} & \underline{22.98} & \underline{99.37} & \underline{99.71} & \underline{99.48} & \underline{90.52} & \underline{94.54} & \underline{87.79} & \underline{97.12} & \underline{93.65} & \underline{93.76} & \underline{71.24} & \underline{86.11} & \underline{61.17} & \underline{95.68} & \underline{95.74} & \underline{92.06} & \underline{77.54} & \underline{90.32} & \underline{68.32} \\
SAGL & \textbf{97.57} & \textbf{97.33} & \textbf{95.35} &  \textbf{56.95} & \textbf{24.73} & \textbf{26.12} & \textbf{99.76} & \textbf{99.80} & \textbf{99.65} & \textbf{ 91.54} & \textbf{95.35} & \textbf{89.01} & \textbf{99.92} & \textbf{99.80} & \textbf{99.82} &\textbf{73.34} & \textbf{86.29} & \textbf{62.27} & \textbf{97.68} & \textbf{97.42} & \textbf{95.53} & \textbf{79.34} &  \textbf{90.98} &  \textbf{70.50} \\
\hline
\end{tabular}
\end{table*}

\subsection{Performance Evaluation}
\subsubsection{Self-Supervised Transfer Learning}
For self-supervised transfer learning, we follow the standard unsupervised evaluation protocol \cite{Gadetsky2024UT, Chen2026MSRL}, where label information is unavailable during training. The proposed SAGL model is trained on the training set of each dataset and evaluated on the test set without any fine-tuning. The number of clusters is set to the ground-truth number of semantic categories for each dataset. The clustering results of all competing methods on the self-supervised learning task are reported in Table~\ref{tb:ssl:results}, where the best and second-best results are highlighted in bold and underlined, respectively. We observe that SAGL consistently outperforms the competing methods in terms of ACC, NMI, and ARI. In particular, the performance gap becomes increasingly evident on larger datasets, such as SUN397, Food101 and ImageNet-1K. For example, SAGL improves ACC by 2.16\%, 2.61\%, and 4.69\% on SUN397, Food101, and ImageNet-1K, respectively. SRLF does not exploit the generalization capability of pretrained backbones for clustering tasks. Although PRR and MIM incorporate pretrained backbones, they rely on individual views that often fail to provide sufficient information for clustering tasks. In contrast, SAGL, MSRL and TURTLE employ various pretrained backbones on large-scale unlabeled data to produce multiple views,  thereby effectively capturing complementary information across different views. Additionally, SAGL consistently outperforms both MSRL and TURTLE across all the datasets. This is because SAGL further considers the intrinsic subspace structures embedded in the high-dimensional multiview data. Consequently, these results validate the effectiveness of the proposed SAGL approach for self-supervised learning tasks.

\begin{figure}[!htbp]
\centering
\includegraphics[width=0.6\linewidth]{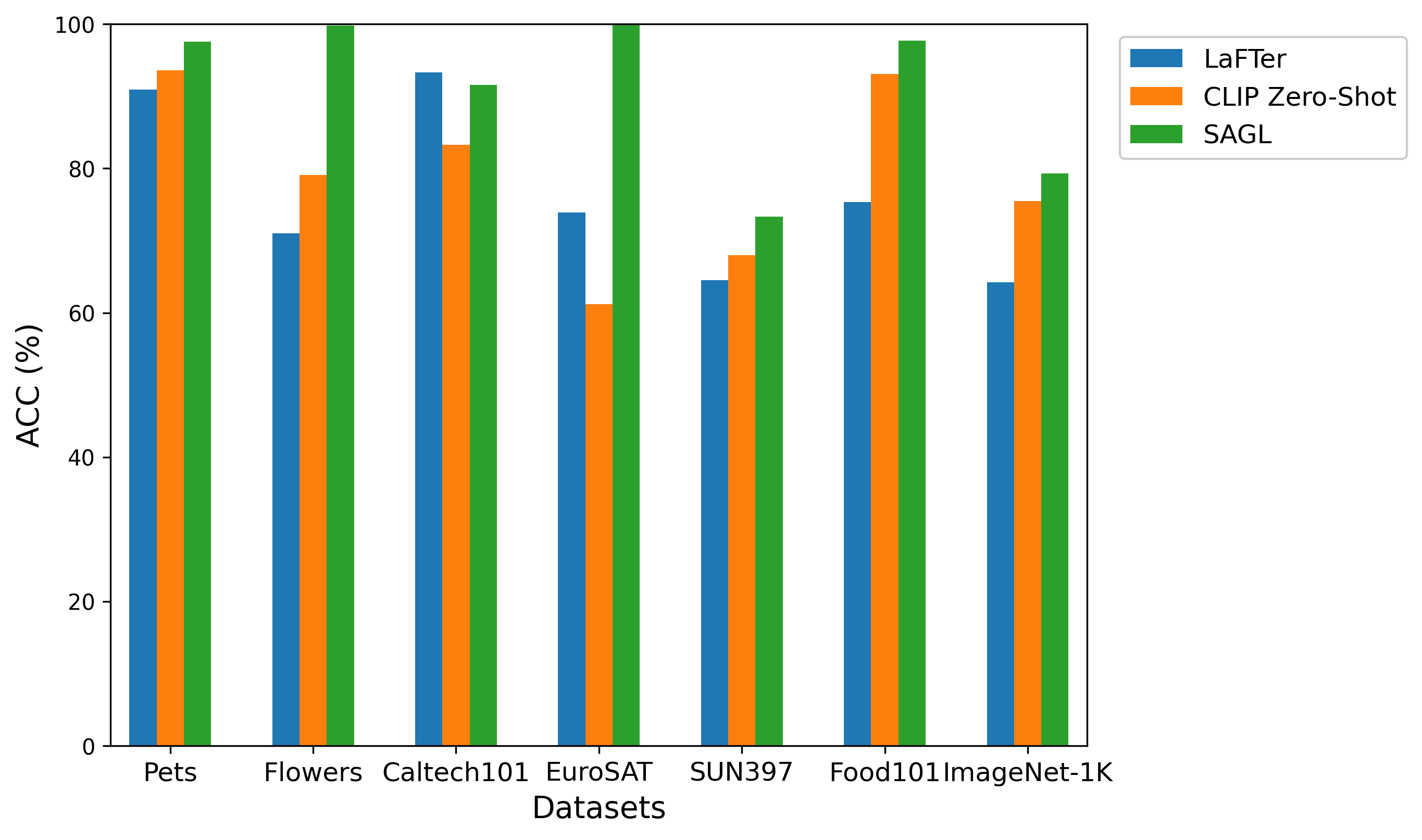}
\caption{A comparison of performance between SAGL and the zero-shot transfer learning baselines.}
\label{fig:zeroshot} %% label for entire figure
\end{figure}

\subsubsection{Zero-Shot Transfer Learning}
The zero-shot transfer learning baselines, such as CLIP zero-shot transfer (ZST) and LaFTer , utilize descriptions of ground-truth classes as a form of supervision. CLIP ZST and LaFTer employ CLIP ViT-L/14 and ViT-B/32 backbones, respectively. Fig.~\ref{fig:zeroshot} shows the clustering accuracies of SAGL and these zero-shot transfer learning baselines across all datasets except KITTI. The KITTI dataset is not included due to its relatively small number of clusters. SAGL outperforms these weakly supervised baselines across most datasets without requiring any class descriptions or label information.

\subsubsection{Supervised Transfer Learning}
We employ a supervised linear probe as an important evaluation protocol to measure the performance gap between SAGL and supervised baselines. The probe is implemented via a standard logistic regression (LR) model. Specifically, for the supervised LR baselines, we  freeze the pretrained backbones and train only a linear classifier on top of the features extracted by the individual pretrained backbones (i.e., DINOv3 and SigLIP 2) using all available labels. Table~\ref{tb:stl:results} shows a comparison between the clustering accuracy of SAGL and the classification accuracy of the supervised LR baselines across eight datasets. Despite the absence of label information during training, SAGL outperforms the supervised LR baselines on the Pets, Flowers, EuroSAT, and Food101 datasets. While the supervised baselines are often limited by the features produced by individual pretrained backbones, SAGL effectively leverages the complementary information of multiple views to produce more discriminative representations, even without label supervision.

\subsection{Evaluating Various Pre-trained Model Combinations}
We construct feature pairs from different pretrained backbones, including DINOv3, SigLIP 2, CLIP ViT-L/14, and ConvNeXt V2. Specifically, we consider four combinations: (1) ConvNeXt V2 and DINOv3, (2) ConvNeXt V2 and SigLIP 2, (3) SigLIP 2, DINOv3, and ConvNeXt V2, and (4) SigLIP 2, DINOv3, and CLIP ViT-L/14. Table~\ref{tb:cbn:results} shows clustering performance of SAGL with four pretrained model combinations across eight datasets. Incorporating additional views is generally expected to provide richer complementary information. However, the results indicate that increasing the number of views does not consistently lead to improved clustering performance. For example, SAGL with two-view combinations achieves better clustering performance than with three-view combinations, such as (3) and (4). This is because highly uncorrelated views provide limited complementary information while increasing computational cost. These findings highlight the importance of carefully selecting view combinations that maximize complementary information rather than simply increasing the number of views. In addition, model pairs sharing the same architecture achieve substantially higher clustering results than heterogeneous pairs combining different architectures. For example, the results show that a combination of ViT-based models consistently outperforms a mixed-architecture combination across most datasets. The large distributional gap between heterogeneous architecture pairs makes it difficult for the SAGL model to recover the intrinsic subspace structures across views.

\subsection{Ablation Study}
We conduct an ablation study to verify the contribution of each key component in SAGL. Specifically, we consider three variants of SAGL: (1) replacing the sparse attention graphs with identity matrices, denoted as SAGL$_{i}$, (2) retaining only the dense attention graphs generated by bilinear attention factorization, denoted as SAGL$_{b}$, and (3) removing the dynamic sparsity gating mechanism, denoted as SAGL$_{w}$. Table~\ref{tb:abl:results} reports the clustering performance of these three variants against the full SAGL model. The full SAGL model significantly outperforms its variants, SAGL$_{i}$, SAGL$_{b}$ and SAGL$_{w}$. For example, the performance gap between SAGL and SAGL$_{i}$ becomes particularly significant on the large-scale datasets. When the structured sparse projection is replaced with dense attention graphs, the clustering performance of SAGL$_{b}$ drops dramatically across most datasets compared to the full SAGL model. The results demonstrate the effectiveness of the sparse information aggregation operation in Eq.~\eqref{eq:rep}. Additionally, the clustering results of SAGL$_{w}$ also highlight the crucial role of the dynamic sparsity gate in adaptive neighbor selection.

\section{Conclusion}
\label{sec:conclusion}
In this paper, we present the SAGL method for learning subspace-preserving sparse attention graphs from heterogeneous multiview data. SAGL utilizes bilinear attention factorization to capture asymmetric similarities among high-dimensional features, successfully breaking the symmetry constraints in  traditional representation learning techniques. Moreover, we introduce a dynamic sparsity gating mechanism that adaptively controls the topological contributions of neighbors by predicting a feature-specific compression factor. To enforce structured sparsity, we adopt a structured sparse projection via $\alpha$-entmax to generate subspace-preserving sparse attention graphs. By leveraging these view-specific graphs, SAGL conducts sparse information aggregation to produce discriminative representations for unsupervised transfer learning without the need for computationally expensive iterative solvers. Finally, we provide a rigorous theoretical analysis bridging the gap between differentiable sparse attention and probability simplex constraints. This analysis demonstrates how structured sparsity enforces subspace-preserving representations. Extensive experiments demonstrate that SAGL consistently outperforms state-of-the-art unsupervised transfer learning approaches.

{
\small
\bibliographystyle{IEEEtran}
\bibliography{sagl}
}

\newpage

\appendix

\section{Technical Appendices}

\subsection{Enforcing Multiview Alignment under Structural Sparsity Constraints}
\label{sec:sparsity}
To align the heterogeneous representations across multiple views in a fully unsupervised manner, SAGL is optimized by jointly minimizing three complementary components: pseudolabel prediction, cluster diversity regularization, and cross-view consistency alignment \cite{Chen2026MSRL}. We optimize the entire network using an adaptive momentum-based mini-batch gradient descent strategy \cite{Kingma2015ADAM}, with the pretrained backbones remaining frozen during training. The complete optimization procedure of the proposed SAGL method is summarized in Algorithm~\ref{alg:sagl}.

\begin{algorithm}[tb]
\caption{Optimization Procedure for SAGL}
\label{alg:sagl}
\textbf{Input}: An unlabeled training dataset $\mathcal{D}_{tr} = \left\{ \mathbf{x}_1, \mathbf{x}_2, \ldots, \mathbf{x}_n \right\}$ and a testing dataset $\mathcal{D}_{ts} = \left\{ \hat{\mathbf{x}}_1, \hat{\mathbf{x}}_2, \ldots, \hat{\mathbf{x}}_{\hat{n}} \right\}$ belonging to $C$ categories, and $L$ pre-trained models. \\
\textbf{Parameters}: The number of training epochs $\mathit{epochs}$, batch size $B$, and hyperparameters $\gamma > 0$ and $\beta > 0$. \\
\textbf{Output}: The predicted labels $\mathbf{Y} = \left[ \hat{y}_1, \hat{y}_2, \ldots, \hat{y}_{\hat{n}} \right]$ for $\mathcal{D}_{ts}$.
\begin{algorithmic}[1]
\STATE Construct the multiview data $\mathcal{H}_{tr}$ from $\mathcal{D}_{tr}$ via Eq.~\eqref{eq:feat};
\FOR{$t = 1$ \textbf{to} $\mathit{epochs}$}
    \FOR{each mini-batch $\mathcal{B}$ of size $B$ sampled from $\mathcal{H}_{tr}$}
        \FOR{$l = 1$ \textbf{to} $L$}
            \STATE Compute the linear features $\mathbf{Z}^{(l)}$ for samples in $\mathcal{B}$ via Eq.~\eqref{eq:linearfeat};
            \STATE Compute the raw similarity matrix $\mathbf{S}^{(l)}$ via Eq.~\eqref{eq:lrs};
            \STATE Compute the directed similarities $\tilde{\mathbf{S}}^{(l)}$ via Eq.~\eqref{eq:sim};
            \STATE Construct the subspace-preserving sparse attention graph $\mathbf{A}^{(l)}$ via Eq.~\eqref{eq:entmax};
            \STATE Compute the assignment probability distribution $\mathbf{q}_i^{(l)}$ for each sample $\mathbf{x}_i \in \mathcal{B}$ via Eq.~\eqref{eq:prob};
        \ENDFOR
        \STATE Determine the pseudolabel $y_i$ for each sample $\mathbf{x}_i \in \mathcal{B}$ via Eq.~\eqref{eq:pseudo};
        \STATE Update the network parameters by minimizing $\mathcal{L}$ in Eq.~\eqref{eq:totalloss};
    \ENDFOR
\ENDFOR
\STATE Construct the multiview data $\mathcal{H}_{ts}$ from $\mathcal{D}_{ts}$ via Eq.~\eqref{eq:feat};
\STATE Perform forward computation on $\mathcal{H}_{ts}$ to predict the labels $\hat{y}_i$ for all samples $\hat{\mathbf{x}}_i \in \mathcal{D}_{ts}$ via Eq.~\eqref{eq:pseudo}.
\end{algorithmic}
\end{algorithm}

\textbf{Pseudolabel Prediction Loss.} Consider the view-specific representations $\mathbf{q}_i^{(l)}$ produced by the sparse information aggregation process described in Eq.~\eqref{eq:prob}. To enable self-supervised training to be implemented in the absence of ground-truth annotations,  we construct pseudolabels by aggregating the predicted soft assignments across all $L$ views:
\begin{equation}\label{eq:pseudo}
\hat{y}_i = \arg\max_{k} \ \frac{1}{L} \sum_{l=1}^{L}q_{ik}^{(l)}
\end{equation}
where $q_{ik}^{(l)}$ denotes the predicted probability of sample $\mathbf{x}_i$ belonging to class $k$  $\left( 1 \le k \le C \right)$ in the $l$th view. A negative log-likelihood loss is applied to refine each view-specific prediction:
\begin{equation}\label{eq:pred_loss}
\mathcal{L}_{pseudo} = -\frac{1}{n} \sum_{l=1}^{L} \sum_{i=1}^{n} \log q_{i\hat{y}_i}^{(l)}.
\end{equation}
This loss encourages each independent view to produce predictions that align with the global consensus distribution.

\textbf{Cluster Diversity Regularization.} Unsupervised transfer learning-based clustering models are highly prone to producing trivial solutions, where all instances collapse into a single underlying subspace (i.e., a single cluster). We explicitly maximize the marginal entropy of the global cluster assignments. Let $\bar{\mathbf{q}}_c^{(l)}$ denote the average assignment probability of all instances that are assigned to category $j$. Cluster diversity can be achieved by minimizing the negative entropy:
\begin{equation}
\mathcal{L}_{div} = \sum_{l=1}^{L} \sum_{j=1}^{C} \bar{q}_j^{(l)} \log \bar{q}_j^{(l)}
\end{equation}
where
\begin{equation}
\bar{q}_j^{(l)} = \frac{1}{n} \sum_{i=1}^{n} q_{i,j}^{(l)}.
\end{equation}

\textbf{Cross-View Consistency Alignment.}
The semantic assignment probability distributions $\mathbf{q}_i^{(\mu)}$ and $\mathbf{q}_i^{(\nu)}$ produced by view $\mu$ and view $\nu$, respectively, should exhibit strong semantic alignment. To enforce semantic consistency across heterogeneous views $\left( \mu, \nu \right)$, we minimize the cross-entropy among the soft assignment distributions of different pretrained models:
\begin{equation}
\mathcal{L}_{align} = - \frac{1}{n} \sum_{i=1}^{n} \sum_{\mu=1}^{L}  \sum_{\nu=1, \mu \ne \nu}^{L} \sum_{j=1}^{C} q_{i,j}^{(\mu)} \log q_{i,j}^{(\nu)}.
\end{equation}

\textbf{Overall Objective.} The overall optimization objective of SAGL is defined as follows:
\begin{equation}
\label{eq:totalloss}
\mathcal{L}_{total} = \mathcal{L}_{pseudo} + \gamma \mathcal{L}_{div} + \beta \mathcal{L}_{align}
\end{equation}
where $\gamma$ and $\beta$ are the trade-off parameters. The overall objective $\mathcal{L}_{total}$ represents a specific implementation of the general self-supervised alignment loss $\mathcal{L} \left( \mathbf{p}_i^{(\mu)}, \mathbf{p}_i^{(\nu)} \right)$ presented in Eq.~\eqref{eq:obj}.

\textbf{Limitations}. SAGL relies on the availability of sufficient neighbors within the same latent subspace to perform sparse information aggregation. A limitation is introduced during mini-batch training when sample shuffling is applied, as it may lead to an insufficient number of samples from certain semantic categories within a single batch. This violates the sufficient intrasubspace sampling assumption, which requires each feature's neighborhood to contain at least $d_k$ linearly independent samples from its corresponding subspace to ensure faithful recovery of the underlying subspace structures. In such cases, the $\alpha$-entmax projection may fail to establish robust and meaningful connections. As a result, this may degrade the separability of the learned representations and reduce overall clustering performance.

\subsection{Proof of Theorem~\ref{thm:capacity}}
\begin{proof}
Let $\mathbf{A}^* = \mathbf{Z}\mathbf{W}^*\mathbf{Z}^\top$ denote the target adjacency matrix. The objective is to show that there exist matrices $\mathbf{U}, \mathbf{V} \in \mathbb{R}^{C \times C}$ such that
\begin{equation}
\frac{1}{\sqrt{c}} \mathbf{Z}\mathbf{U}\mathbf{V}^\top \mathbf{Z}^\top = \mathbf{A}^*.
\end{equation}

Since $\mathbf{W}^*$ is an arbitrary (potentially asymmetric) matrix, we can perform singular value decomposition (SVD) on the scaled target $\sqrt{C}\mathbf{W}^*$:
\begin{equation}
\sqrt{c}\mathbf{W}^* = \mathbf{P} \mathbf{\Sigma} \mathbf{Q}^\top
\end{equation}
where $\mathbf{P}, \mathbf{Q} \in \mathbb{R}^{C \times C}$ are orthogonal matrices containing the left and right singular vectors, respectively, and $\mathbf{\Sigma} \in \mathbb{R}^{C \times C}$ is the diagonal matrix of singular values. We can explicitly construct the optimal projection matrices as follows:
\begin{equation}
\mathbf{U} = \mathbf{P} \mathbf{\Sigma}^{1/2}, \quad \mathbf{V} = \mathbf{Q} \mathbf{\Sigma}^{1/2}.
\end{equation}
Since the diagonal entries of $\mathbf{\Sigma}$ are nonnegative singular values, $\mathbf{\Sigma}^{1/2}$ is a strictly real-valued matrix, which ensures that $\mathbf{U}, \mathbf{V} \in \mathbb{R}^{C \times C}$. Then, we have
\begin{equation}\label{eq:uv}
\begin{aligned}
\mathbf{U}\mathbf{V}^\top &= (\mathbf{P} \mathbf{\Sigma}^{1/2}) (\mathbf{Q} \mathbf{\Sigma}^{1/2})^\top \\
&= \mathbf{P} \mathbf{\Sigma}^{1/2} \mathbf{\Sigma}^{1/2} \mathbf{Q}^\top \\
&= \mathbf{P} \mathbf{\Sigma} \mathbf{Q}^\top \\
&= \sqrt{c}\mathbf{W}^*
\end{aligned}
\end{equation}
Upon substituting Eq.~\eqref{eq:uv} back into the bilinear attention factorization formulation given in Eq.~\eqref{eq:lrs}, we obtain
\begin{equation}
\begin{aligned}
\mathbf{A} &= \frac{1}{\sqrt{c}}\mathbf{Z}(\mathbf{U}\mathbf{V}^\top)\mathbf{Z}^\top \\
& = \frac{1}{\sqrt{c}}\mathbf{Z}(\sqrt{C}\mathbf{W}^*)\mathbf{Z}^\top \\
& = \mathbf{Z}\mathbf{W}^*\mathbf{Z}^\top = \mathbf{A}^*
\end{aligned}
\end{equation}
Therefore, $\mathbf{U}$ and $\mathbf{V}$ always exist such that the bilinear attention factorization exactly recovers any arbitrary asymmetric adjacency matrix $\mathbf{A}^*$, where $\mathbf{U}$ and $\mathbf{V}$ are constructed independently from the left and right singular vectors, respectively.
\end{proof}

\subsection{Proof of Theorem~\ref{thm:relaxation}}
\begin{proof}
\textbf{(1)} Let $\mathbf{z}_i^{(l)} \in \mathcal{S}_k$. Under the independent subspace assumption, any linear feature $\mathbf{z}_i^{(l)} \in \mathcal{S}_k$ cannot be represented as a linear combination of features derived from $\mathcal{S}_{k'} \left(  k' \ne k \right)$. Consequently, assigning coefficients for the features located outside $\mathcal{S}_k$ does not reduce the reconstruction error $\left\| \mathbf{z}_i^{(l)} - \mathbf{Z}^{(l)} \mathbf{a}_i^{(l)} \right\|_2^2$. Conversely, this strictly increases the $\ell_0$-norm penalty. Thus, the optimal solution $\widetilde{\mathbf{a}}_i^{(l)}$ that minimizes Eq.~\eqref{eq:l0} must assign zero to all cross-subspace features, resulting in a subspace-preserving representation.

\textbf{(2)}
The entmax operator for $\alpha > 1$ is defined as follows:
\begin{equation}
\text{entmax}_\alpha(\tilde {\mathbf{s}}_i^{(l)})_j = [(\alpha-1)\tilde s_{ij}^{(l)} - \tau_i]_+^{1/(\alpha-1)},
\end{equation}
where $[ \cdot ]_+ = \max(0, \cdot)$ and $\tau_i \in \mathbb{R}$ is the unique threshold determined by the simplex constraint $\sum_j a_{ij}^{(l)} = 1$.

Let $\mathbf{z}_i^{(l)} \in \mathcal{S}_{y_i}$. Under Assumption~\ref{ass:simsep}, a margin $\epsilon > 0$ separates the intrasubspace similarities from the intersubspace similarities:
\begin{equation}
\min_{j: \mathbf{z}_j^{(l)} \in \mathcal{S}_{y_i}}\tilde s_{ij}^{(l)} > \max_{j: \mathbf{z}_j^{(l)} \notin \mathcal{S}_{y_i}} \tilde s_{ij}^{(l)} + \epsilon.
\end{equation}

\textbf{Step 1: Intersubspace sparsity.} If the separation margin $\epsilon$ is sufficiently large (combined with the sufficient intrasubspace sampling of $\mathcal{S}_{y_i}$ from Assumption~\ref{ass:sufsamp}), then the threshold $\tau_i$ strictly exceeds the values of all inter-subspace similarity terms $(\alpha-1) \tilde s_{ij}^{(l)}$ for $j: \mathbf{z}_j^{(l)} \notin \mathcal{S}_{y_i}$; i.e.,
\begin{equation}
\tau_i \ge (\alpha-1) \max_{j: \mathbf{z}_j^{(l)} \notin \mathcal{S}_{y_i}} \tilde s_{ij}^{(l)}.
\end{equation}
For any cross-subspace feature $\mathbf{z}_j^{(l)} \notin \mathcal{S}_{y_i}$, we have
\begin{equation}
(\alpha-1)\tilde s_{ij}^{(l)} - \tau_i \le 0,  \ \forall j \notin \mathcal{S}_{y_i}.
\end{equation}
This implies that $\widetilde{a}_{ij}^{(l)} = 0, \  \forall j \notin \mathcal{S}_{y_i}$.

\textbf{Step 2: Intrasubspace sparsity.}
To satisfy the unit sum constraint $\sum_j a_{ij}^{(l)} = 1$, the threshold $\tau_i$ must be bounded from above by the similarity of the active intrasubspace elements. Thus, there exists a threshold $\tau_i$ such that
\begin{equation}
(\alpha-1)\max_{j: \mathbf{z}_j^{(l)} \notin \mathcal{S}_{y_i}}\tilde s_{ij}^{(l)} \le \tau_i < (\alpha-1)\max_{j: \mathbf{z}_j^{(l)} \in \mathcal{S}_{y_i}} \tilde s_{ij}^{(l)}.
\end{equation}
In $\mathcal{S}_{y_i}$, the threshold $\tau_i$ is not required to be below all intrasubspace similarities. On the contrary, $\tau_i$ is uniquely determined by the simplex constraint and the distribution of intrasubspace similarities. Specifically,
for $j: \mathbf{z}_j^{(l)} \in \mathcal{S}_{y_i}$:
\begin{equation}
\widetilde{a}_{ij}^{(l)} > 0 \iff (\alpha-1)\tilde{s}_{ij}^{(l)} > \tau_i.
\end{equation}
By refining $\tilde{s}_{ij}^{(l)}$ to amplify the similarity gap within $\mathcal{S}_{y_i}$, the dynamic sparsity gating mechanism ensures that only the most spatially proximate intrasubspace neighbors satisfy the connection condition, resulting in a sparse attention graph.

Therefore, the support set $\mathcal{A}_i = \{j : \widetilde{a}_{ij}^{(l)} > 0\}$ contains only a specific subset of intrasubspace features, i.e., $|\mathcal{A}_i| \ll |\mathcal{S}_{y_i}|$, which successfully preserves the subspace structures while enforcing the sparse attention property of the $\ell_0$-norm.
\end{proof}

\subsection{Theoretical Bridge: Investigating Block-Diagonal Structures in Directed Similarity Graphs}
\begin{assumption}
\label{ass:indsub}
[\textbf{Independent Subspaces}]
A collection of $K$ latent subspaces $\{\mathcal{S}_k\}_{k=1}^K$ is linearly independent if for any $k \in \{1, \dots, K\}$, the following holds:
\begin{equation}
\mathcal{S}_k \cap \sum_{j \neq k} \mathcal{S}_j = \{\mathbf{0}\}.
\end{equation}
\end{assumption}

\begin{assumption}
\label{ass:sufsamp}
[\textbf{Sufficient Intrasubspace Sampling}]
For any feature $\mathbf{z}_i^{(l)} \in \mathcal{S}_k$, its neighborhood set $\mathcal{N}_i$ contains at least $d_k$ linearly independent samples from $\mathcal{S}_k$, such that
\begin{equation}
\text{span}\left(\{\mathbf{z}_j^{(l)}\}_{j \in \mathcal{N}_i \cap \mathcal{S}_k}\right) = \mathcal{S}_k,
\end{equation}
where $d_k = \dim(\mathcal{S}_k)$.
\end{assumption}

\begin{assumption}
\label{ass:simsep}
[\textbf{Similarity Separation}]
For any feature $\mathbf{z}_i^{(l)} \in \mathcal{S}_k$, the directed similarities satisfy a margin condition: there exists a sufficiently large margin $\epsilon > 0$ such that
\begin{equation}
\max_{j: \mathbf{z}_j^{(l)} \notin \mathcal{S}_k} \tilde s_{ij}^{(l)}
<
\min_{j: \mathbf{z}_j^{(l)} \in \mathcal{S}_k} \tilde s_{ij}^{(l)} - \epsilon,
\end{equation}
where $\tilde s_{ij}^{(l)}$ denotes the directed similarity between $\mathbf{z}_i^{(l)}$ and $\mathbf{z}_j^{(l)}$.
\end{assumption}

\begin{theorem}
\label{cor:blockdiag}
[\textbf{Sparse Attention Recovers Block-Diagonal Structures}]
Under Assumption~\ref{ass:indsub}, let $\mathbf{A}^{(l)} \in \mathbb{R}^{n \times n}$ be the attention graph produced via the $\alpha$-entmax projection method with $\alpha = 1.5$. If the bilinear similarity $\mathbf{s}_i^{(l)}$ satisfies the separation condition (Assumption~\ref{ass:simsep}),  there exists a permutation matrix $\mathbf{P}$ such that the permuted matrix
\begin{equation}
\widetilde{\mathbf{A}}^{(l)} = \mathbf{P}\mathbf{A}^{(l)}\mathbf{P}^\top
\end{equation}
has a block-diagonal structure, i.e.,
\begin{equation}
\widetilde{\mathbf{A}}^{(l)} =
\begin{pmatrix}
\mathbf{B}_1^{(l)} & \mathbf{0} & \cdots & \mathbf{0} \\
\mathbf{0} & \mathbf{B}_2^{(l)} & \cdots & \mathbf{0} \\
\vdots & & \ddots & \vdots \\
\mathbf{0} & \mathbf{0} & \cdots & \mathbf{B}_K^{(l)}
\end{pmatrix}
\end{equation}
where $\mathbf{B}_k^{(l)} \in \mathbb{R}^{n_k \times n_k}$ denotes the intrasubspace attention submatrix corresponding to $\mathcal{S}_k$.
\end{theorem}

\begin{proof}
Under Assumption~\ref{ass:indsub}, each feature $\mathbf{z}_i^{(l)}$ belongs to a unique subspace $\mathcal{S}_k$. According to Theorem~\ref{thm:relaxation}, the entmax operator with $\alpha > 1$ produces a subspace-preserving sparse representation. Specifically, for each $\mathbf{z}_i^{(l)} \in \mathcal{S}_k$, we have
\begin{equation}
\widetilde{a}_{ij}^{(l)} = 0, \quad \forall j \text{ such that } \mathbf{z}_j^{(l)} \notin \mathcal{S}_k.
\end{equation}
Therefore, the nonzero entries of $\mathbf{A}^{(l)}$ are confined to pairs $(i,j)$ with $\mathbf{z}_i^{(l)}, \mathbf{z}_j^{(l)} \in \mathcal{S}_k$ for some $k$.

We consider a permutation matrix $\mathbf{P}$ that reorders the indices such that all samples acquired from the same subspace are grouped together. Let
\begin{equation}
\widetilde{\mathbf{A}}^{(l)} = \mathbf{P} \mathbf{A}^{(l)} \mathbf{P}^\top.
\end{equation}
Since the cross-subspace entries are zero, $\widetilde{\mathbf{A}}^{(l)}$ has a block-diagonal structure, i.e.,
\begin{equation}
\widetilde{\mathbf{A}}^{(l)} =
\begin{pmatrix}
\mathbf{B}_1^{(l)} & \mathbf{0} & \cdots & \mathbf{0} \\
\mathbf{0} & \mathbf{B}_2^{(l)} & \cdots & \mathbf{0} \\
\vdots & & \ddots & \vdots \\
\mathbf{0} & \mathbf{0} & \cdots & \mathbf{B}_K^{(l)}
\end{pmatrix},
\end{equation}
where $\mathbf{B}_k^{(l)} \in \mathbb{R}^{n_k \times n_k}$ corresponds to the intrasubspace attention structure contained in $\mathcal{S}_k$. Thus, $\mathbf{A}^{(l)}$ is block-diagonal up to a permutation of the indices.
\end{proof}

Theorem~\ref{cor:blockdiag} suggests that the sparse attention graph $\mathbf{A}^{(l)}$ can recover a complete block-diagonal structure if two conditions are satisfied: sufficient intrasubspace sampling (Assumption~\ref{ass:sufsamp}) and the similarity separation condition (Assumption~\ref{ass:simsep}). Under these conditions, the $\alpha$-entmax projection maps all cross-subspace attention weights to exactly zero regardless of feature positions. Because the memberships of features in different subspaces are determined by relative similarities rather than absolute locations, the block-diagonal structure is preserved for unseen test samples. This guarantees consistent clustering performance across both training and test sets.

\subsection{Theoretical Bridge: From the $\ell_0$-Norm to Simplex Projection}
The sparse self-representation lies in the discrete $\ell_0$-norm constraint in Eq.~\eqref{eq:lzsimplex}, which strictly limits the number of nonzero reconstruction coefficients, i.e., $\|\mathbf{a}_i^{(l)}\|_0 \ll n$. However, directly optimizing the $\ell_0$-norm over high-dimensional data is an NP-hard and strictly nondifferentiable problem. The traditional continuous relaxation schemes, such as the $\ell_1$-norm regularization, often require computationally prohibitive iterative solvers, e.g., an alternating direction method of multipliers framework \cite{Boyd2011ADMM}. To overcome this limitation, SAGL formulates the construction of the sparse attention graph as a Tsallis entropy-regularized projection onto the probability simplex $\Delta^{n-1}$. We provide the following theorem to establish the equivalence between the continuous simplex projection and the discrete $\ell_0$-norm constraint.

\begin{theorem}
\label{thm:lzbridge}
[\textbf{Differentiable Relaxation and Support Sparsity of the $\ell_0$-Norm Constraint}]
Let $\tilde{\mathbf{s}}_i^{(l)} \in \mathbb{R}^n$ be the directed similarity vector for the $i$th feature. The structured sparse projection imposed on the probability simplex $\Delta^{n-1}$ via the $\alpha$-entmax transformation acts as a continuous and differentiable relaxation of the discrete $\ell_0$-norm constraint. Specifically, it guarantees a strictly sparse support set $\text{supp}(\mathbf{a}_i^{(l)}) = \{j \mid a_{ij}^{(l)} > 0\}$, where the discrete cardinality $\|\mathbf{a}_i^{(l)}\|_0$ is exactly determined by the continuous topological uncertainty factor $\tau_i^{(l)}$.
\end{theorem}

\begin{proof}
The $\alpha$-entmax transformation computes the attention distribution by solving a constrained variational optimization problem over the probability simplex $\Delta^{n-1}$:
\begin{equation}
\mathbf{a}_i^{(l)} = \mathop{\arg\max}_{\mathbf{a}_i^{(l)} \in \Delta^{n-1}} \left( \left( {\mathbf{a}_i^{(l)}} \right)^\top \tilde{\mathbf{s}}_i^{(l)} + \mathcal{H}_\alpha^T \left(\mathbf{a}_i^{(l)} \right) \right)
\end{equation}
where the Tsallis $\alpha$-entropy for $\alpha > 1$ is defined as
\begin{equation}
\mathcal{H}_\alpha^T(\mathbf{a}_i^{(l)}) = \frac{1}{\alpha(\alpha-1)} \sum_{j} \left(\mathbf{a}_{ij}^{(l)} - \left( \mathbf{a}_{ij}^{(l)} \right)^\alpha \right).
\end{equation}

To reveal its connection with the $\ell_0$-norm, we formulate the Lagrangian of this optimization problem by incorporating the probability simplex constraints ($\mathbf{1}^\top \mathbf{a}_i^{(l)} = 1$ and $\mathbf{a}_i^{(l)} \ge \mathbf{0}$):
\begin{equation}
\begin{aligned}
\mathcal{L}(\mathbf{a}_i^{(l)}, \rho, \boldsymbol{\mu}) = & \sum_{j=1}^n a_{ij}^{(l)} \tilde{s}_{ij}^{(l)} + \frac{1}{\alpha(\alpha-1)} \sum_{j=1}^n  \left( a_{ij}^{(l)} - \left( a_{ij}^{(l)} \right)^\alpha  \right) \\
& - \rho \left(\sum_{j=1}^n a_{ij}^{(l)} - 1\right) + \sum_{j=1}^n \mu_j a_{ij}^{(l)}
\end{aligned}
\end{equation}
where $\rho \in \mathbb{R}$ and $\boldsymbol{\mu} \in \mathbb{R}_+^n$ are the Lagrange multipliers for the unit-sum and non-negativity constraints, respectively.

Setting the partial derivative with respect to $a_{ij}^{(l)}$ to zero gives the Karush-Kuhn-Tucker (KKT) stationarity condition:
\begin{equation}
\frac{\partial \mathcal{L}}{\partial a_{ij}^{(l)}} = \tilde{s}_{ij}^{(l)} + \frac{1 - \alpha \left(a_{ij}^{(l)}\right)^{\alpha-1}}{\alpha(\alpha-1)} - \rho + \mu_j = 0.
\end{equation}

By the KKT complementary slackness condition, $\mu_j = 0$ whenever $a_{ij}^{(l)} > 0$. Upon solving for $a_{ij}^{(l)}$ under this active support condition, we obtain
\begin{equation}
a_{ij}^{(l)} = \left[ \left(\alpha-1\right) \left( \tilde{s}_{ij}^{(l)} - \rho \right) + \frac{1}{\alpha} \right]^{\frac{1}{\alpha-1}}.
\end{equation}

Letting $\rho^* = (\alpha-1)\rho - \frac{1}{\alpha}$ and applying the truncation operator $[x]_+ = \max(x, 0)$ to enforce nonnegativity, we have the exact closed-form solution:
\begin{equation}
a_{ij}^{(l)} = \left[ \left(\alpha-1\right) \left( \tilde{s}_{ij}^{(l)} \right) - \rho^* \right]_+^{\frac{1}{\alpha-1}}.
\end{equation}

By substituting the similarity $\tilde{{s}}_{ij}^{(l)} = {s}_{ij}^{(l)} \cdot \left( 1 - \omega_i^{(l)} \right)$ in Eq.~\eqref{eq:sim} into the above equation, we obtain
\begin{equation}
a_{ij}^{(l)} = \left[ \left(\alpha-1\right) \left( {s}_{ij}^{(l)} \cdot \left( 1 - \omega_i^{(l)} \right) \right) - \rho^* \right]_+^{\frac{1}{\alpha-1}}.
\end{equation}

The truncation operator $[\cdot]_+$ sets any similarity that is below the threshold $\rho^*$ to zero, which directly bounds the $\ell_0$-norm such that $\|\mathbf{a}_i^{(l)}\|_0 \ll n$. Moreover, as the feature-specific compression factor $\omega_i^{(l)}$ increases toward $1$, the variance of the similarities increases. This results in the threshold operator to prune more connections. Therefore, the number of nonzero entries of the $\ell_0$-norm is continuously controlled through the differentiable simplex projection.
\end{proof}

Theorem~\ref{thm:lzbridge} shows that the $\alpha$-entmax projection onto the probability simplex $\Delta^{n-1}$ provides an exact differentiable relaxation of the discrete $\ell_0$-norm constraint. This eliminates the need for computationally expensive iterative solvers, enabling efficient end-to-end optimization of sparse attention graphs.

%% Table I %%
\begin{table}[t]
\caption{Statistics of the datasets.}
\label{tb:datasets}
\setlength{\tabcolsep}{4pt}
\begin{center}
\begin{scriptsize}
\begin{sc}
\begin{tabular}{lccr}
\toprule
Name  & Train & Test & Classes  \\
\midrule
Pets & 3,680 & 3,669 & 37  \\
KITTI & 5,985 & 1,496 & 4  \\
Flowers & 2,040 & 6,149 & 102   \\
Caltech101  & 3,060 & 6,084 & 102 \\
EuroSAT & 10,000 & 5,000 & 10  \\
SUN397 & 19,850 & 19,850 & 397  \\
Food101 & 75,750 & 25,250 & 101  \\
ImageNet-1K & 1,281,167 & 50, 000 & 1,000  \\
\bottomrule
\end{tabular}
\end{sc}
\end{scriptsize}
\end{center}
\end{table}

\begin{table*}[!htbp]
\tiny
\setlength{\tabcolsep}{2pt}
\centering
\caption{CKA scores for various pretrained model pairs}
\label{tb:cka}
\begin{tabular}{cccc|cccccccc}
\hline
DINOv3 & SigLIP 2 & CLIP ViT-L/14 & ConvNeXt V2 & Pets & KITTI & Flowers & Caltech101 & EuroSAT & SUN397 & Food101 & ImageNet-1K \\
\hline
\checkmark & \checkmark & ~ & ~ & 0.63 & 0.84 & 0.65 & 0.62 & 0.70 & 0.49 & 0.67 & 0.46 \\
\checkmark & ~ & \checkmark & ~ &  0.59 & 0.81 & 0.61 & 0.57 & 0.70 & 0.48 & 0.64 & 0.41 \\
\checkmark & ~ & ~ & \checkmark & 0.06 & 0.19 & 0.15 & 0.07 & 0.16 & 0.07 & 0.04 & 0.06 \\
~ & \checkmark & \checkmark & ~ & 0.85 & 0.86 & 0.76 & 0.79 & 0.70 & 0.75 & 0.78 & 0.79 \\
~ & \checkmark & ~ &  \checkmark& 0.05 & 0.18 & 0.25 & 0.09 & 0.18 & 0.08 & 0.07 & 0.09 \\
\hline
\end{tabular}
\end{table*}

\begin{figure}[!htbp]
\centering
\includegraphics[width=0.6\linewidth]{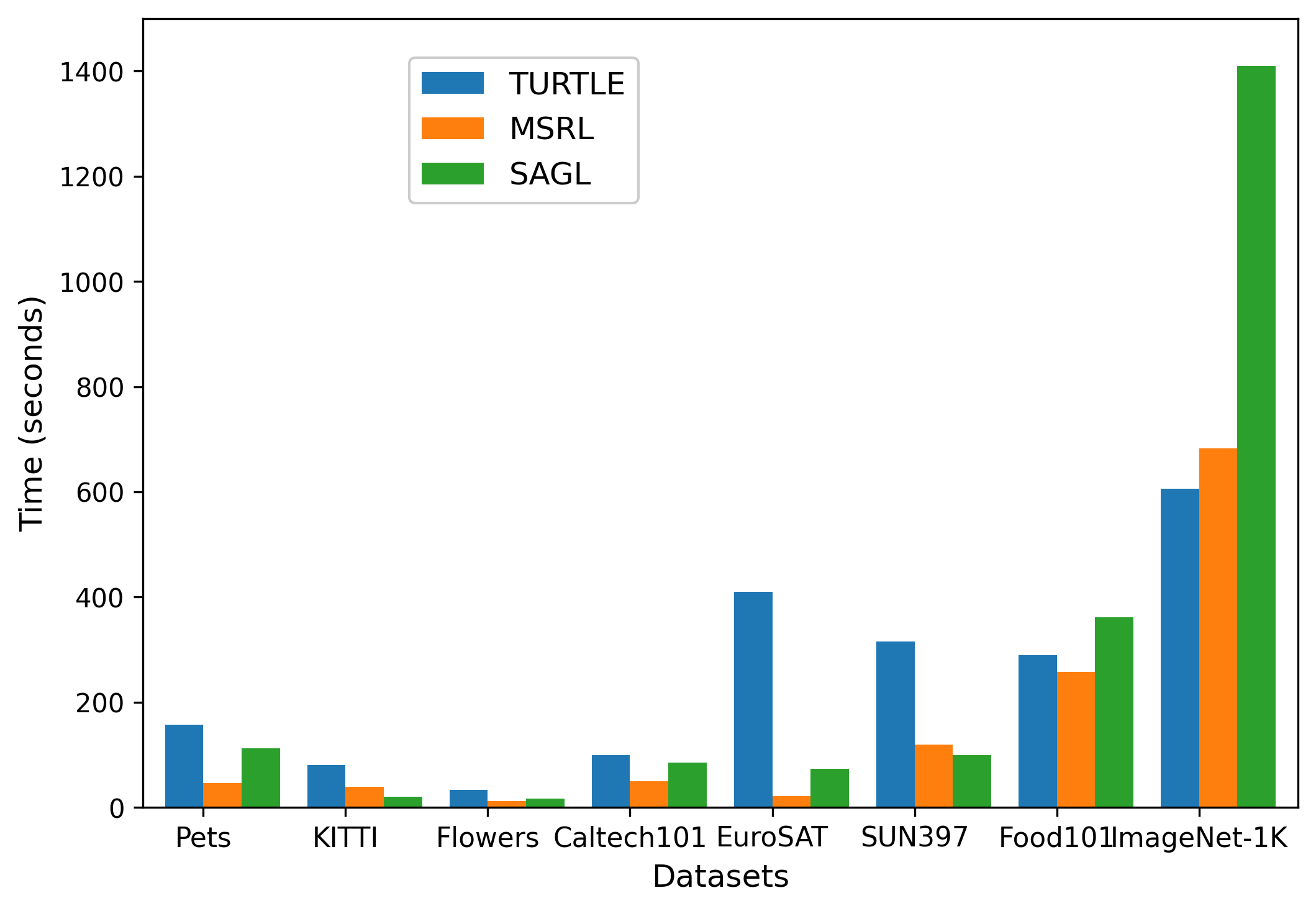}
\caption{Training times (in seconds) of TURTLE, MSRL, and SAGL on eight vision datasets.}
\label{fig:timecost} %% label for entire figure
\end{figure}

\begin{figure*}[htbp]
    \centering
    \subfigure[The original features]{
        \includegraphics[width=0.3\textwidth]{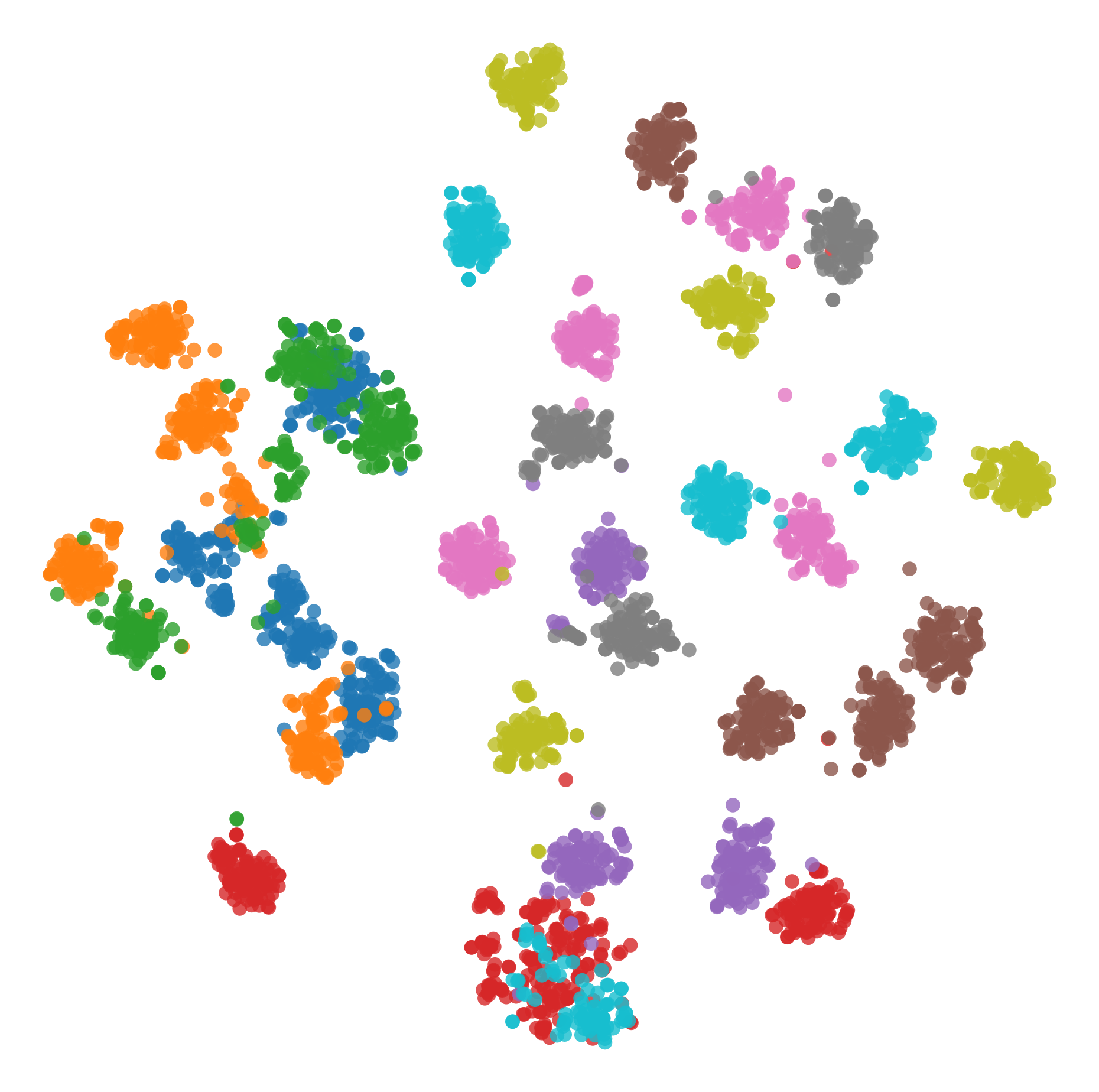}
    }
    \subfigure[The linear features]{
        \includegraphics[width=0.3\textwidth]{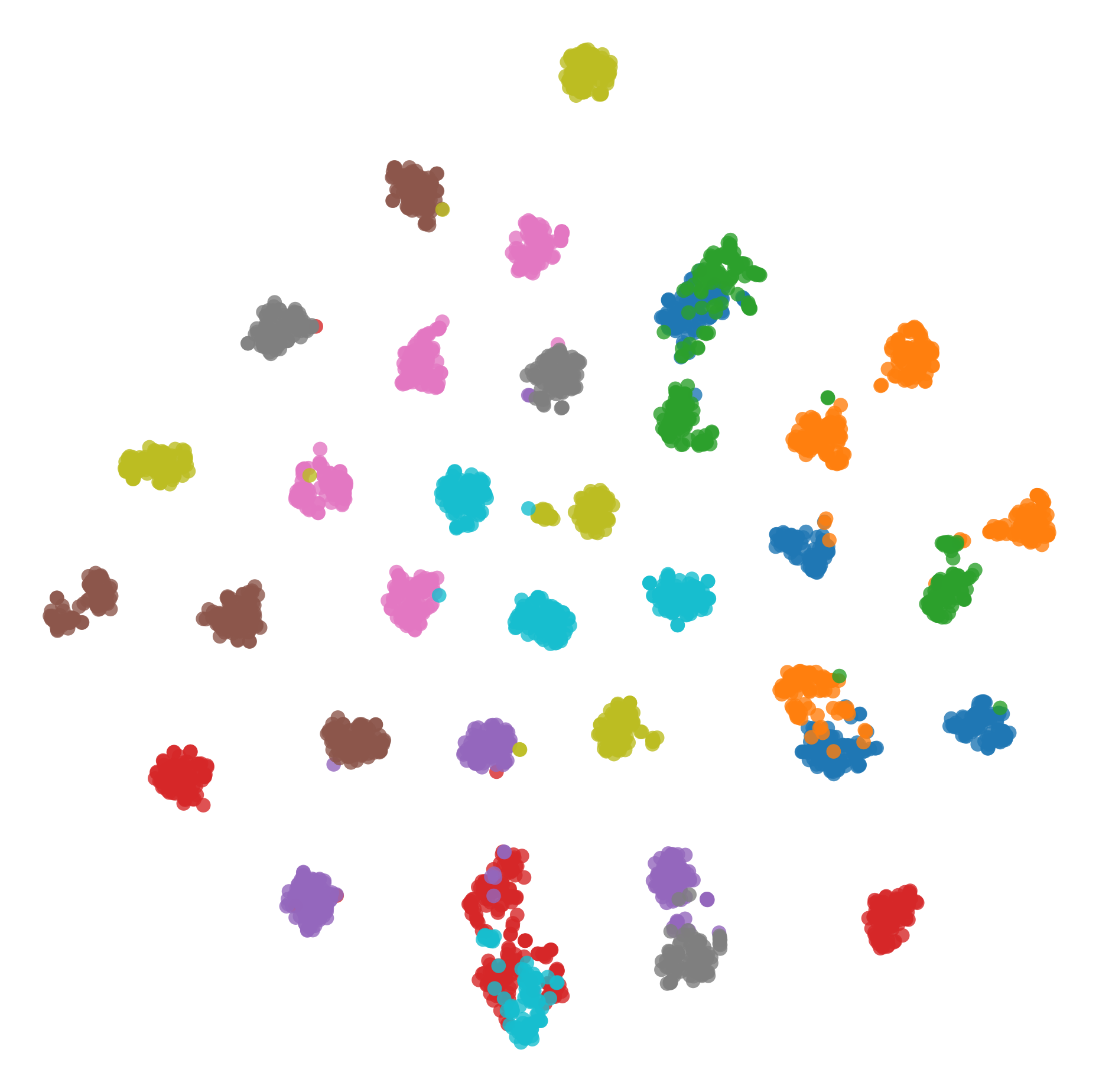}
    }
     \subfigure[The corresponding representations]{
        \includegraphics[width=0.3\textwidth]{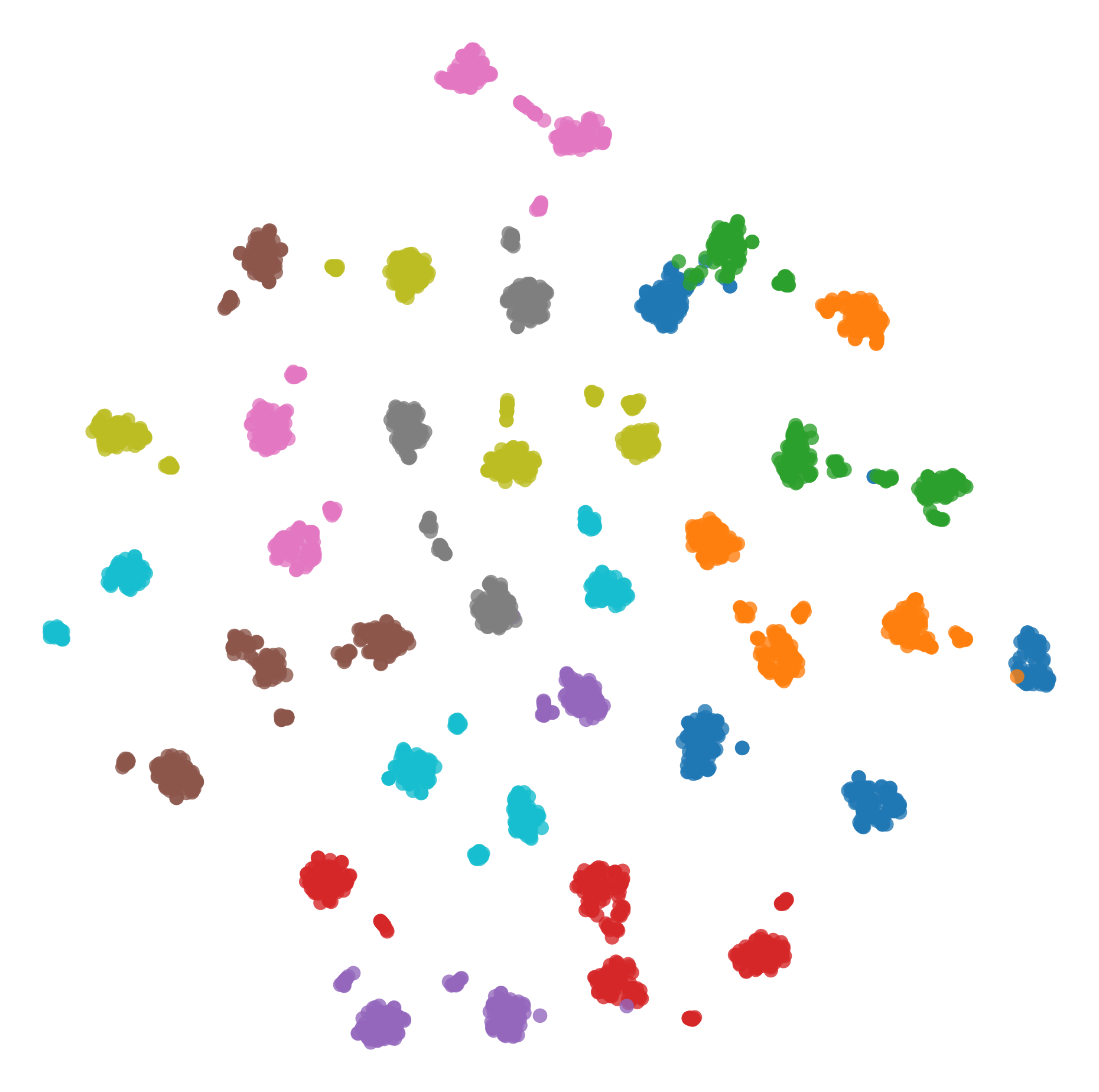}
    }
    \caption{The t-SNE visualization of three levels of features generated by SAGL on the Pets dataset, where the original features are extracted using SigLIP 2.}
    \label{fig:sne:pets1}
\end{figure*}

\begin{figure*}[htbp]
    \centering
    \subfigure[The original features]{
        \includegraphics[width=0.3\textwidth]{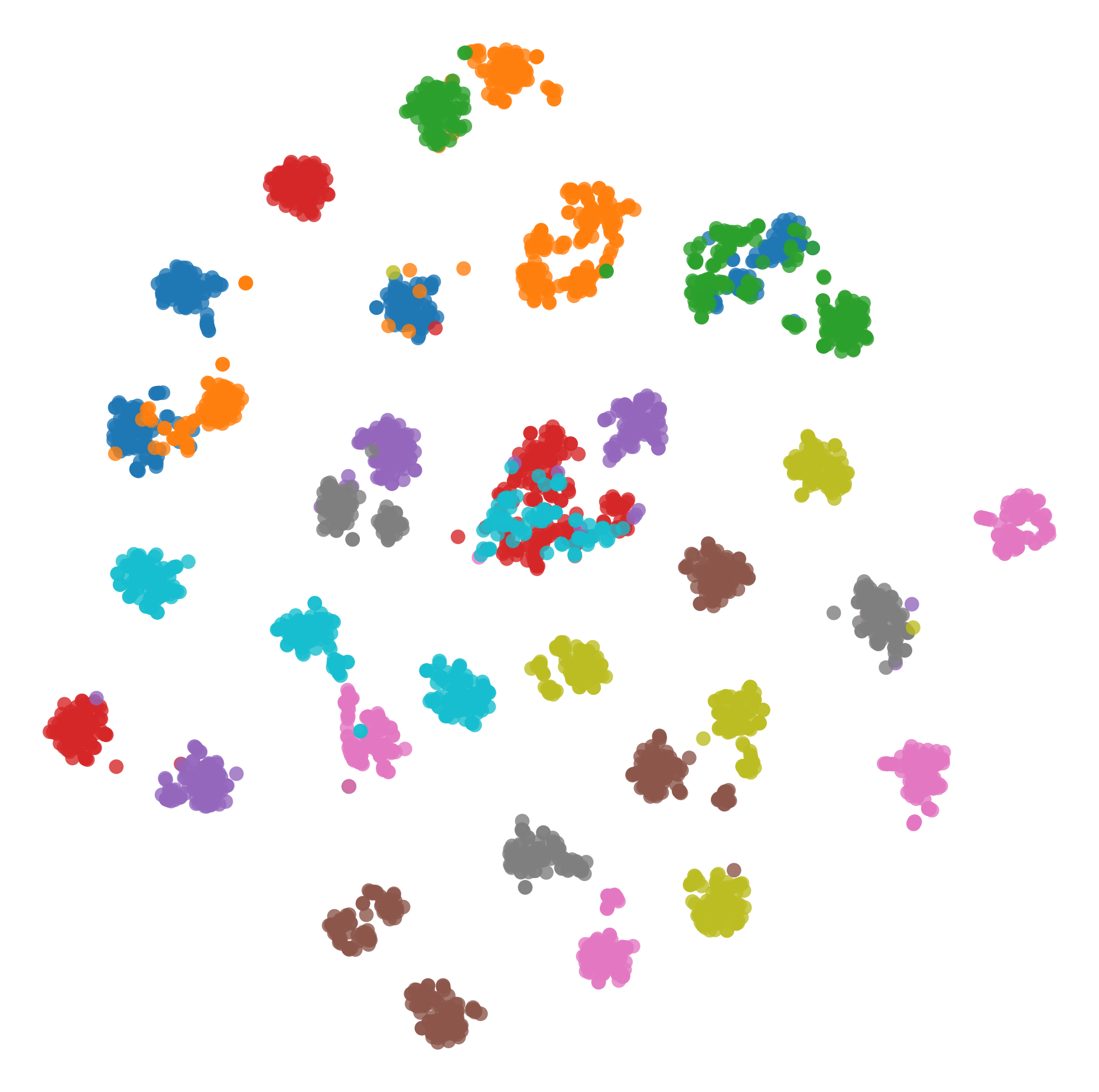}
    }
    \subfigure[The linear features]{
        \includegraphics[width=0.3\textwidth]{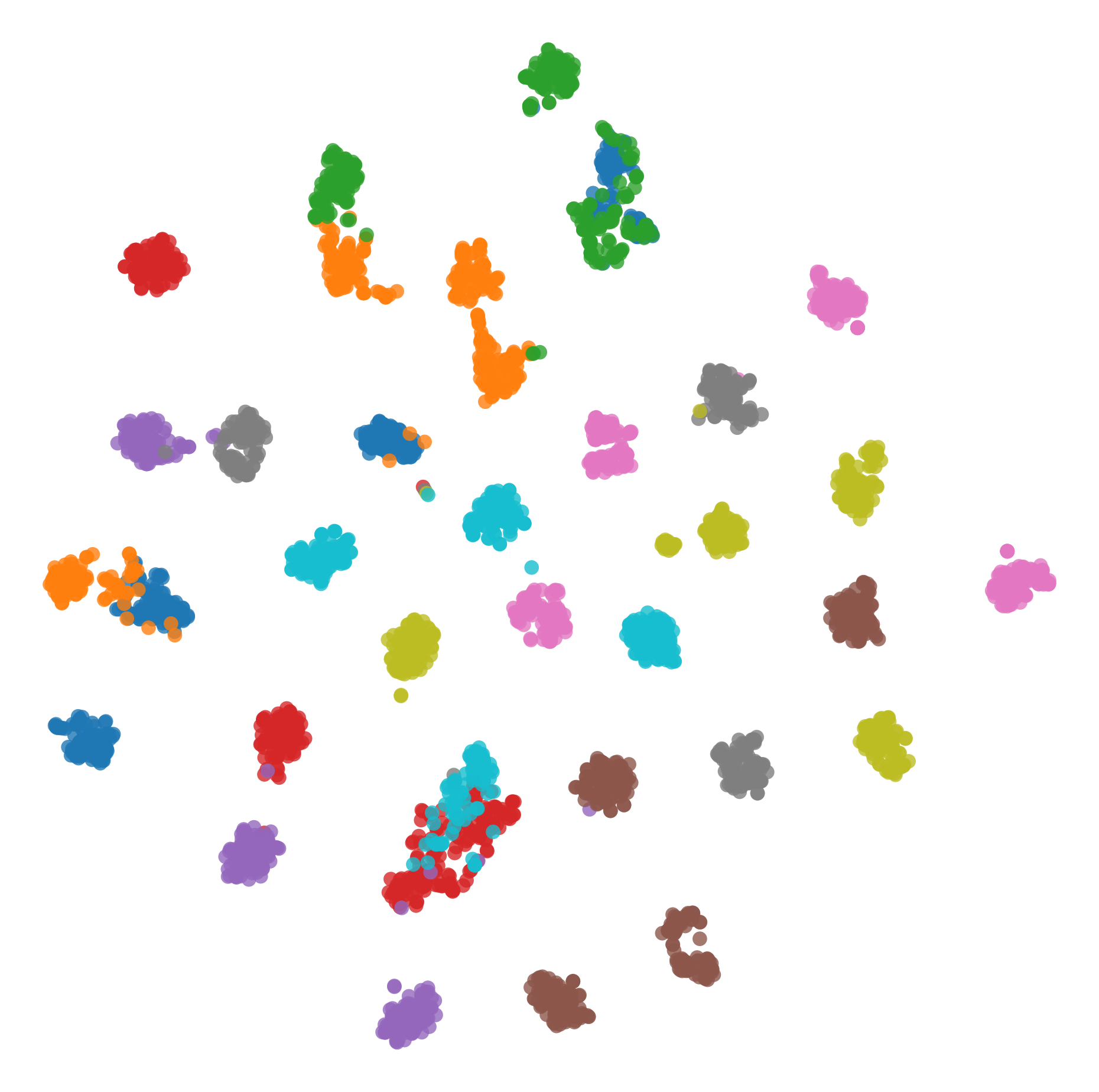}
    }
     \subfigure[The corresponding representations]{
        \includegraphics[width=0.3\textwidth]{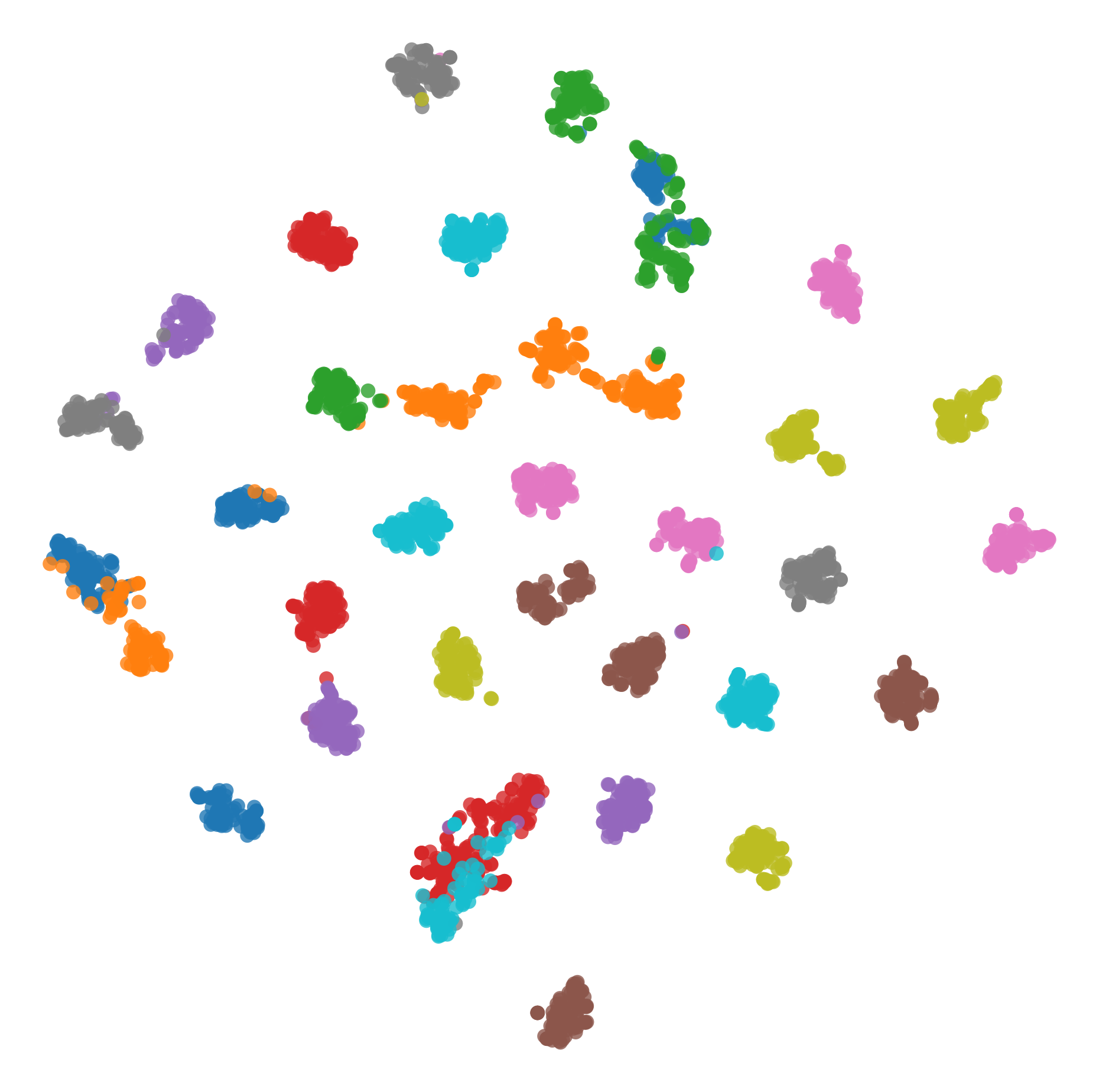}
    }
    \caption{The t-SNE visualization of three levels of features generated by SAGL on the Pets dataset, where the original features are extracted using DINOv3.}
    \label{fig:bs:pets2}
\end{figure*}

\begin{figure*}[htbp]
    \centering
    \subfigure[The original features]{
        \includegraphics[width=0.3\textwidth]{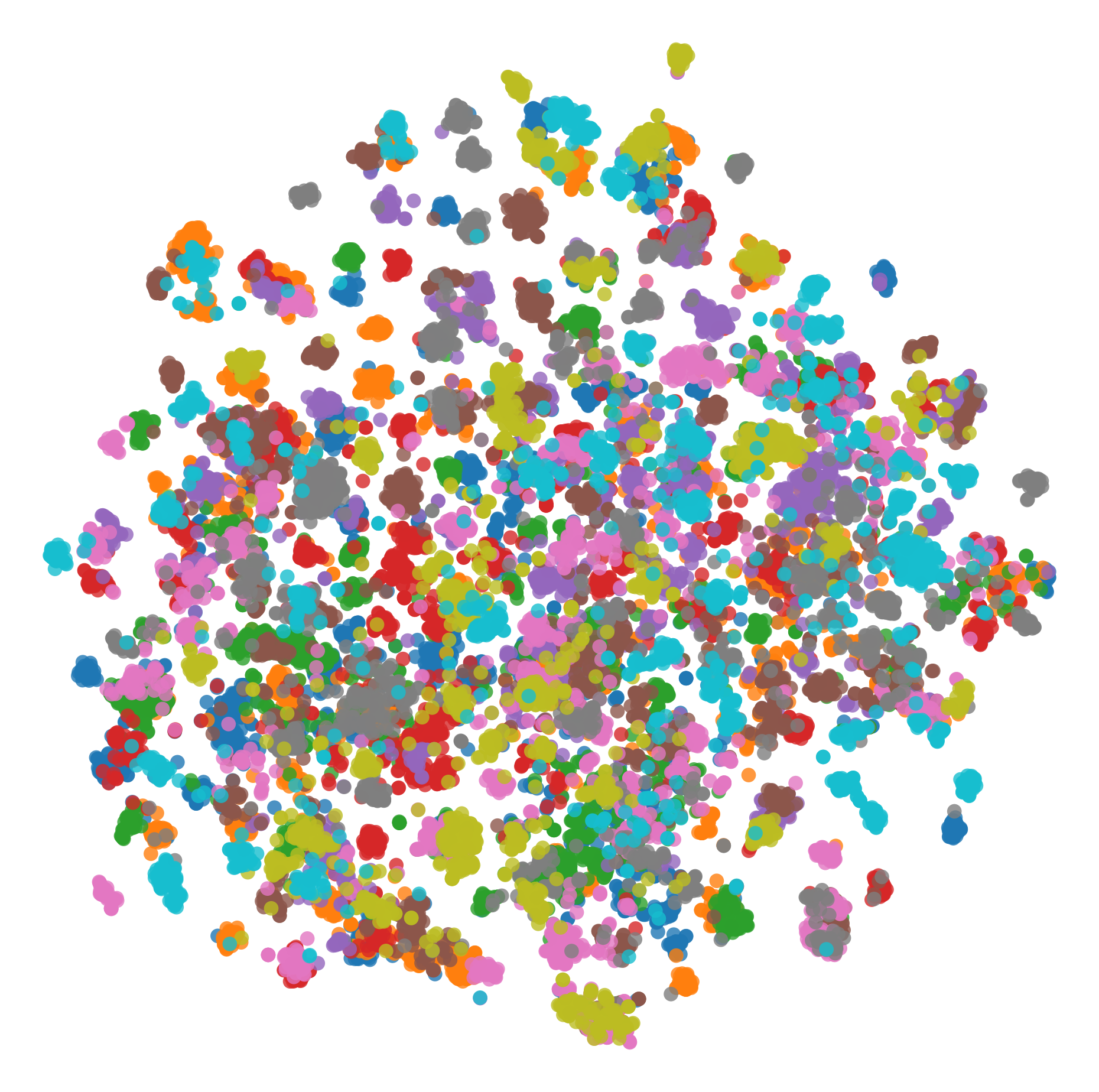}
    }
    \subfigure[The linear features]{
        \includegraphics[width=0.3\textwidth]{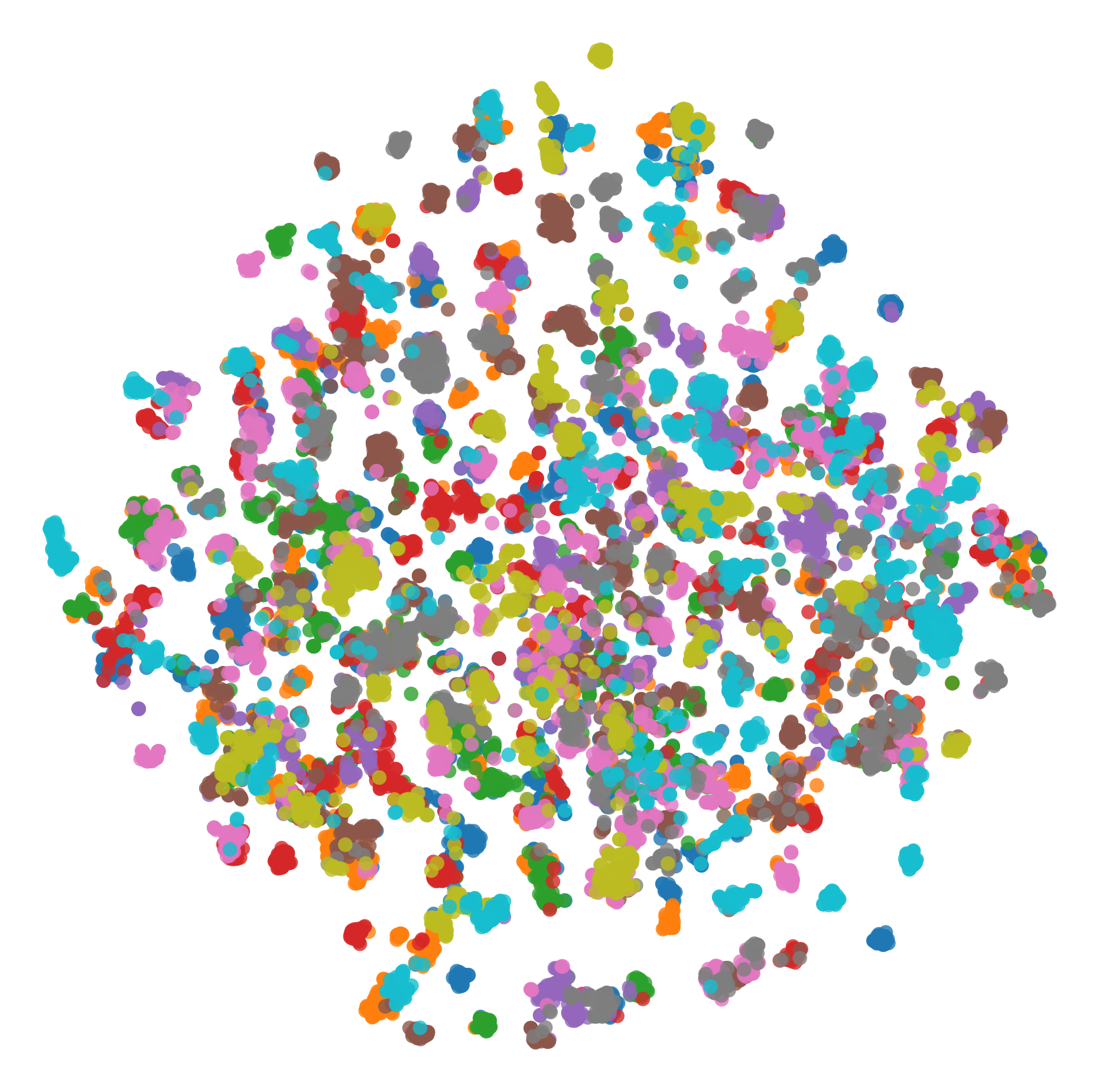}
    }
     \subfigure[The corresponding representations]{
        \includegraphics[width=0.3\textwidth]{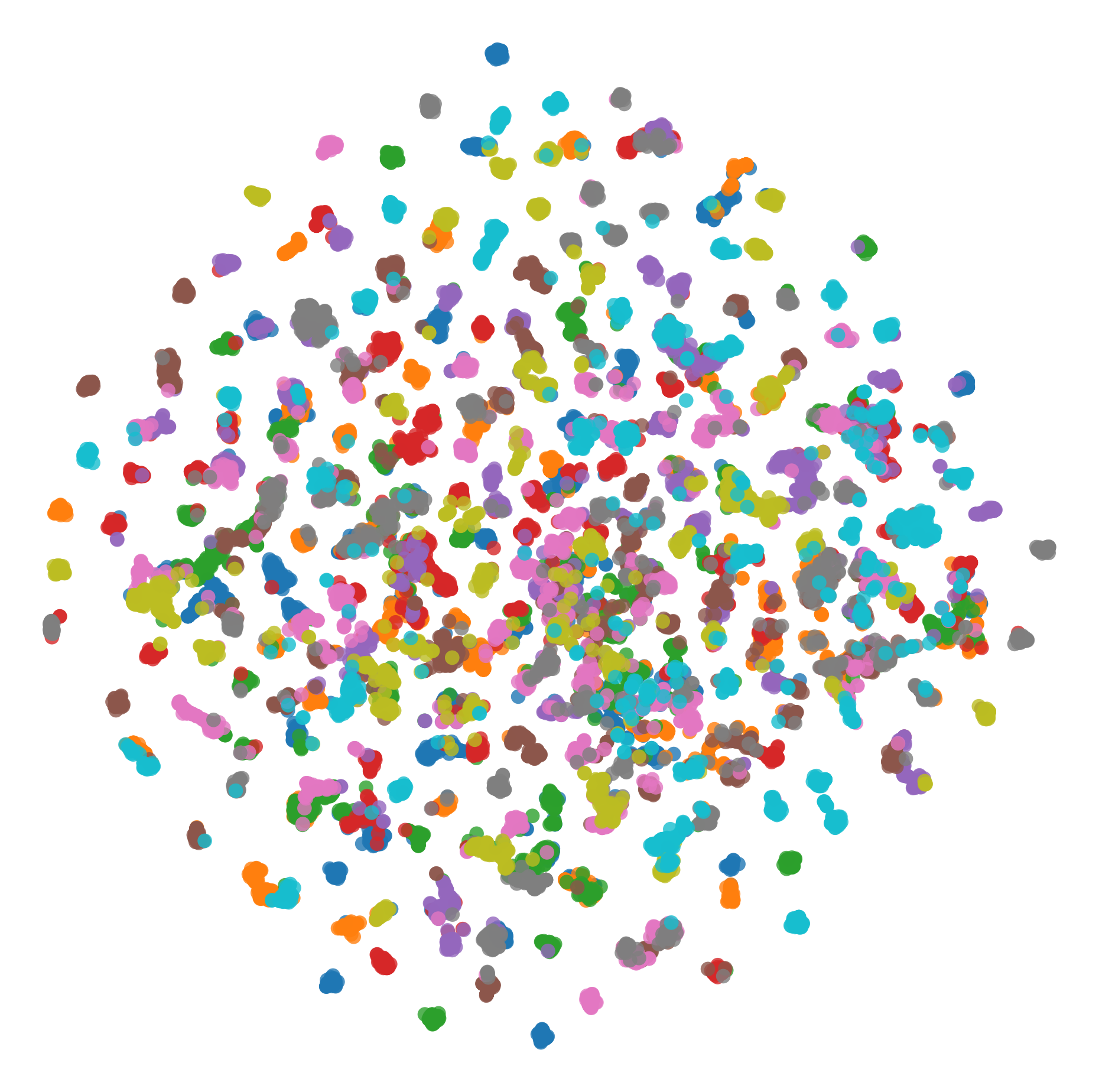}
    }
    \caption{The t-SNE visualization of three levels of features generated by SAGL on the SUN397 dataset, where the original features are extracted using SigLIP 2.}
    \label{fig:sne:sun1}
\end{figure*}

\begin{figure*}[htbp]
    \centering
    \subfigure[The original features]{
        \includegraphics[width=0.3\textwidth]{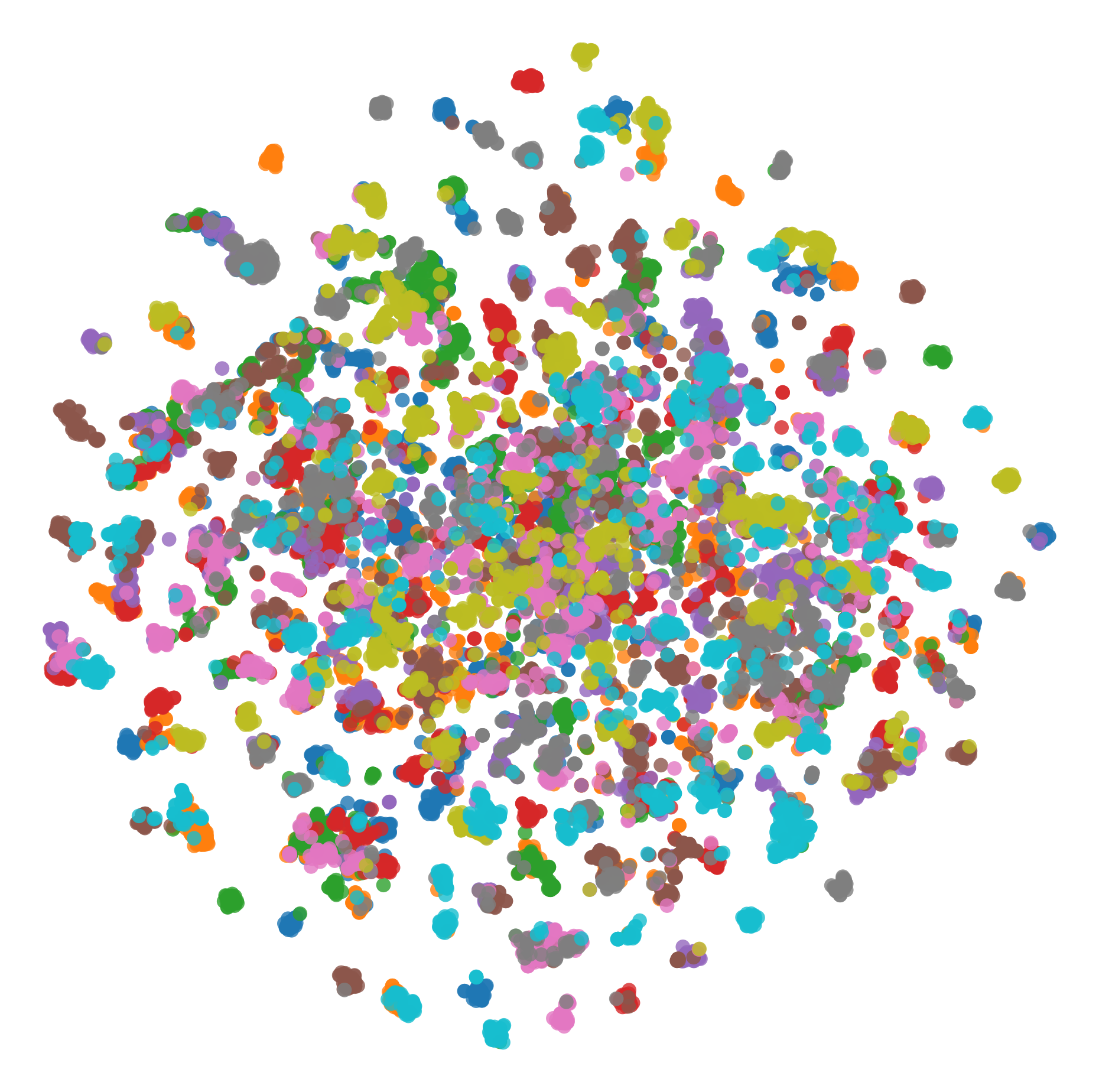}
    }
    \subfigure[The linear features]{
        \includegraphics[width=0.3\textwidth]{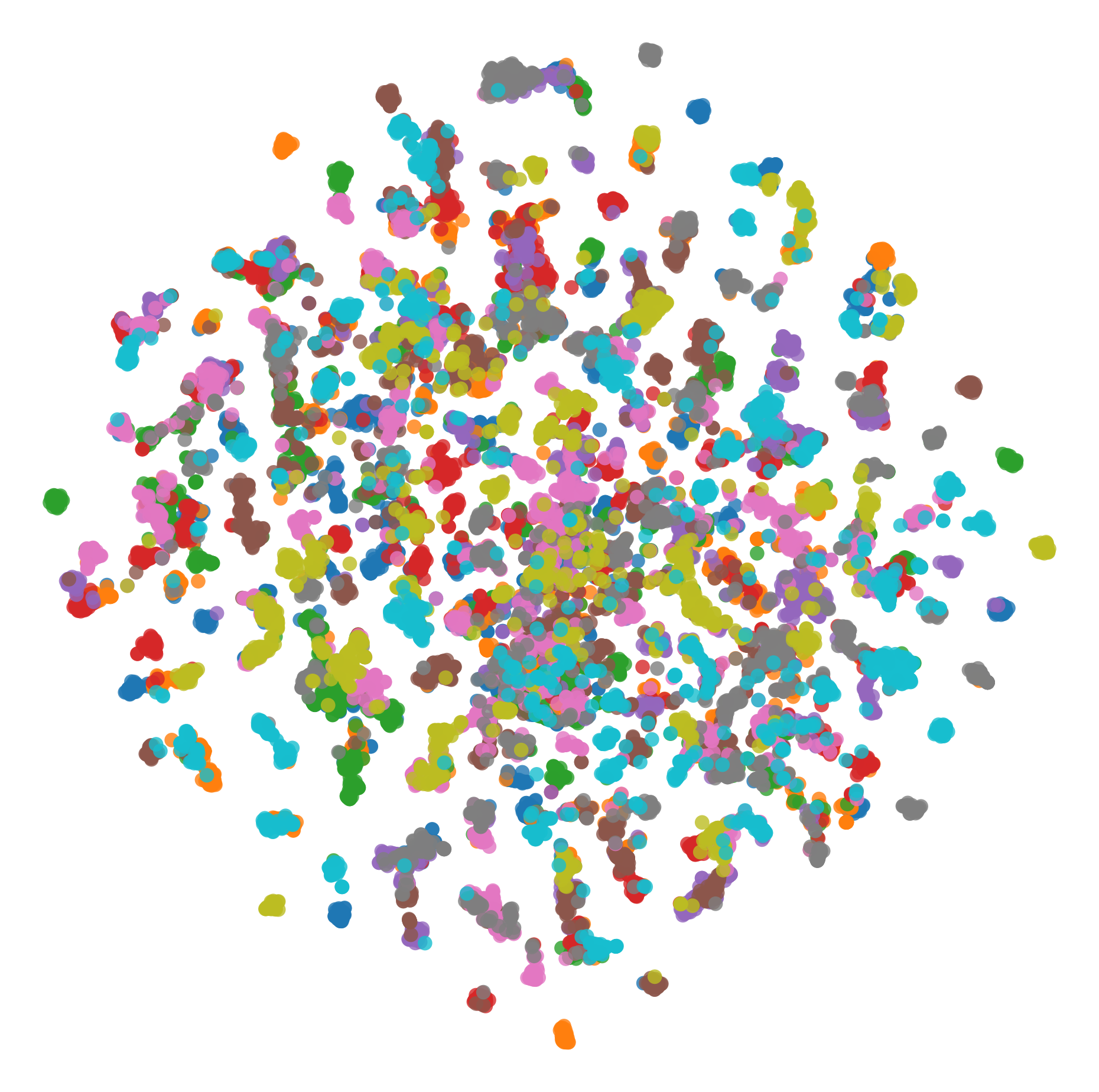}
    }
     \subfigure[The reconstructed representations]{
        \includegraphics[width=0.3\textwidth]{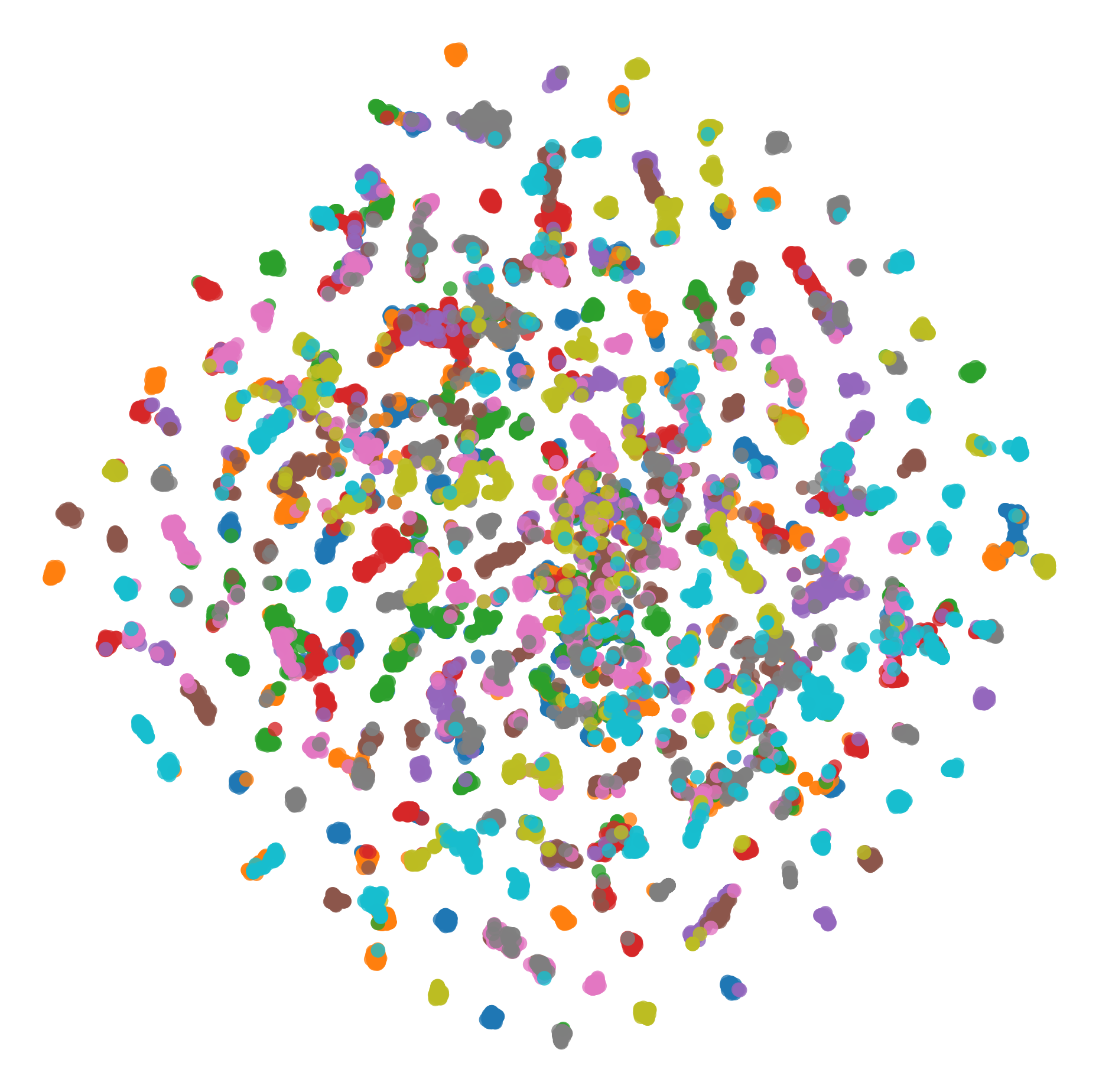}
    }
    \caption{The t-SNE visualization of three levels of features generated by SAGL on the SUN397 dataset, where the original features are extracted using DINOv3.}
    \label{fig:bs:sun2}
\end{figure*}

\begin{figure*}[htbp]
    \centering
    \subfigure[The original features]{
        \includegraphics[width=0.3\textwidth]{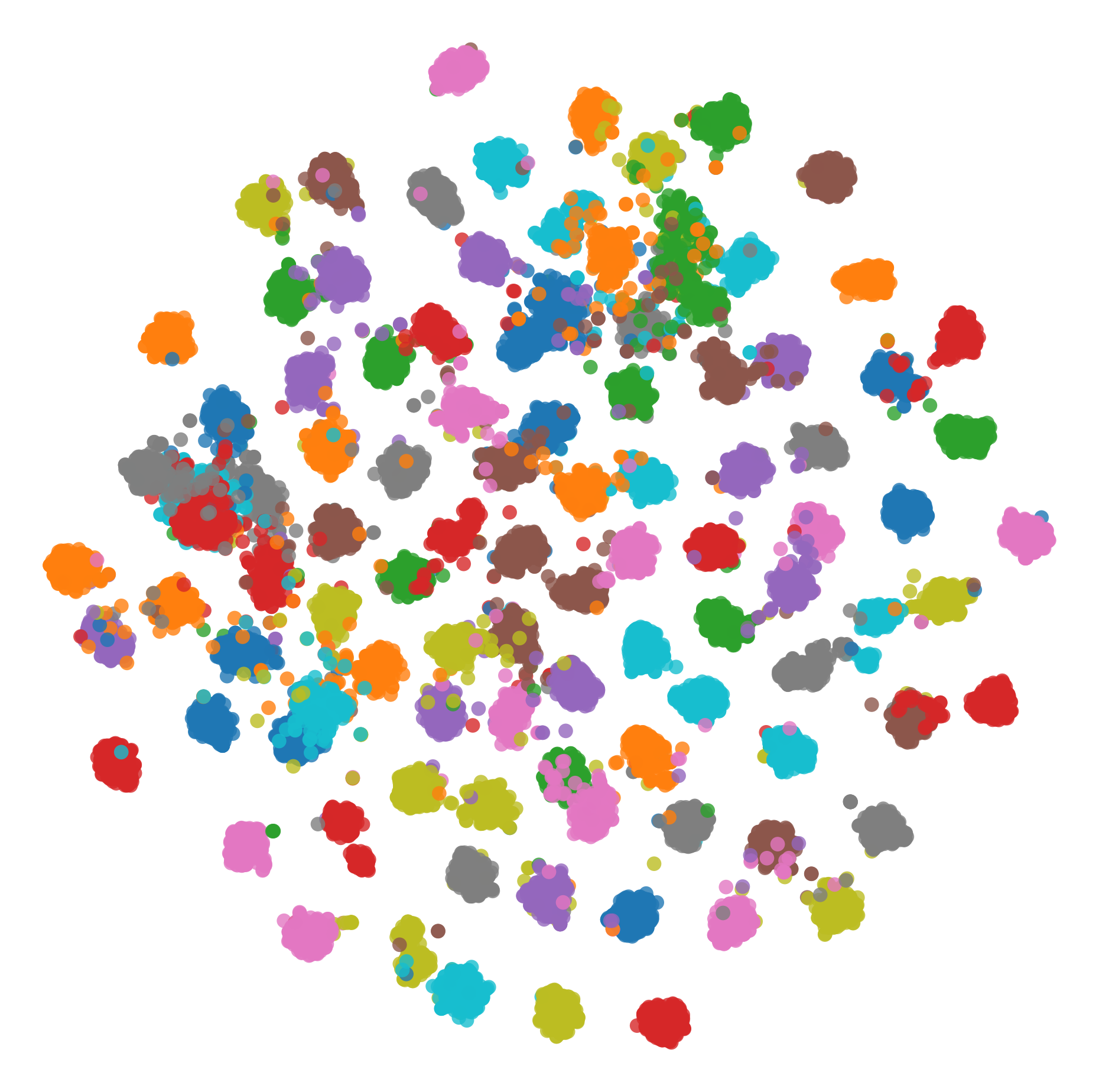}
    }
    \subfigure[The linear features]{
        \includegraphics[width=0.3\textwidth]{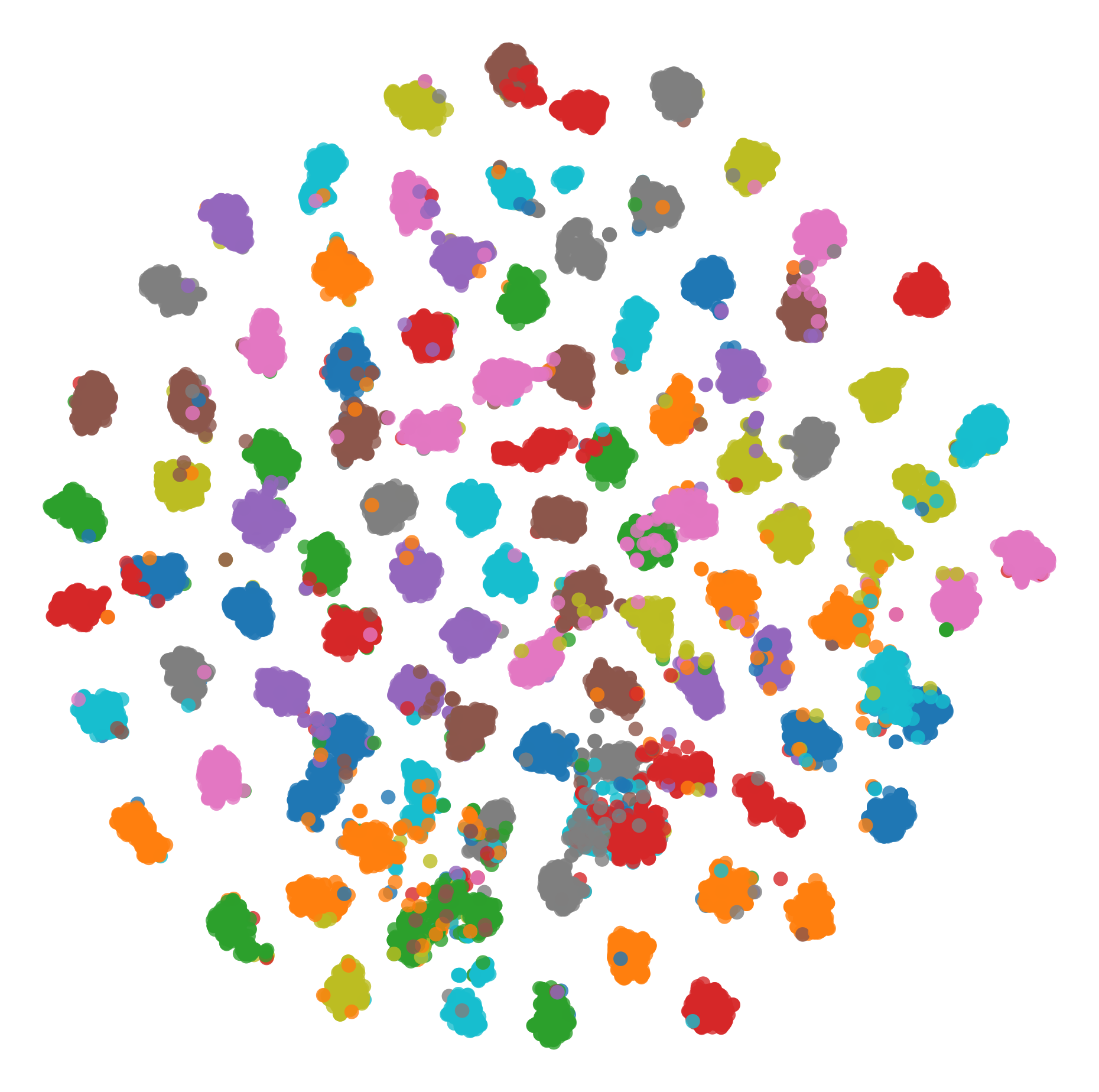}
    }
     \subfigure[The reconstructed representations]{
        \includegraphics[width=0.3\textwidth]{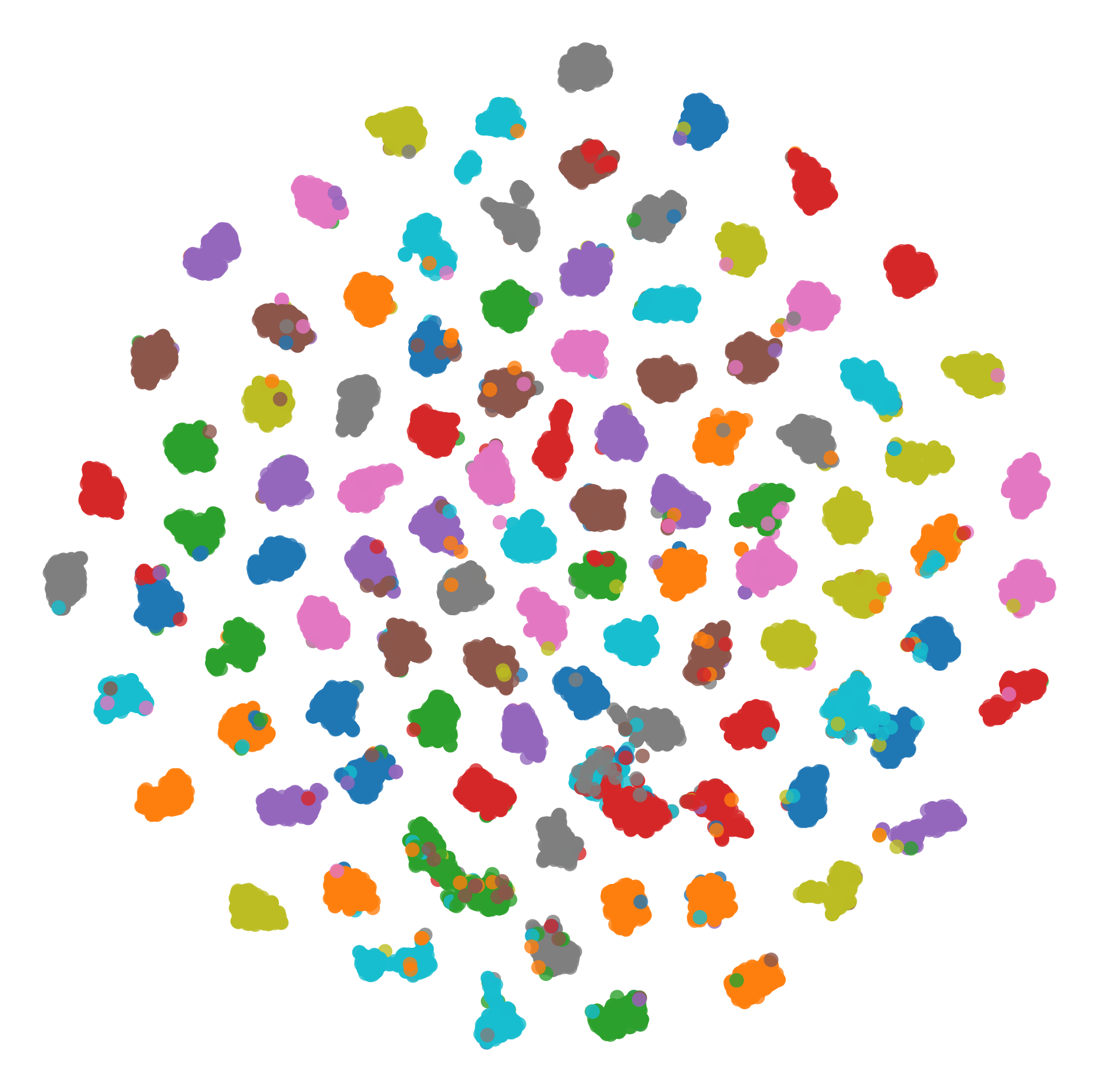}
    }
    \caption{The t-SNE visualization of three levels of features generated by SAGL on the Food101 dataset, where the original features are extracted using SigLIP 2.}
    \label{fig:sne:food1}
\end{figure*}

\begin{figure*}[htbp]
    \centering
    \subfigure[The original features]{
        \includegraphics[width=0.3\textwidth]{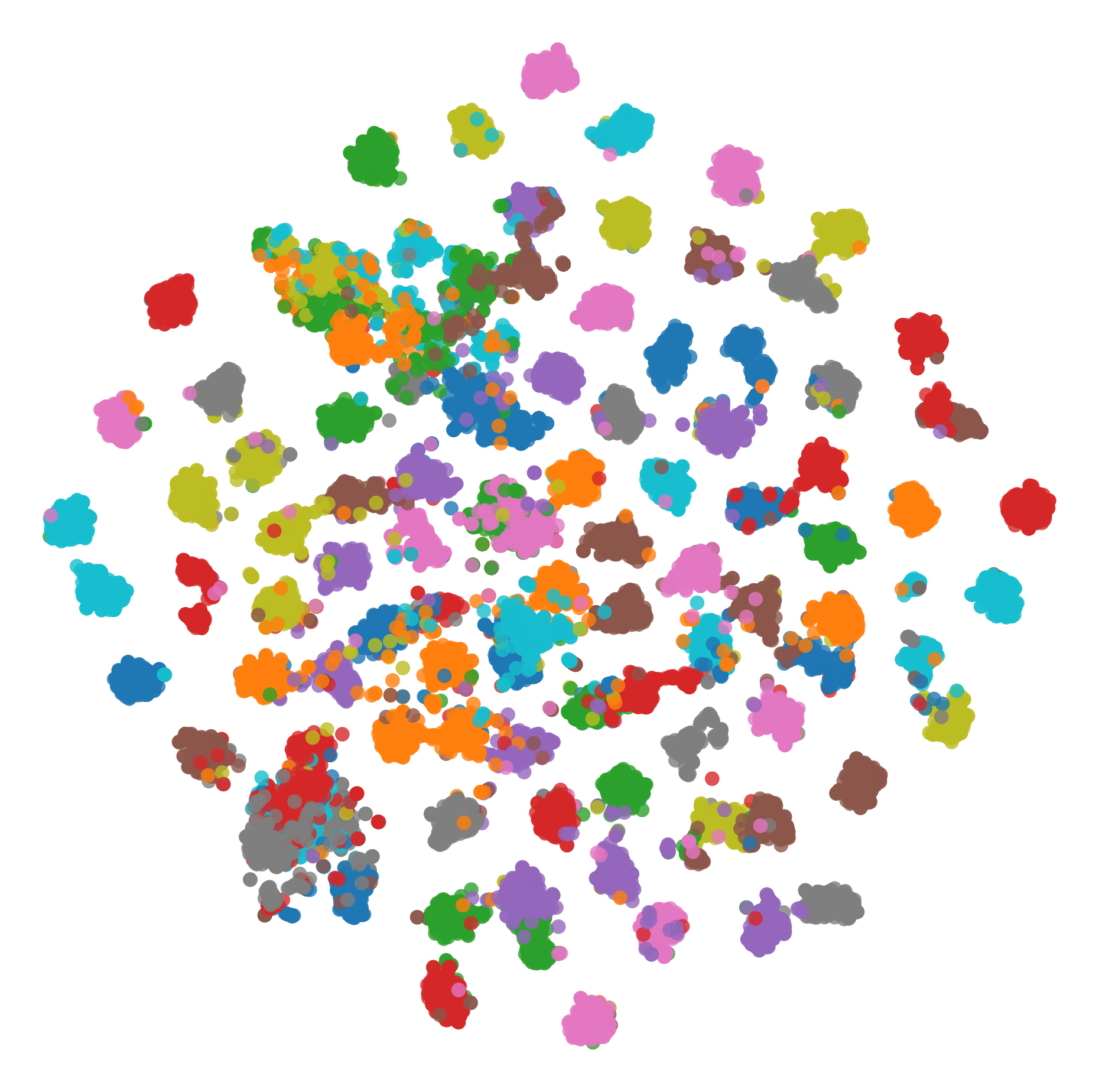}
    }
    \subfigure[The linear features]{
        \includegraphics[width=0.3\textwidth]{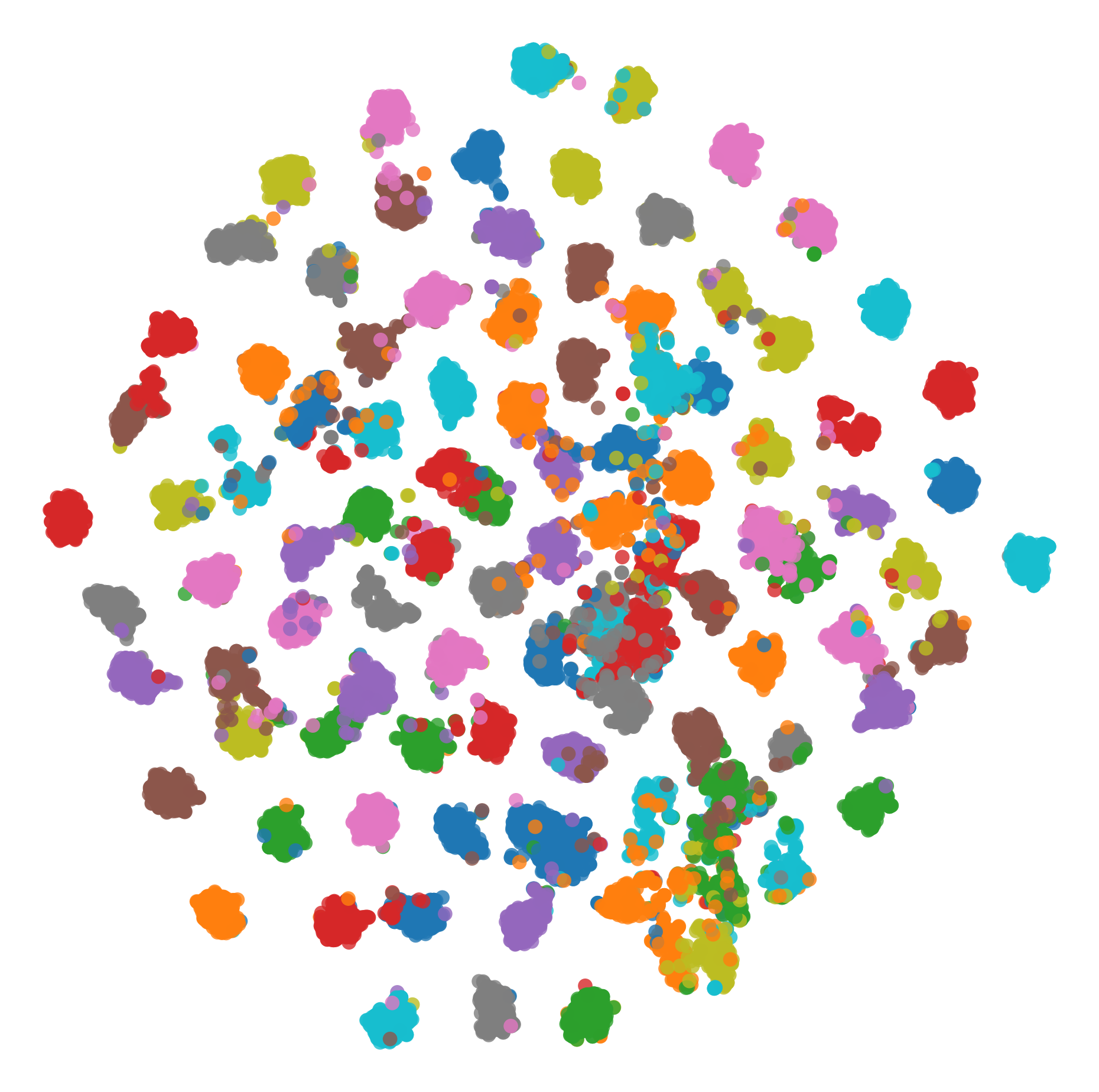}
    }
     \subfigure[The reconstructed representations]{
        \includegraphics[width=0.3\textwidth]{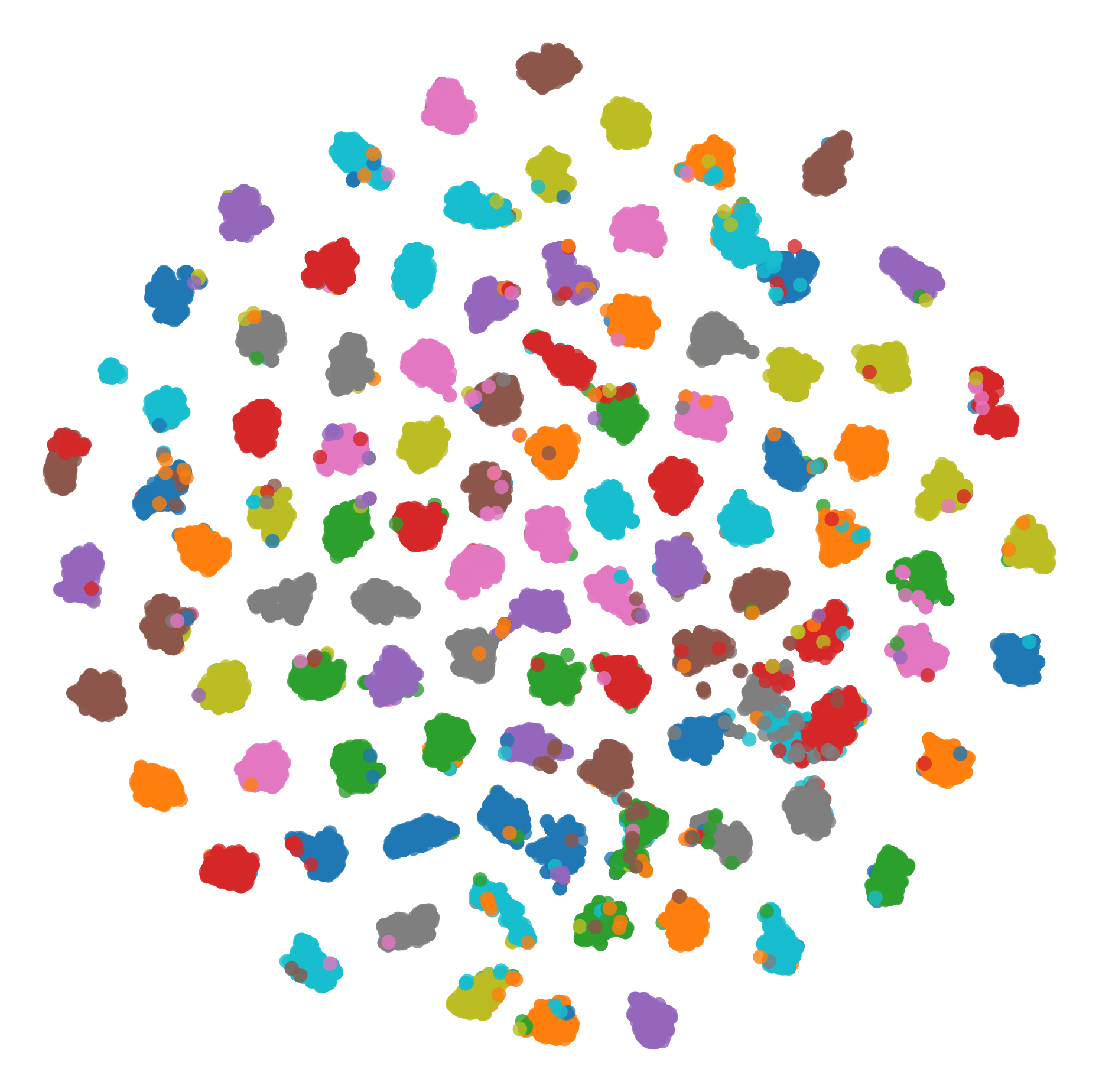}
    }
    \caption{The t-SNE visualization of three levels of features generated by SAGL on the Food101 dataset, where the original features are extracted using DINOv3.}
    \label{fig:sne:food2}
\end{figure*}

\subsection{Experimental Details}
The statistics of the datasets are summarized in Table~\ref{tb:datasets}. All experiments are performed on a Linux workstation equipped with a GeForce RTX 4090 GPU (24 GB memory), an Intel Xeon Platinum 8336C CPU, and 128 GB of RAM. SAGL is implemented in PyTorch \cite{Paszke2019Pytorch}.

\subsubsection{Parameter Settings}
\label{sec:paramsetting}
During both training and testing, the learning rate for the proposed SAGL model is empirically set to $5 \times 10^{-4}$ for the KITTI, Flowers, Food101 and ImageNet-1K datasets and to  $1 \times 10^{-3}$  for all other datasets. The batch size for both training and testing is selected from the set $\{100, 500, 1{,}000, 5{,}000, 10{,}000\}$. Specifically, the batch size is determined by scaling the number of clusters by an approximate multiplier $a \in \{5, 10, 20, 50, 100\}$. For datasets with a relatively small number of clusters, we apply smaller multipliers (e.g., 5, 10, or 20), whereas larger multipliers are used otherwise. The proposed SAGL model is trained for 600 epochs. The dropout rate is selected from the set $\{0.0, 0.1, 0.2, 0.3, 0.4\}$, with lower values applied to smaller datasets and higher values applied to larger datasets. We employ a grid search strategy to determine the optimal combination of hyperparameter values, selecting $\gamma$ from the set $\{5.0, 10.0, 20.0\}$, and $\beta$ from the set $\{0.2, 0.5, 1.0\}$. We report the best-tuned performance for all competing methods to ensure a fair comparison.

\begin{figure*}[htbp]
    \centering
    \subfigure[Caltech101 View 1]{
        \includegraphics[width=0.23\textwidth]{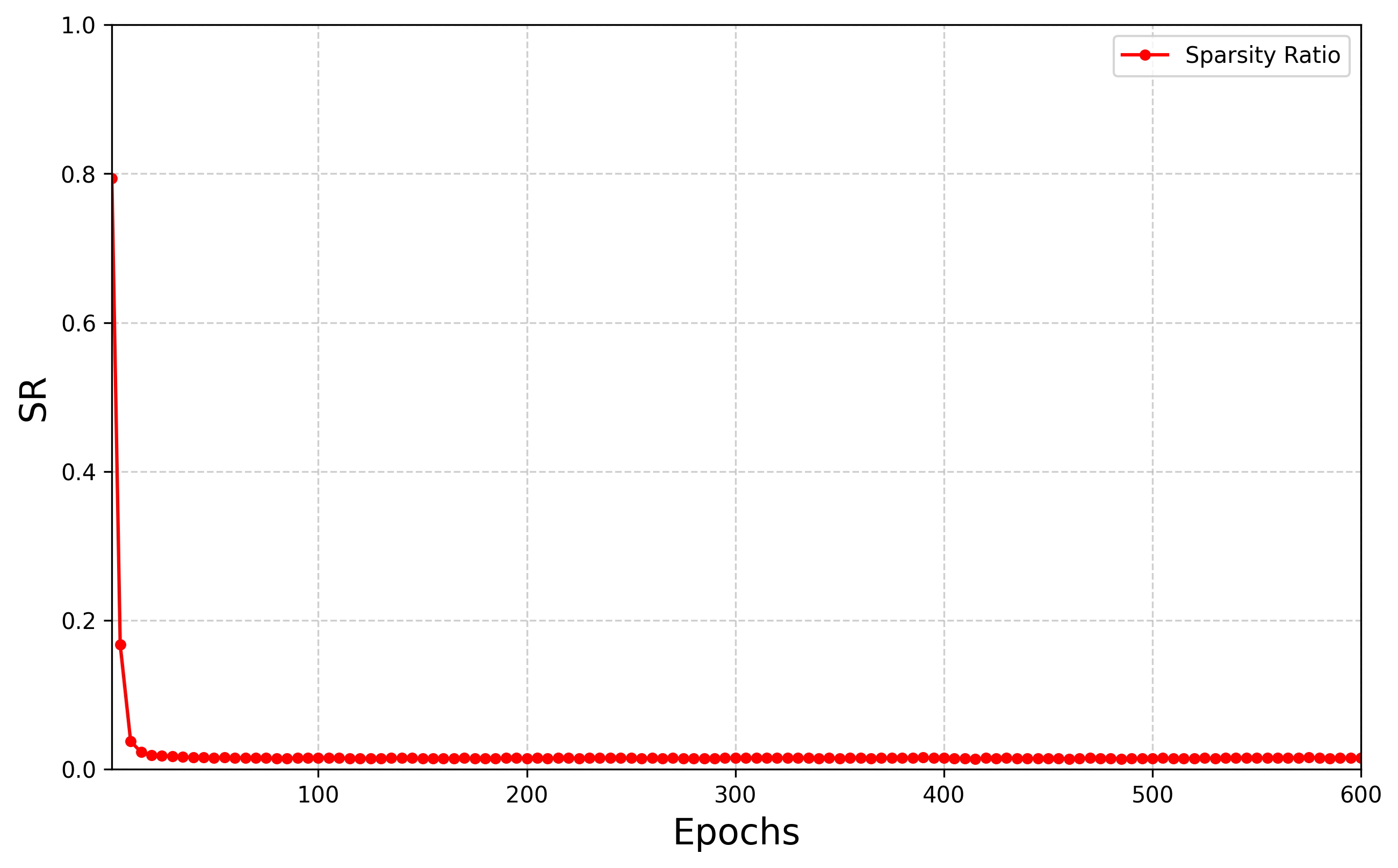}
    }
    \subfigure[Caltech101 View 2]{
        \includegraphics[width=0.23\textwidth]{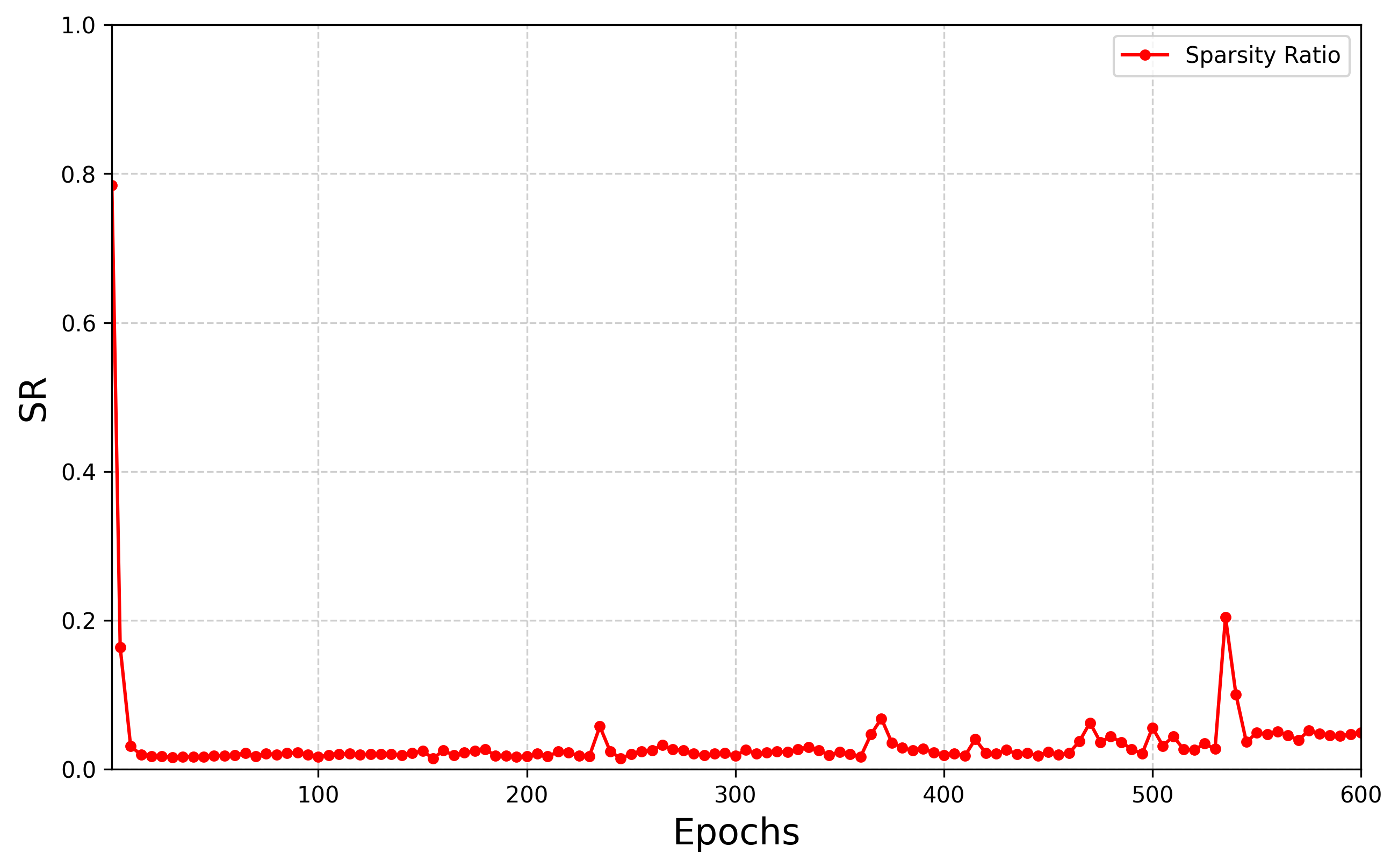}
    }
    \subfigure[Food101 View 1]{
        \includegraphics[width=0.23\textwidth]{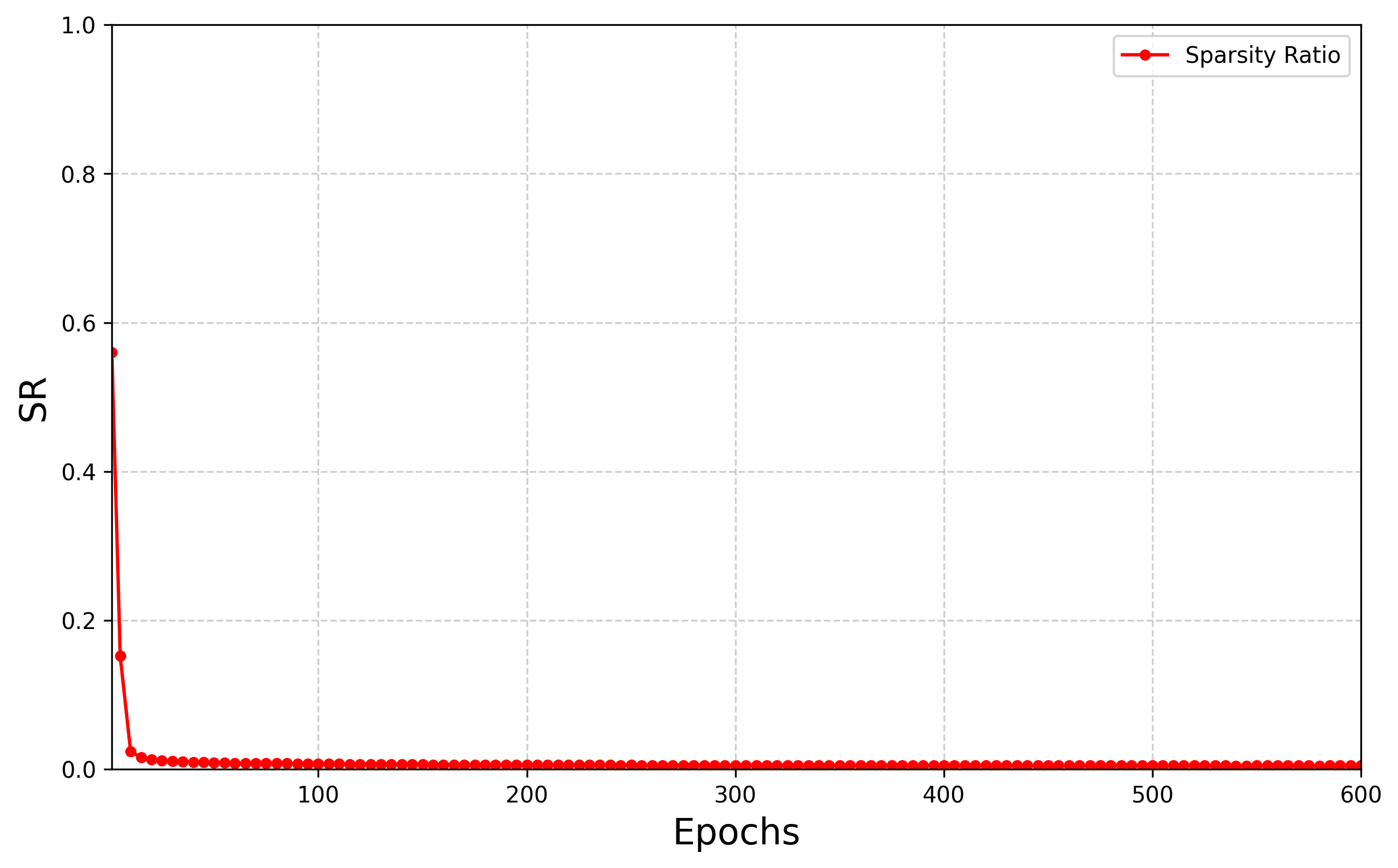}
    }
    \subfigure[Food101 View 2]{
        \includegraphics[width=0.23\textwidth]{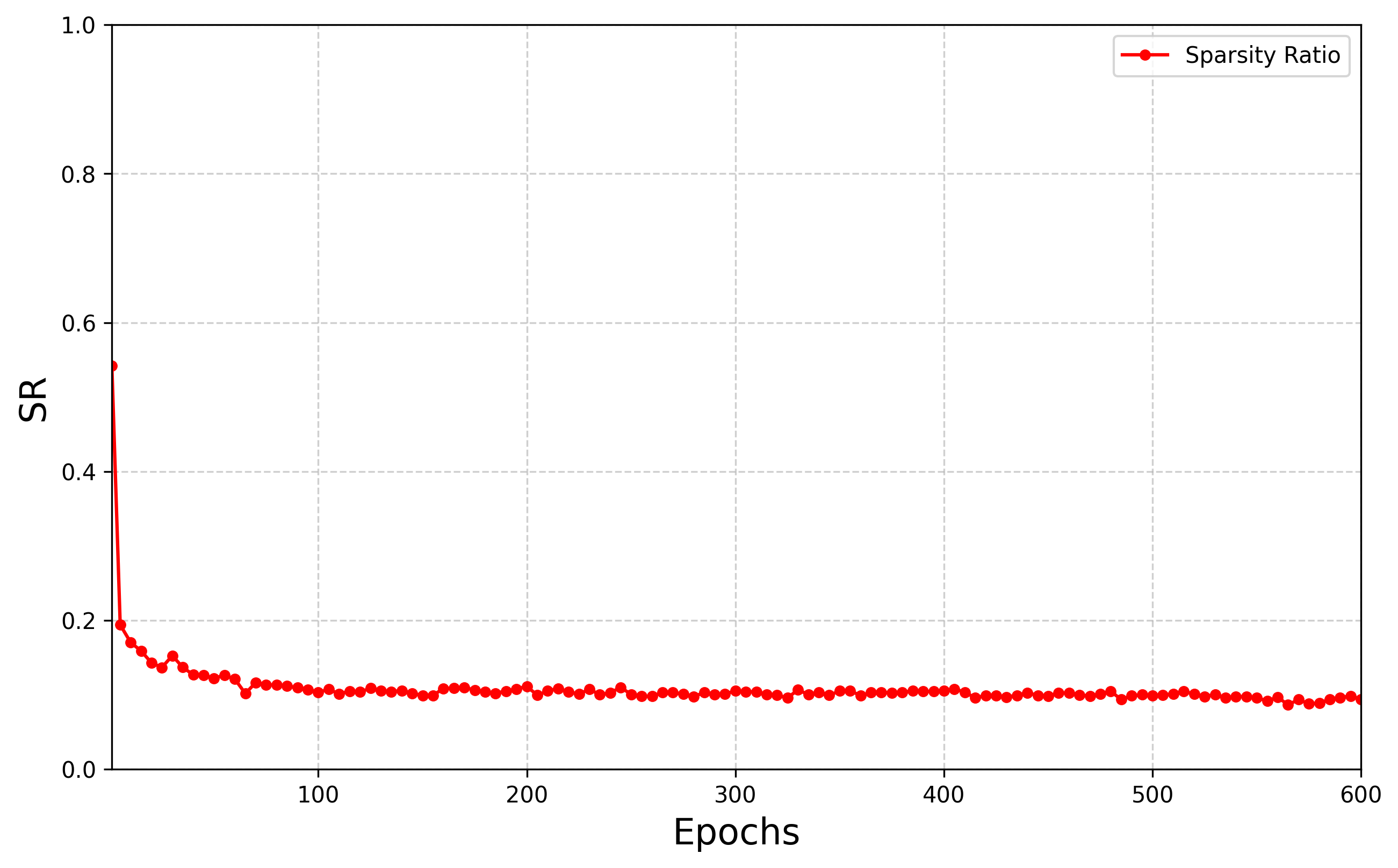}
    }
    \caption{Sparsity ratio evolution of sparse attention graphs during training.}
    \label{fig:sparsity:training:view}
\end{figure*}

\begin{figure*}[htbp]
    \centering
    \subfigure[Caltech101 View 1]{
        \includegraphics[width=0.23\textwidth]{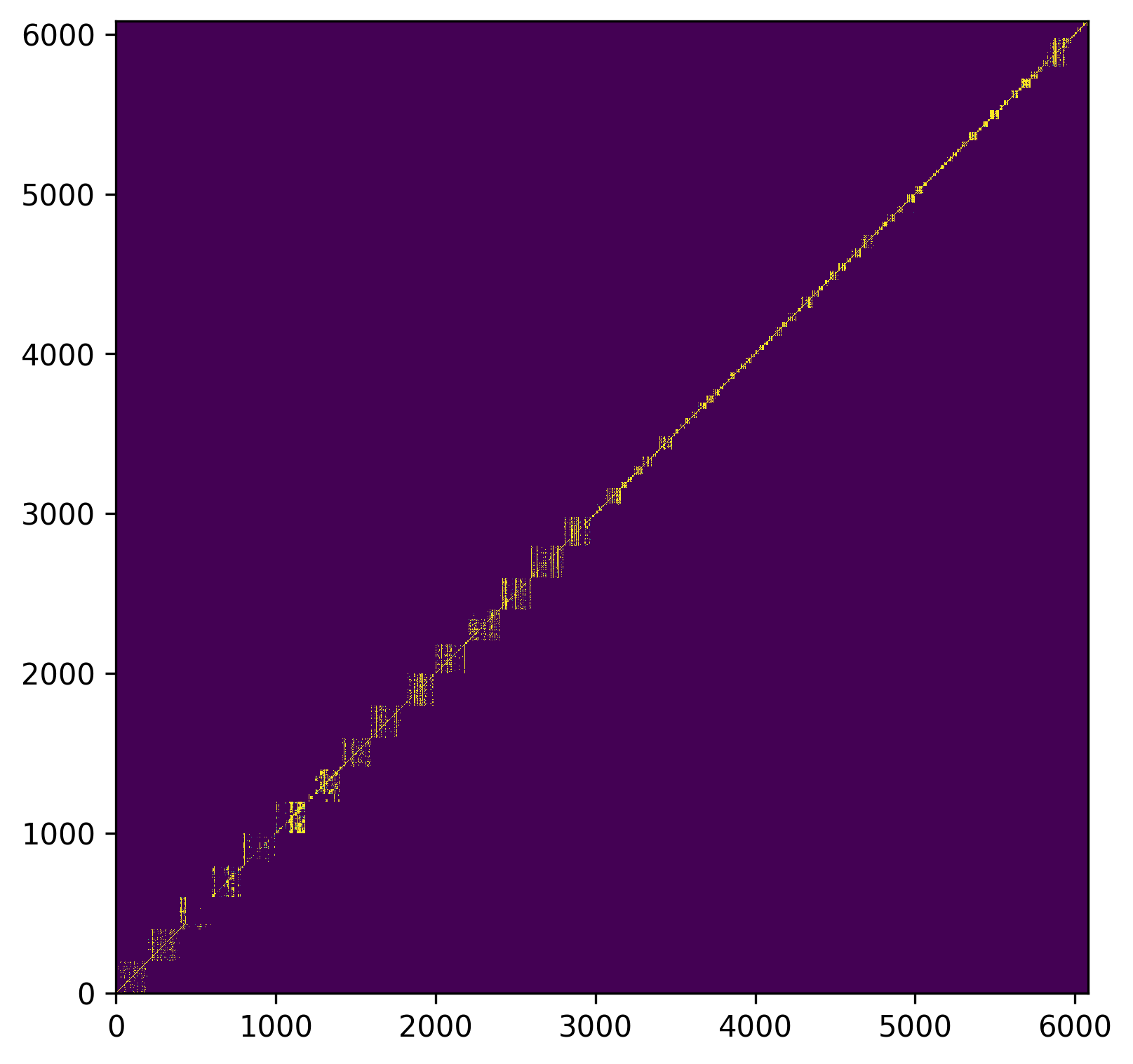}
    }
    \subfigure[Caltech101 View 2]{
        \includegraphics[width=0.23\textwidth]{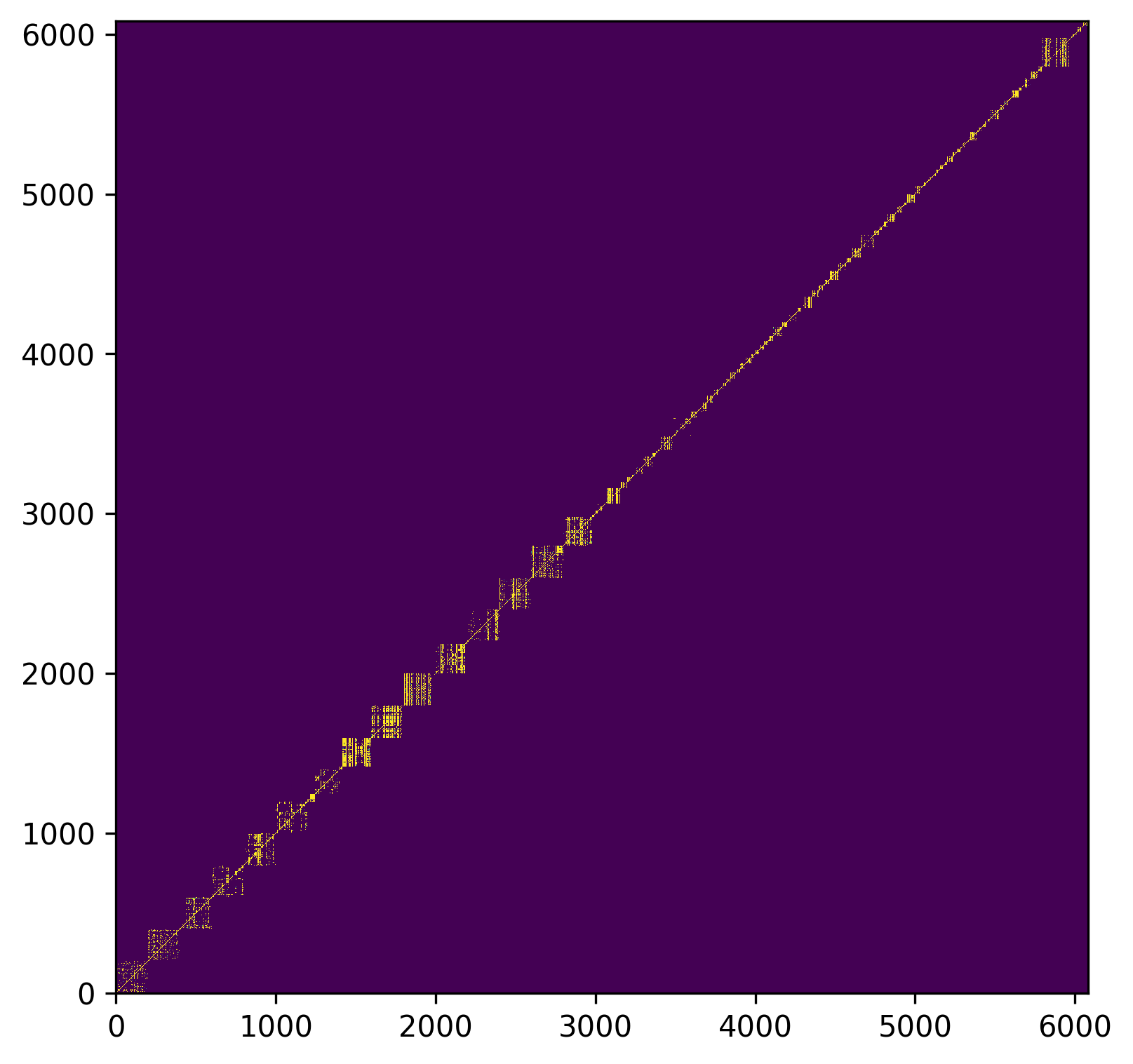}
    }
    \subfigure[Food101 View 1]{
        \includegraphics[width=0.23\textwidth]{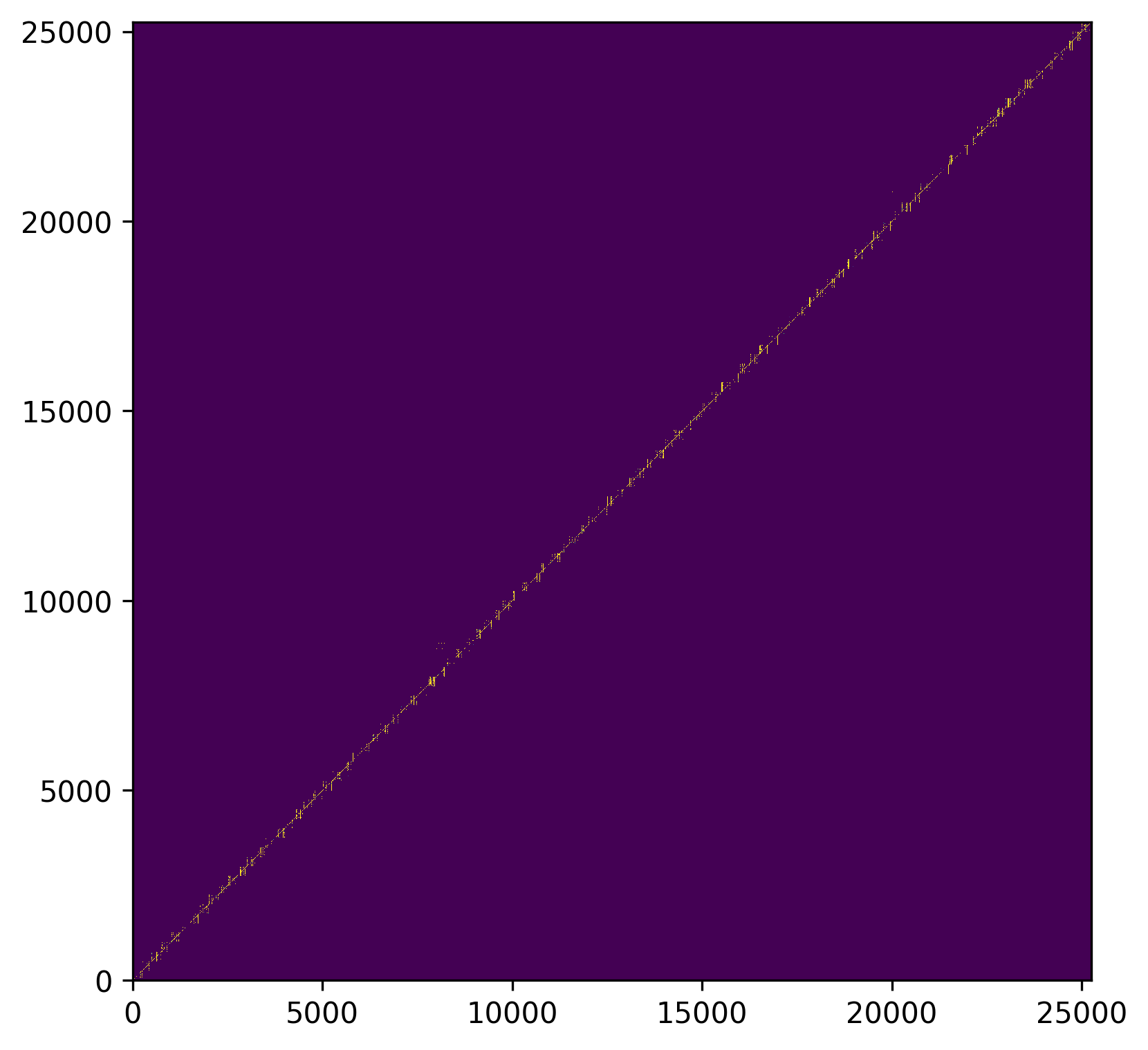}
    }
    \subfigure[Food101 View 2]{
        \includegraphics[width=0.23\textwidth]{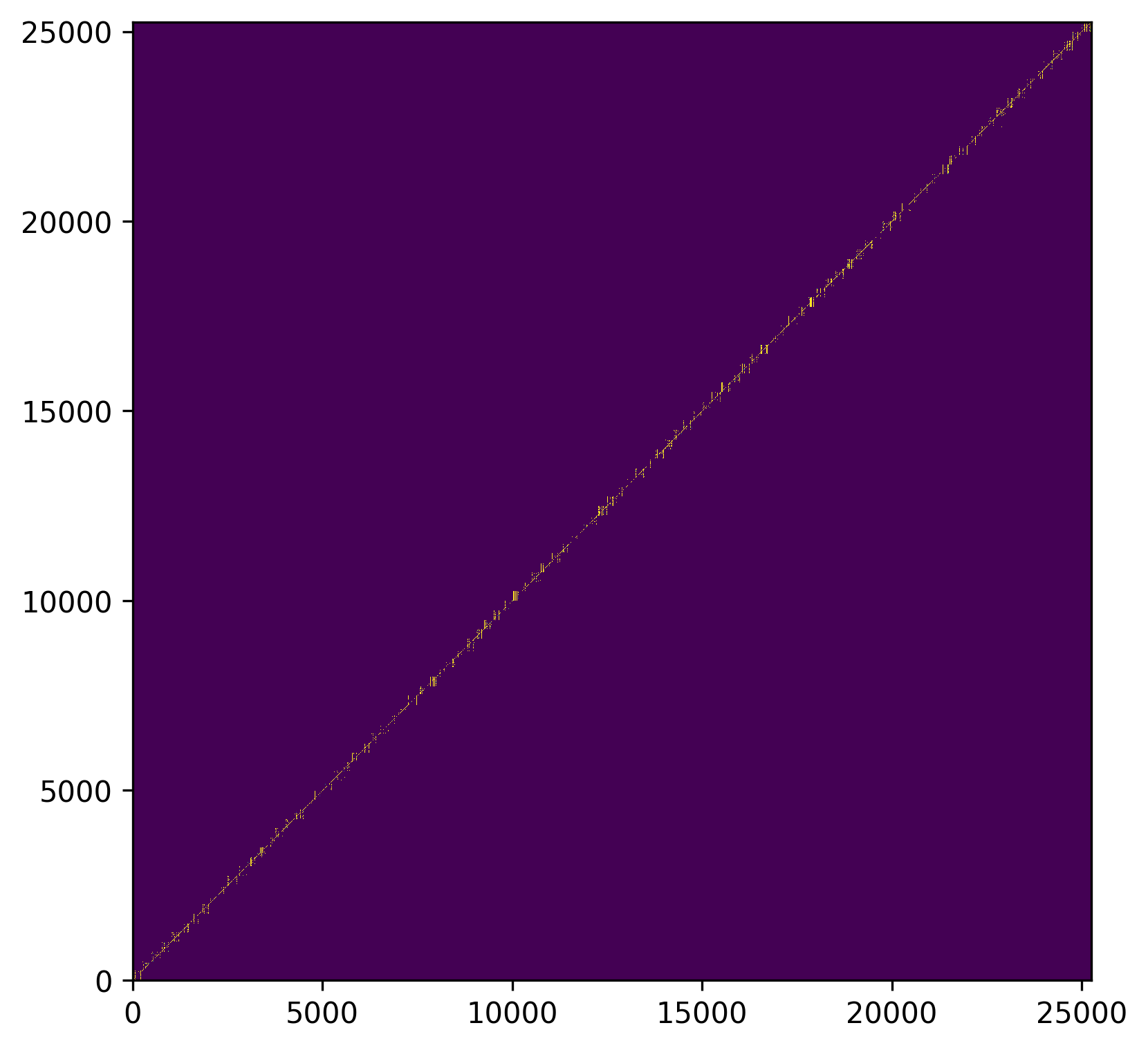}
    }
    \caption{Block-diagonal structures of the learned sparse attention graphs on the Caltech101 and Food101 datasets.}
    \label{fig:sparsity:testing:view}
\end{figure*}

\begin{figure*}[htbp]
    \centering
    \subfigure[ACC]{
        \includegraphics[width=0.30\textwidth]{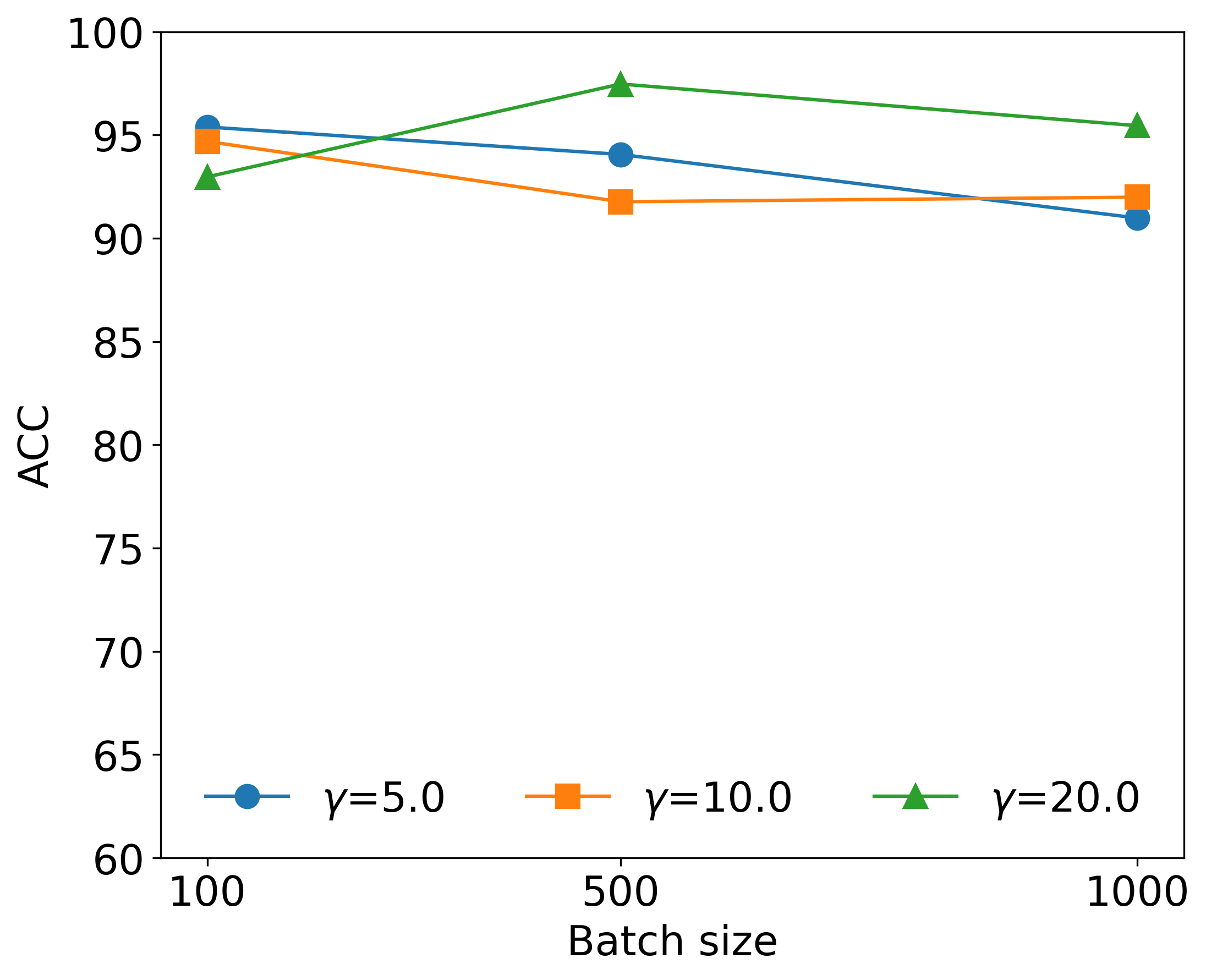}
    }
    \subfigure[NMI]{
        \includegraphics[width=0.30\textwidth]{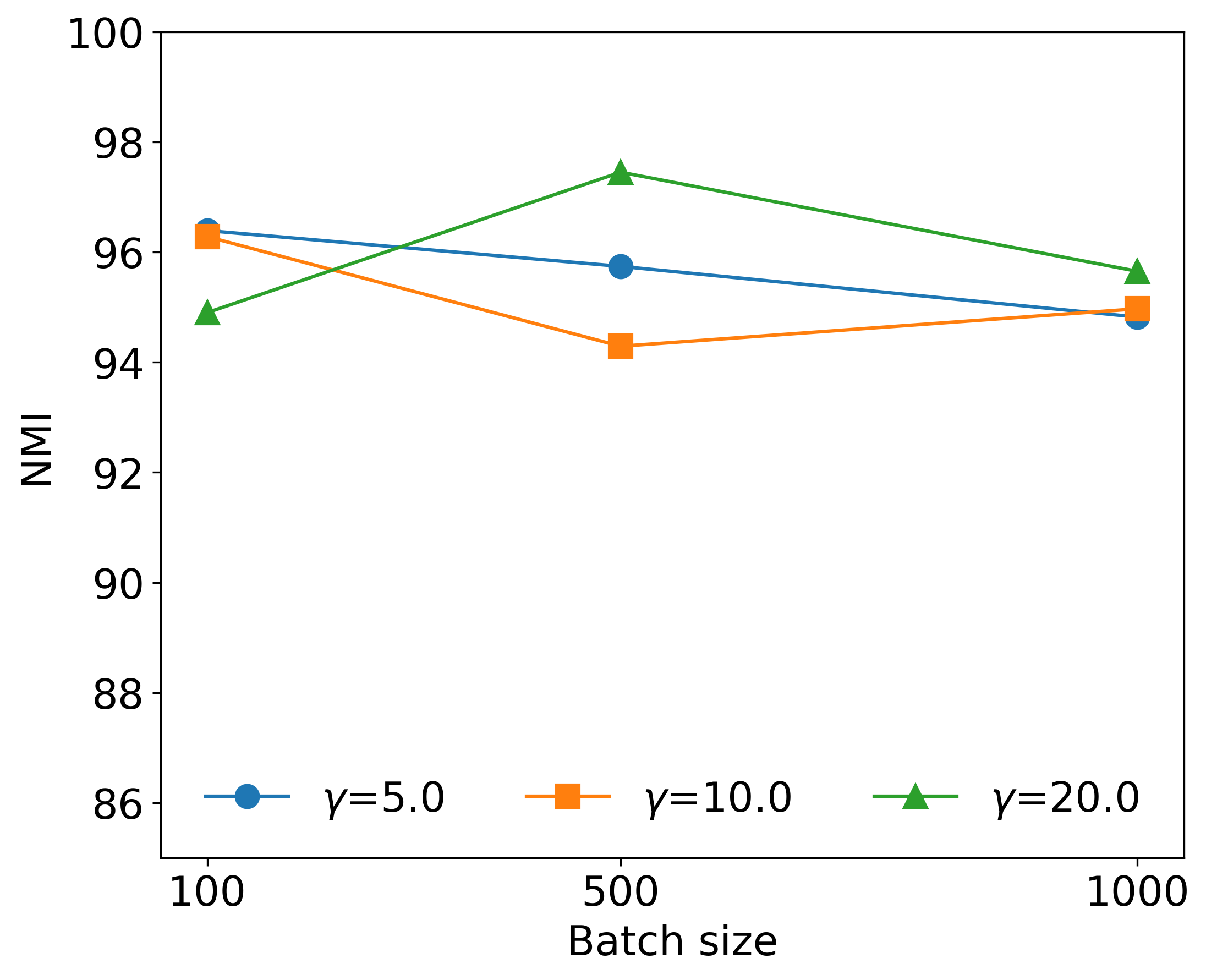}
    }
     \subfigure[ARI]{
        \includegraphics[width=0.30\textwidth]{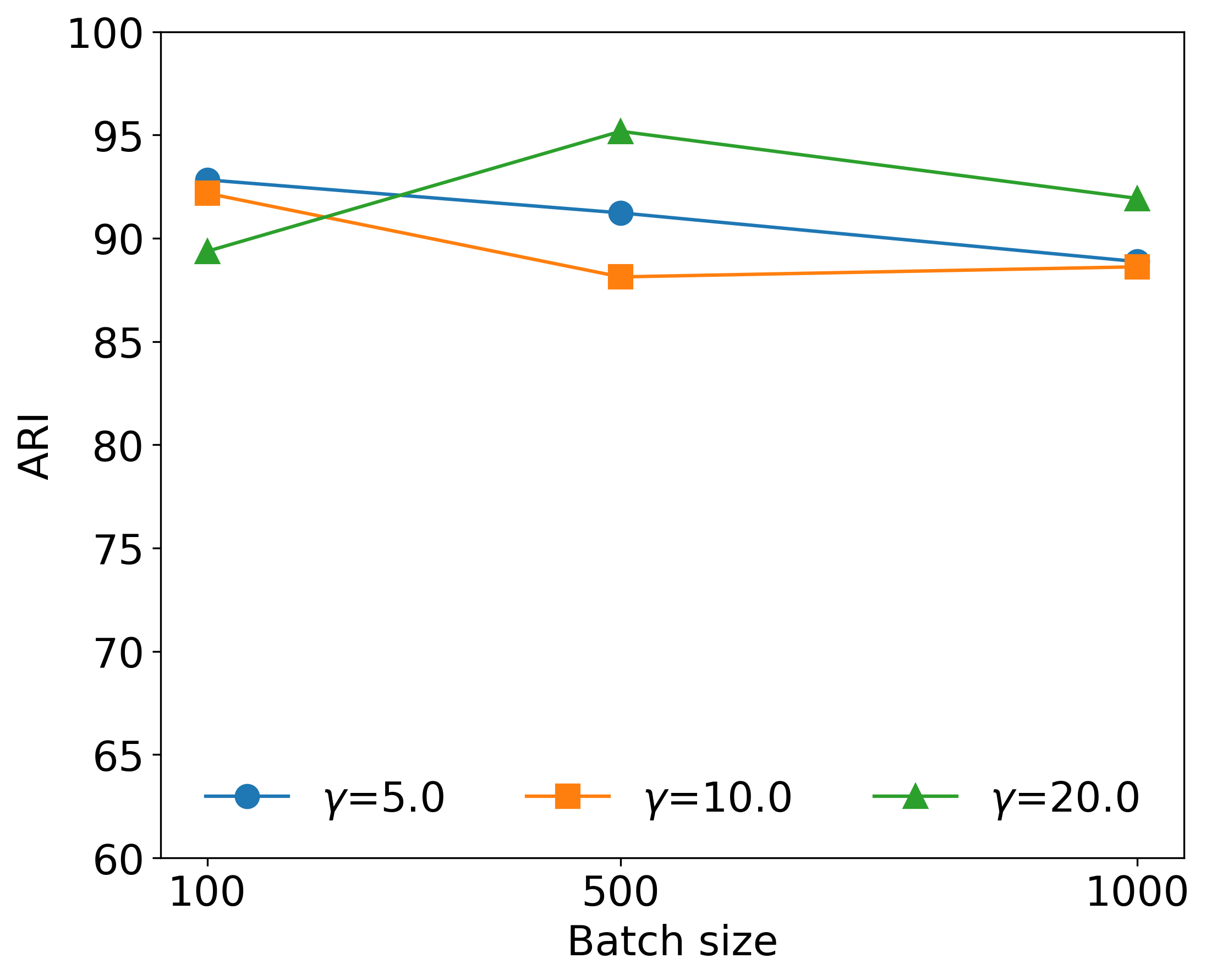}
    }
    \caption{Clustering results of SAGL on the Pets dataset across different batch sizes.}
    \label{fig:bs:pets}
\end{figure*}

\begin{figure*}[htbp]
    \centering
    \subfigure[ACC]{
        \includegraphics[width=0.30\textwidth]{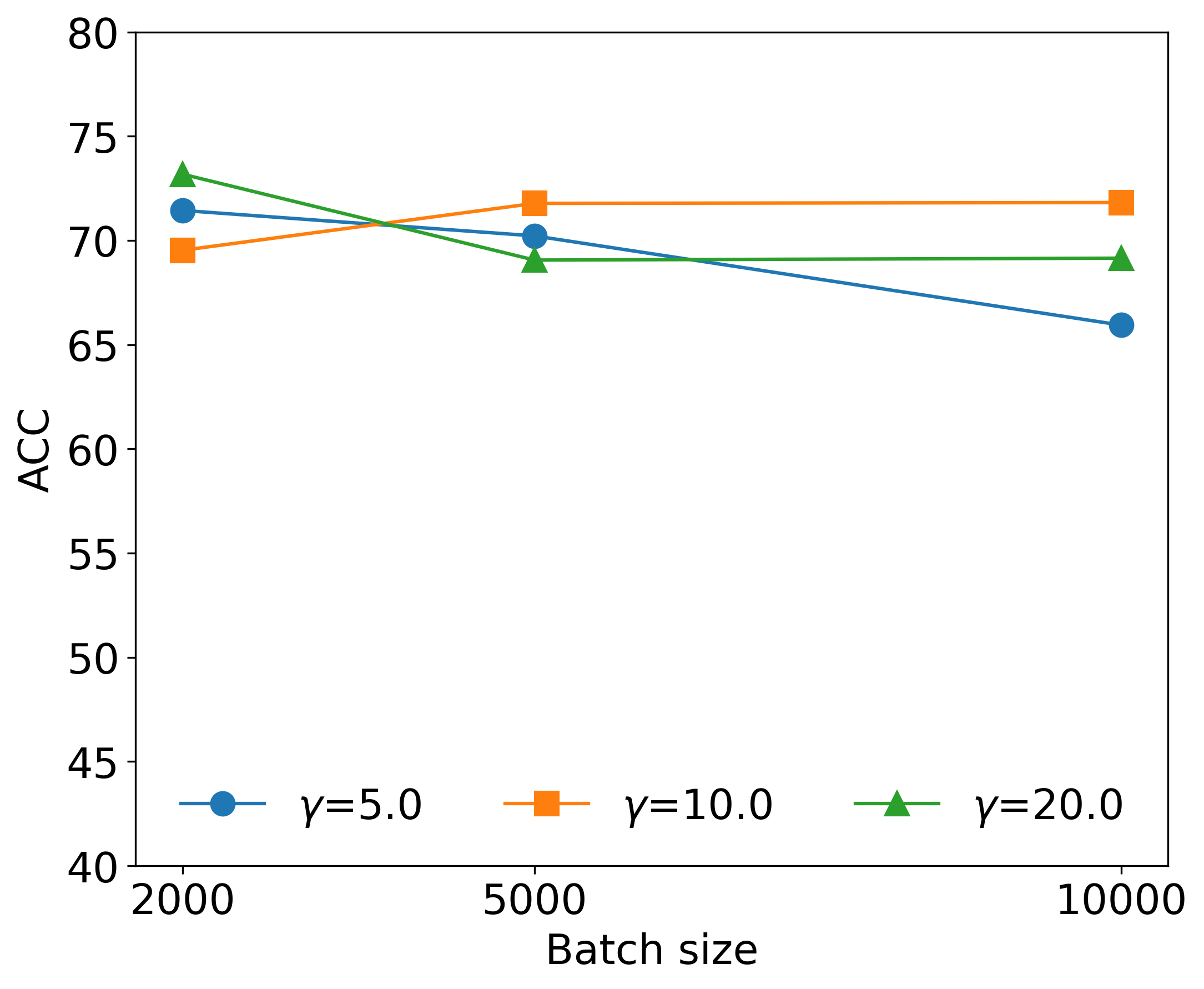}
    }
    \subfigure[NMI]{
        \includegraphics[width=0.30\textwidth]{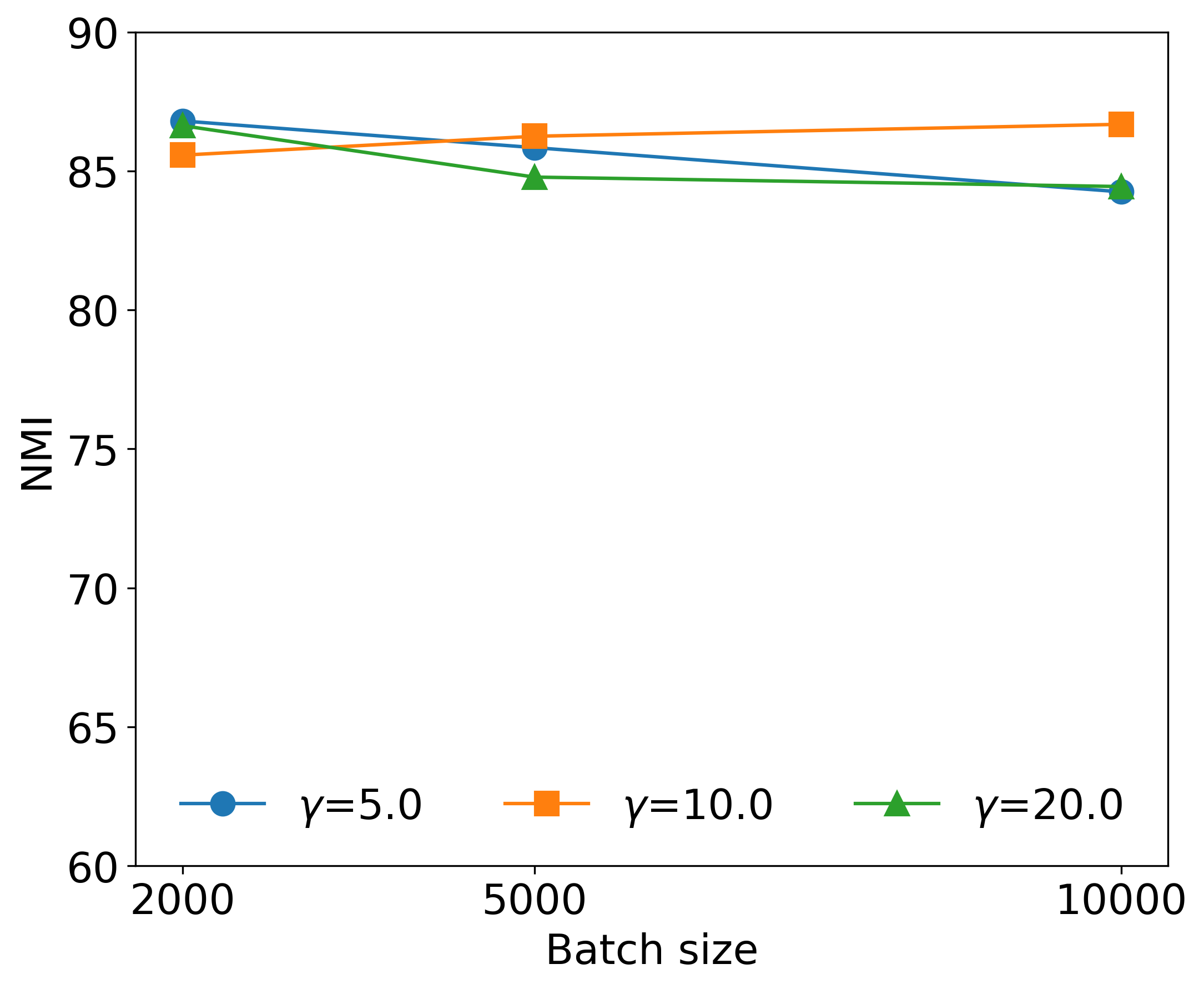}
    }
     \subfigure[ARI]{
        \includegraphics[width=0.30\textwidth]{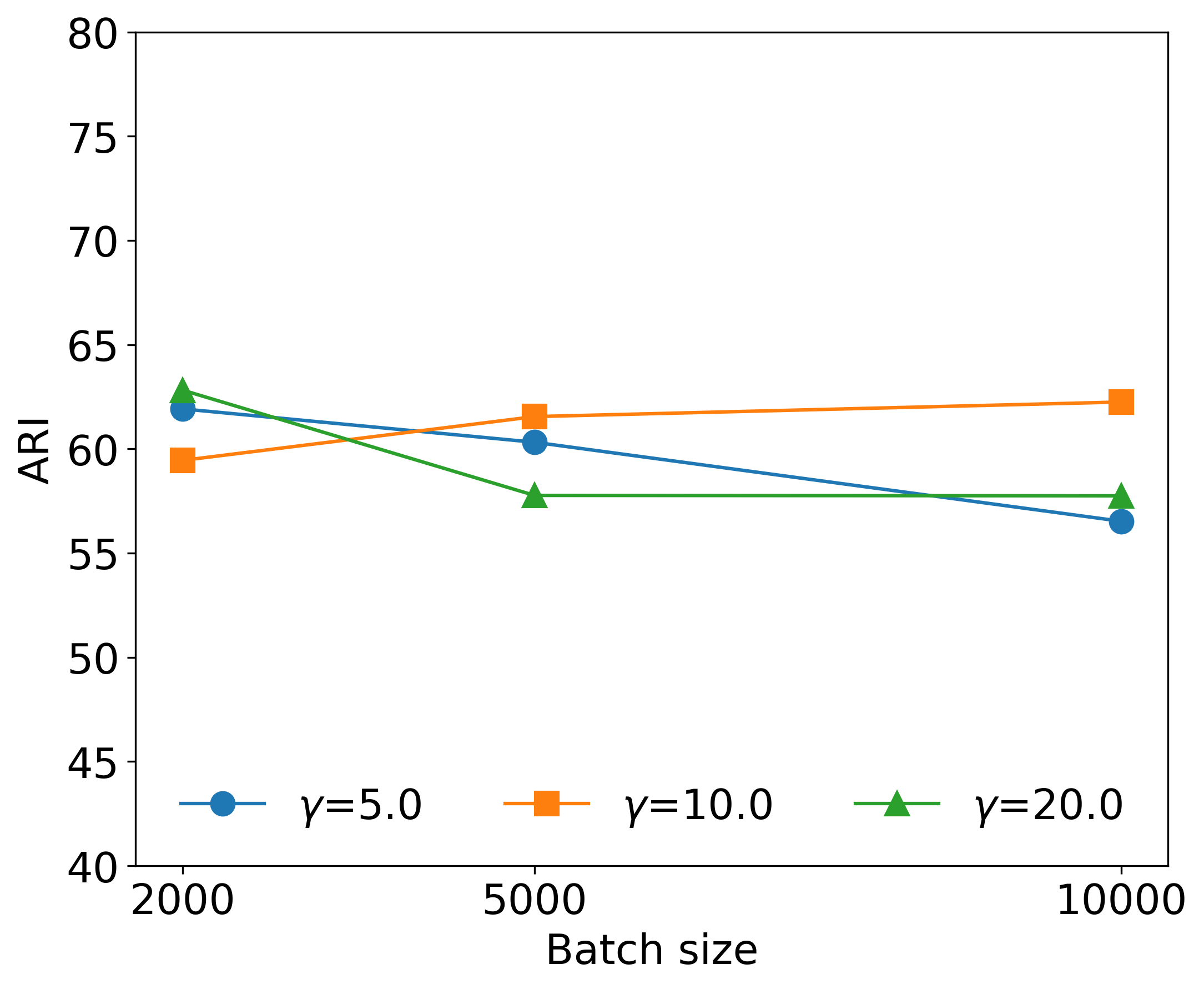}
    }
    \caption{Clustering results of SAGL on the SUN397 dataset across different batch sizes.}
    \label{fig:bs:sun}
\end{figure*}

\begin{figure*}[htbp]
    \centering
    \subfigure[ACC]{
        \includegraphics[width=0.30\textwidth]{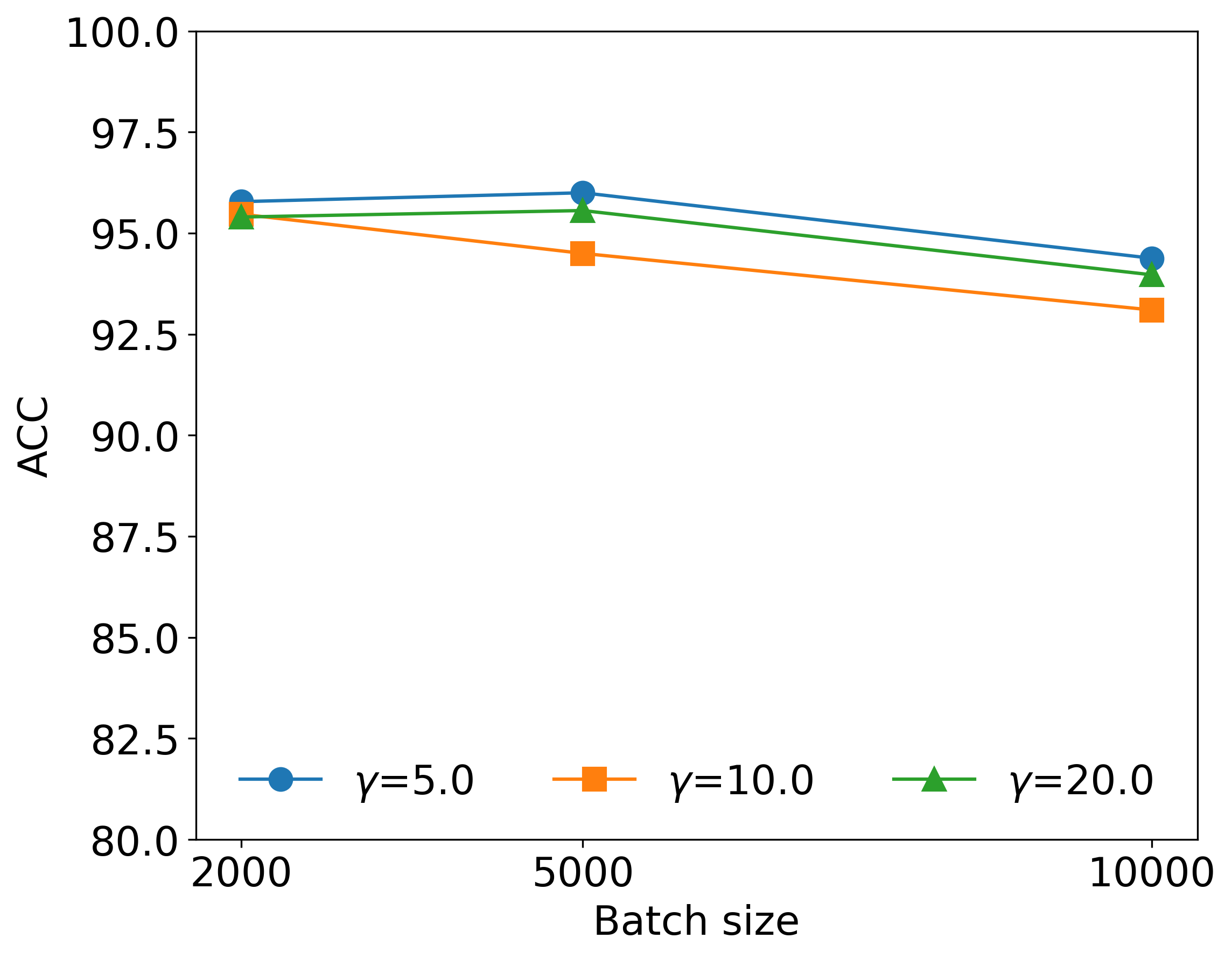}
    }
    \subfigure[NMI]{
        \includegraphics[width=0.30\textwidth]{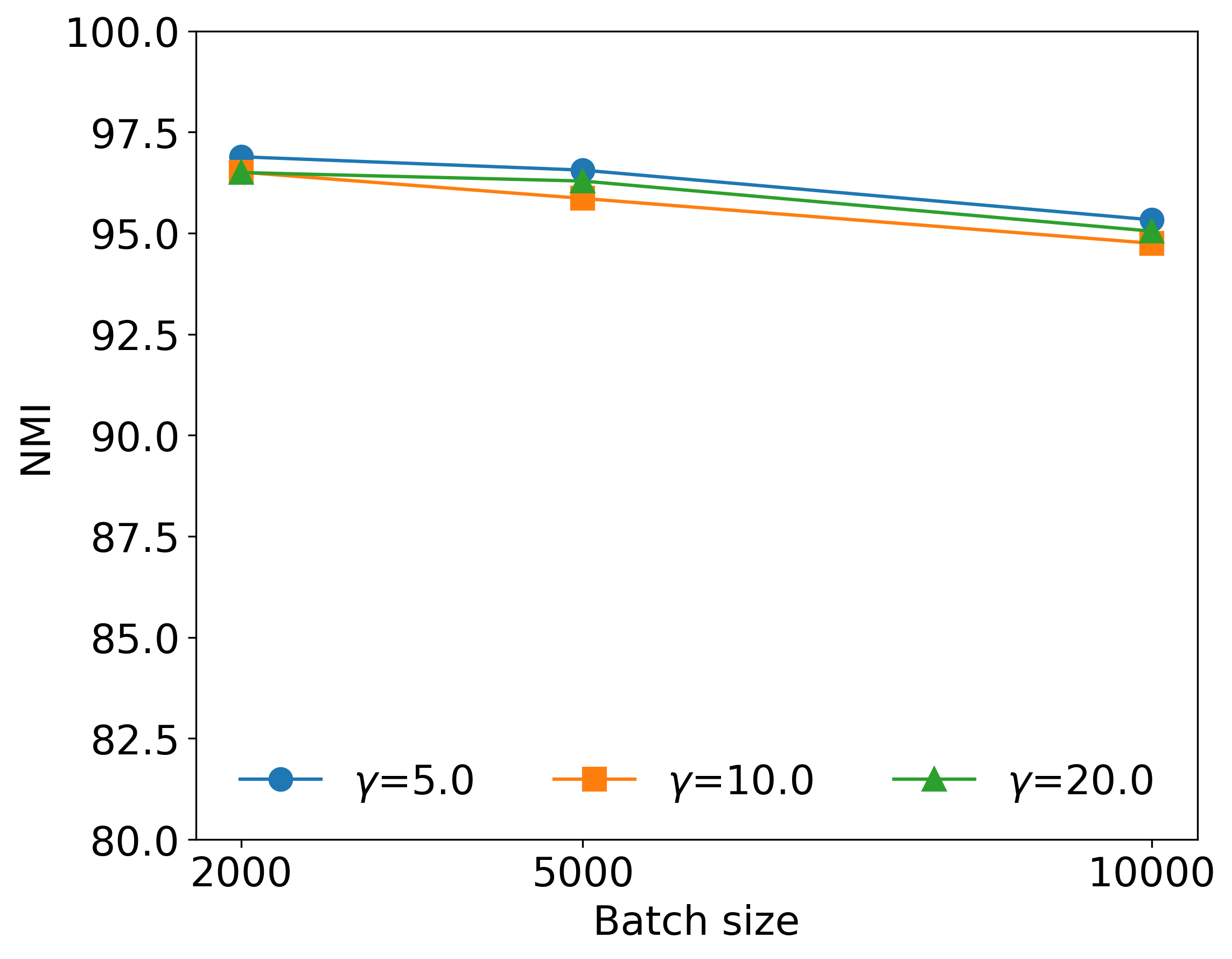}
    }
     \subfigure[ARI]{
        \includegraphics[width=0.30\textwidth]{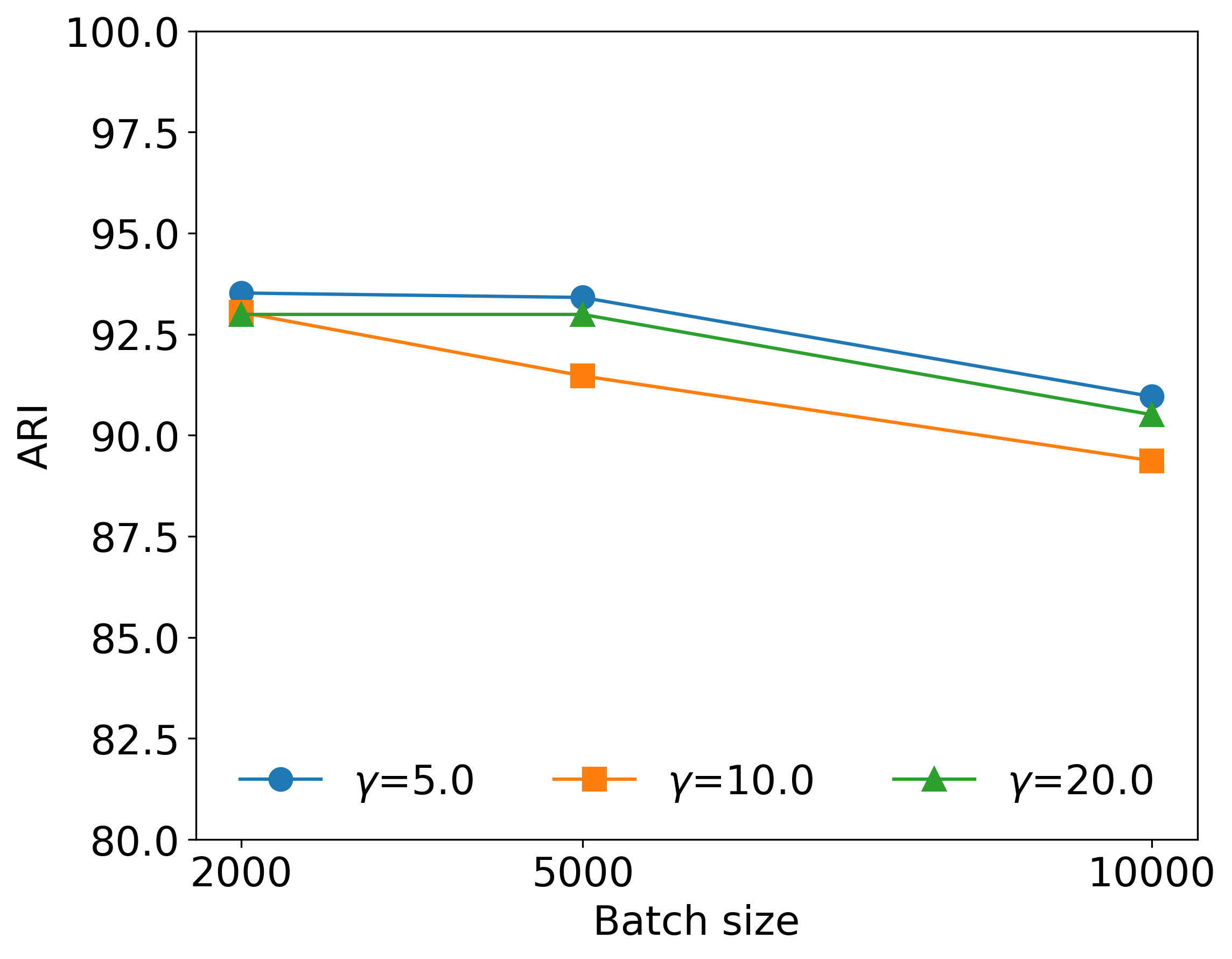}
    }
    \caption{Clustering results of SAGL on the Food101 dataset across different batch sizes.}
    \label{fig:bs:food}
\end{figure*}

\begin{figure*}[htbp]
    \centering
    \subfigure[ACC]{
        \includegraphics[width=0.30\textwidth]{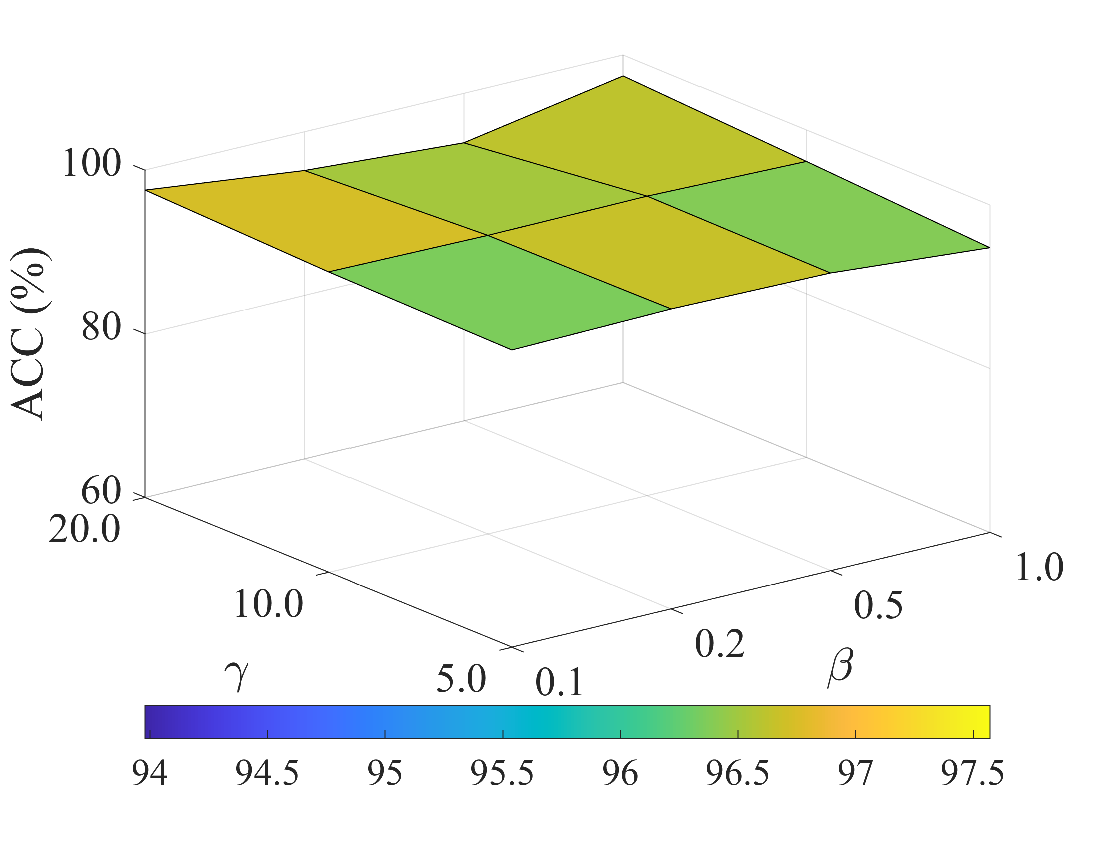}
    }
    \subfigure[NMI]{
        \includegraphics[width=0.30\textwidth]{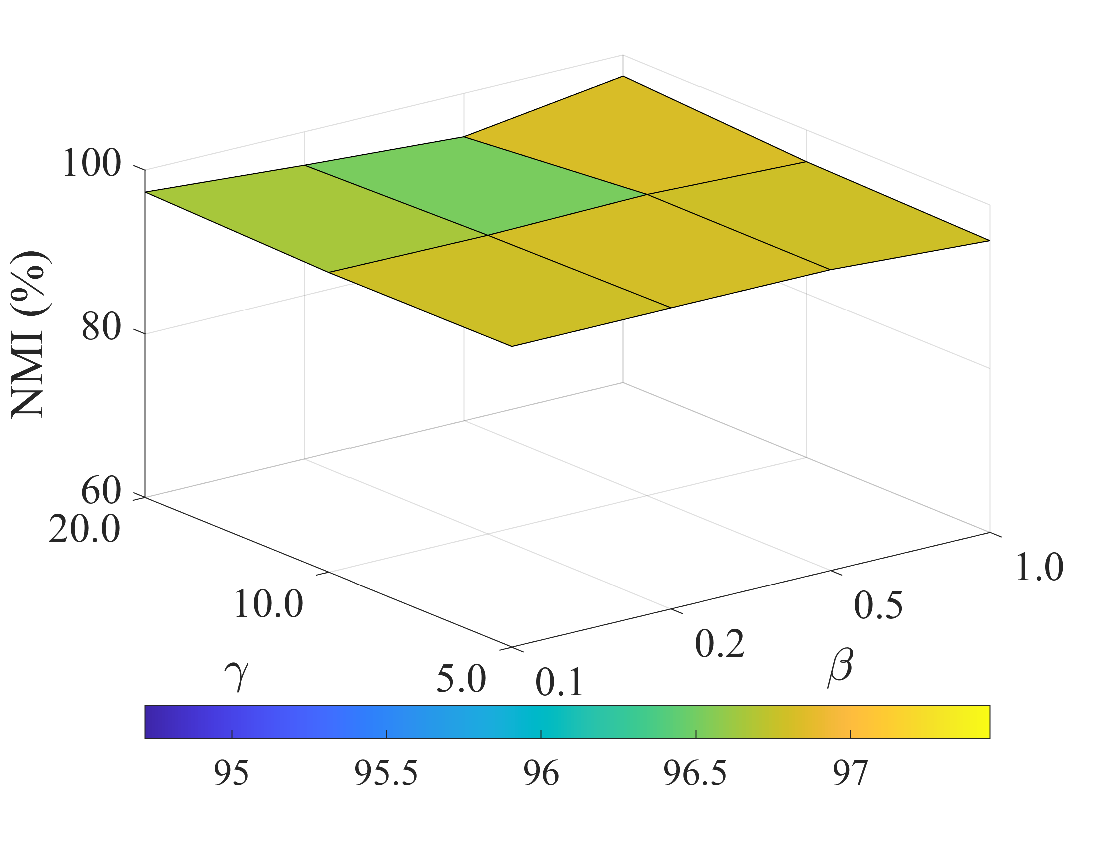}
    }
     \subfigure[ARI]{
        \includegraphics[width=0.30\textwidth]{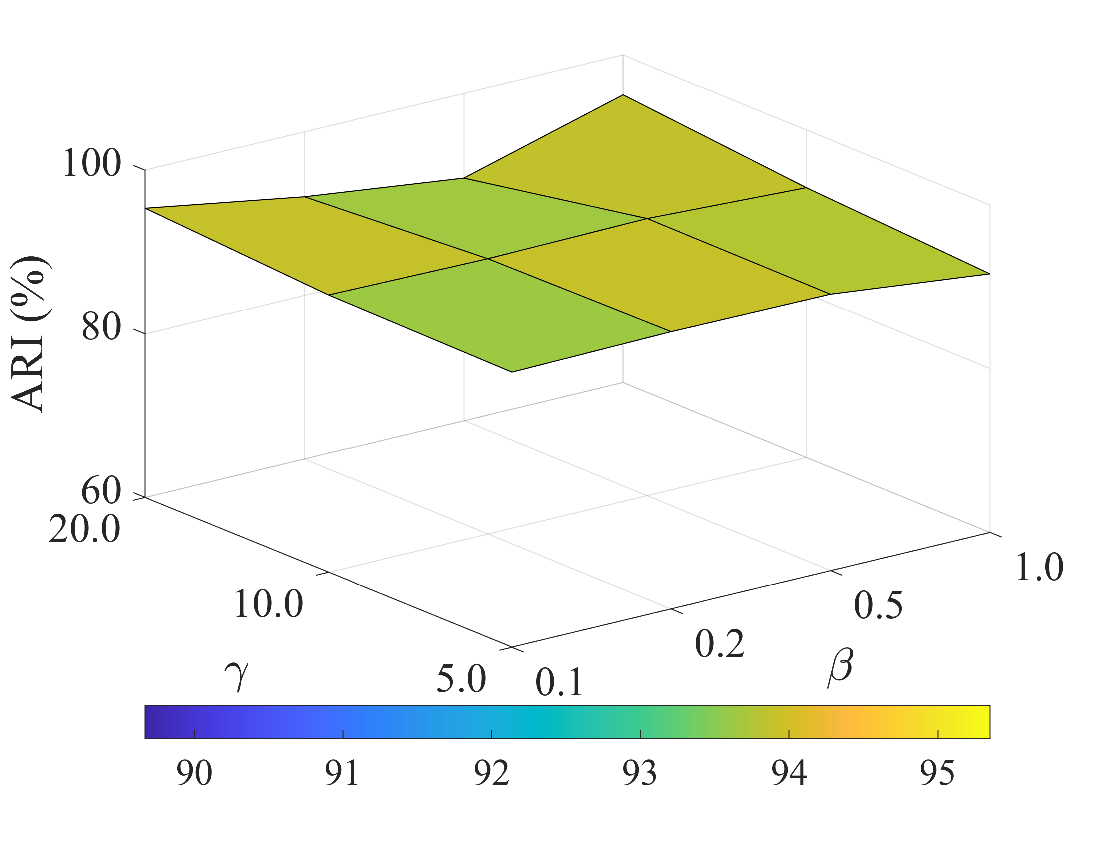}
    }
    \caption{The clustering results of SAGL under different $\gamma$ and $\beta$ combinations on the Pets dataset.}
    \label{fig:param:pets}
\end{figure*}

\begin{figure*}[htbp]
    \centering
    \subfigure[ACC]{
        \includegraphics[width=0.30\textwidth]{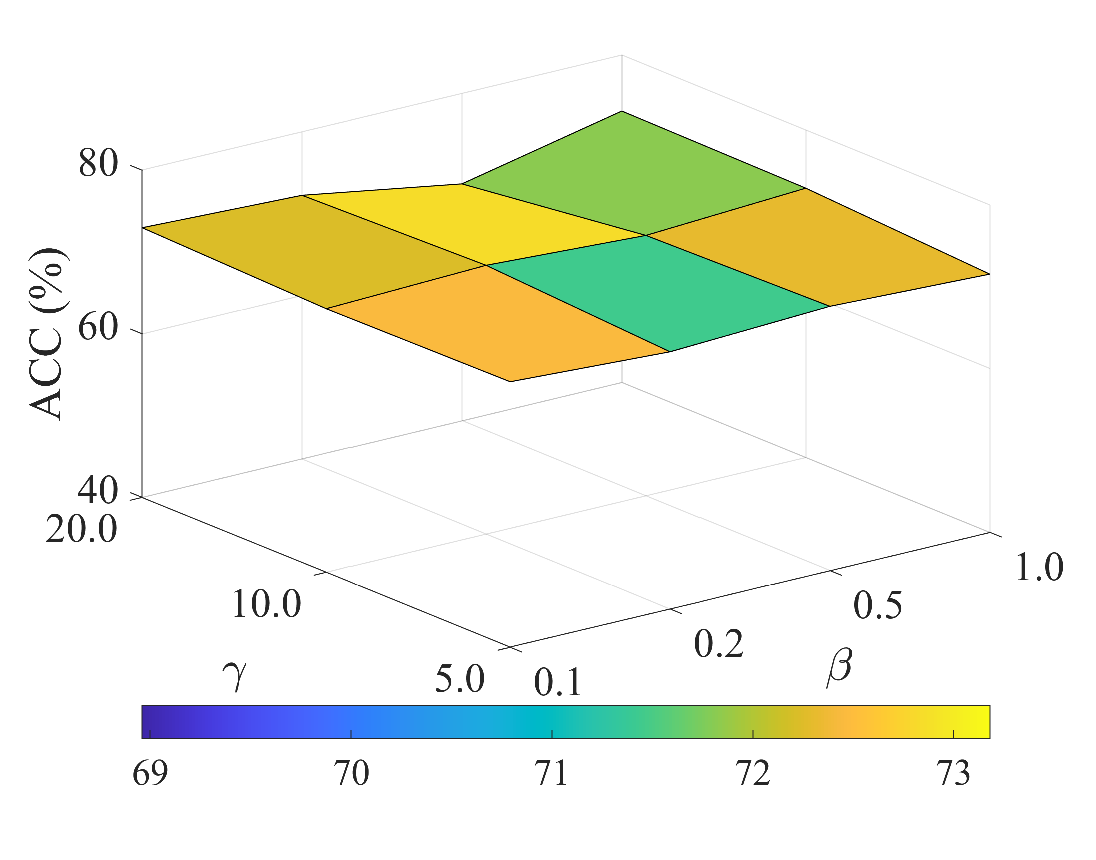}
    }
    \subfigure[NMI]{
        \includegraphics[width=0.30\textwidth]{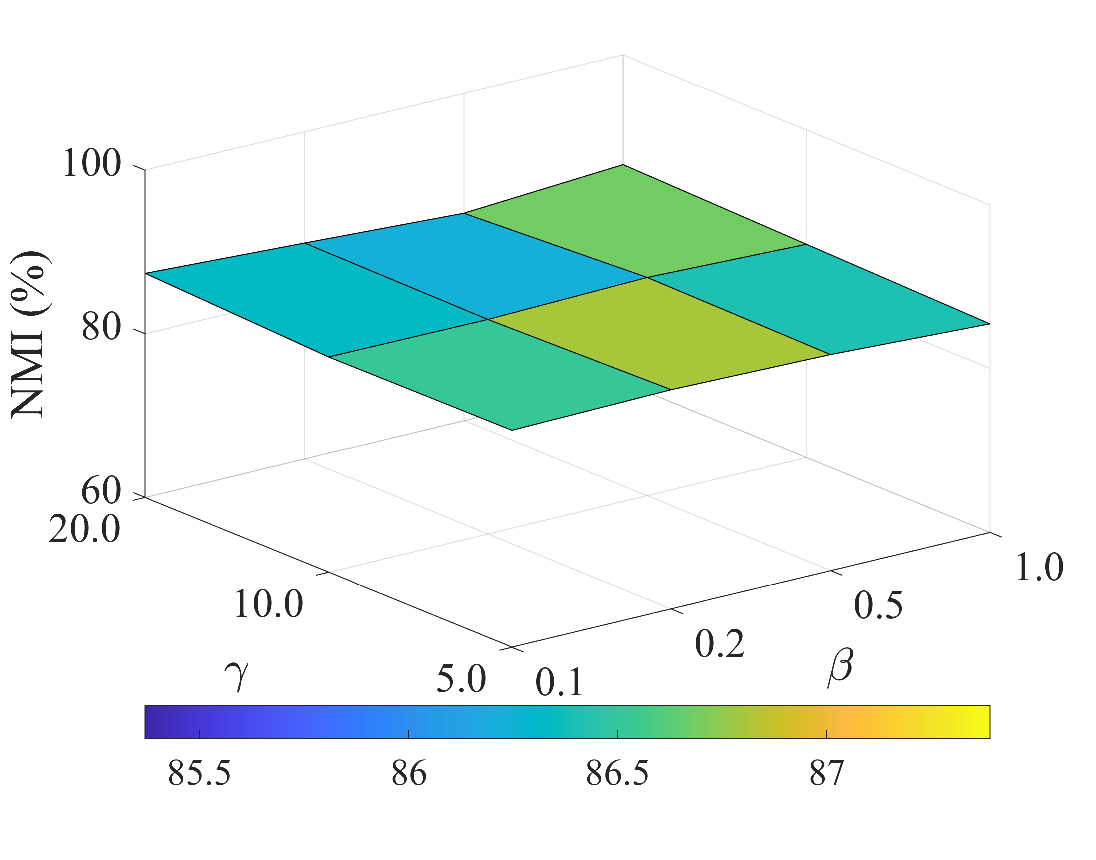}
    }
     \subfigure[ARI]{
        \includegraphics[width=0.30\textwidth]{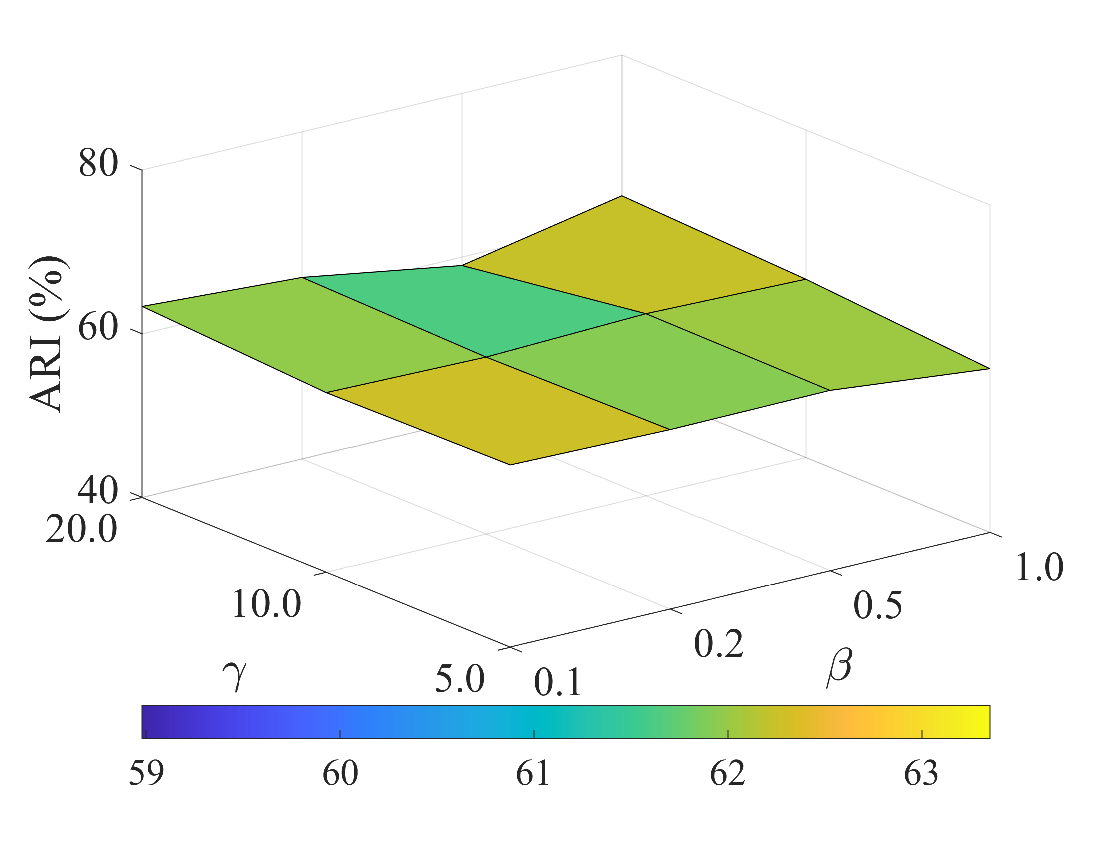}
    }
    \caption{The clustering results of SAGL under different $\gamma$ and $\beta$ combinations on the SUN397 dataset.}
    \label{fig:param:sun}
\end{figure*}

\begin{figure*}[htbp]
    \centering
    \subfigure[ACC]{
        \includegraphics[width=0.30\textwidth]{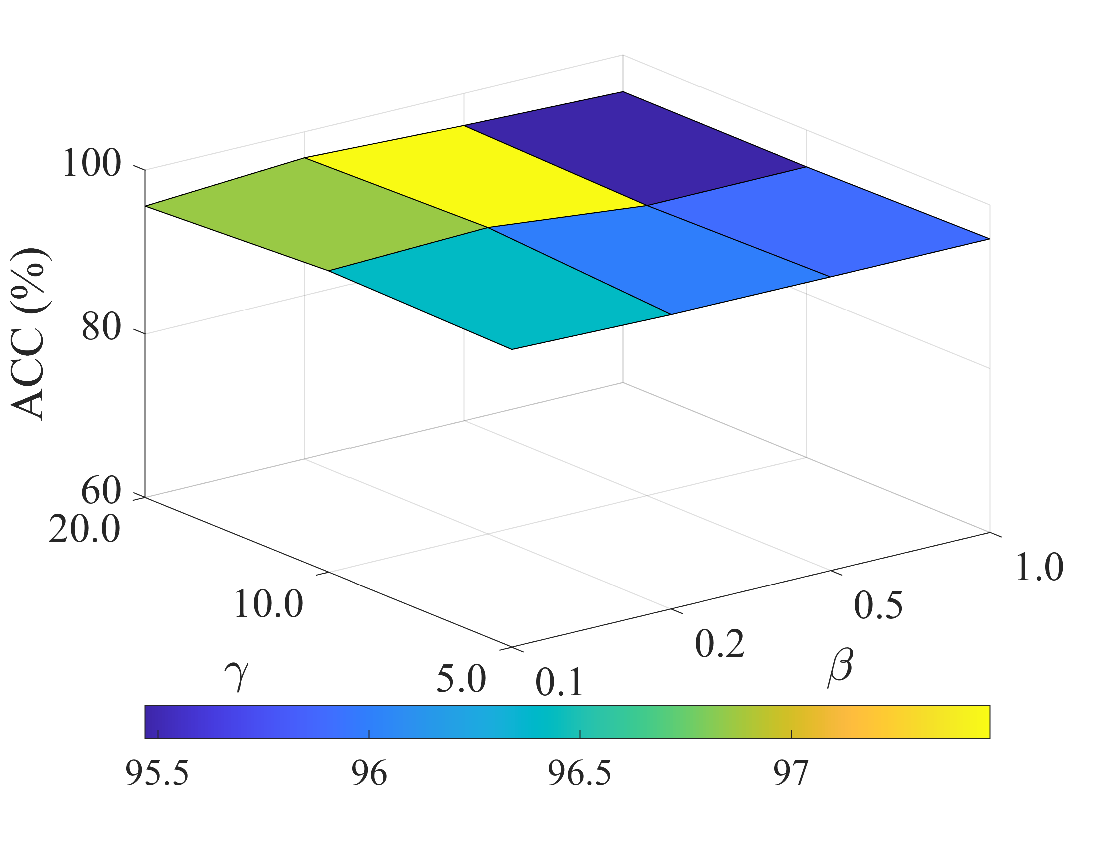}
    }
    \subfigure[NMI]{
        \includegraphics[width=0.30\textwidth]{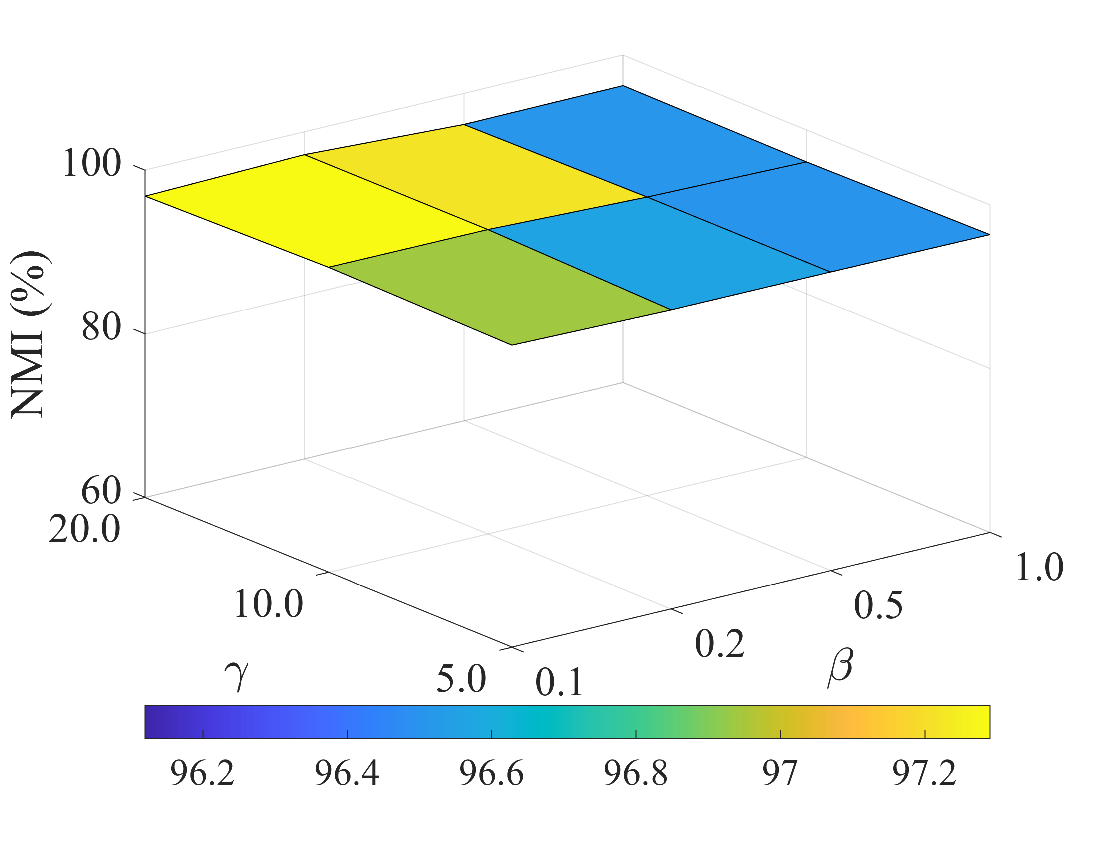}
    }
     \subfigure[ARI]{
        \includegraphics[width=0.30\textwidth]{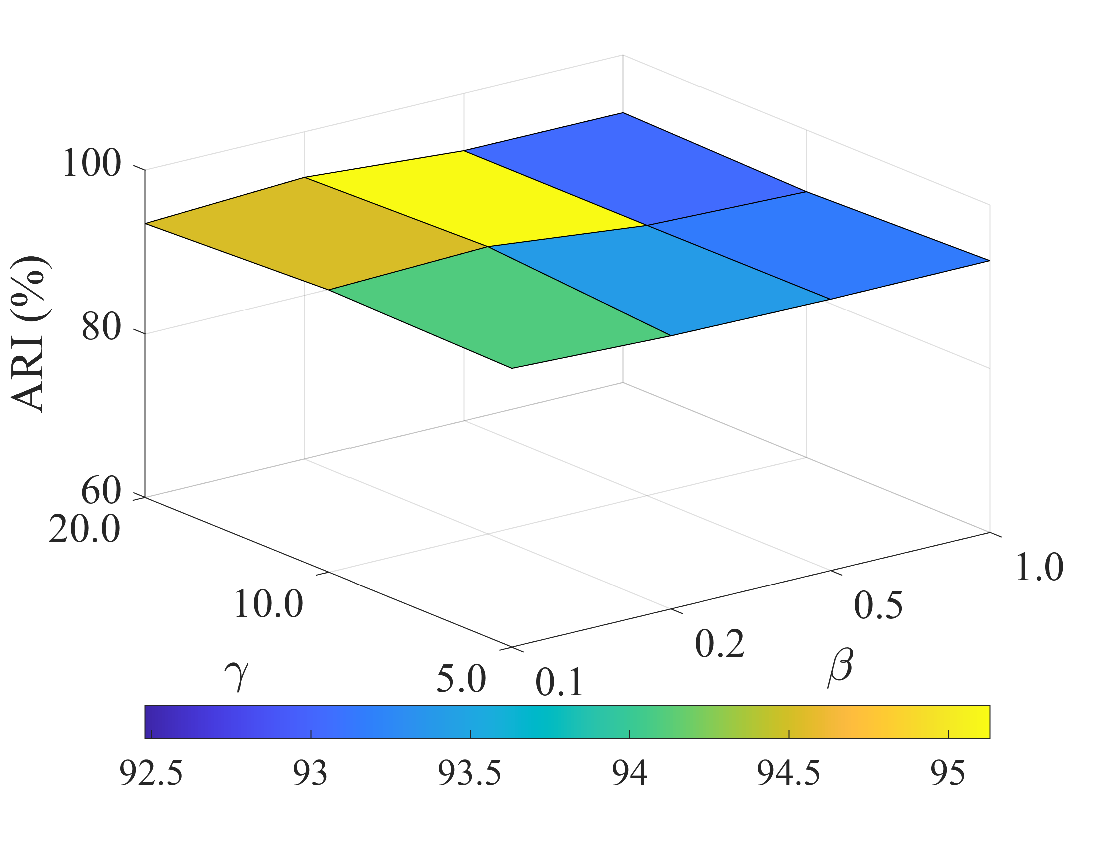}
    }
    \caption{The clustering results of SAGL under different $\gamma$ and $\beta$ combinations on the Food101 dataset.}
    \label{fig:param:food}
\end{figure*}

\begin{figure*}[htbp]
    \centering
    \subfigure[Pets]{
        \includegraphics[width=0.23\textwidth]{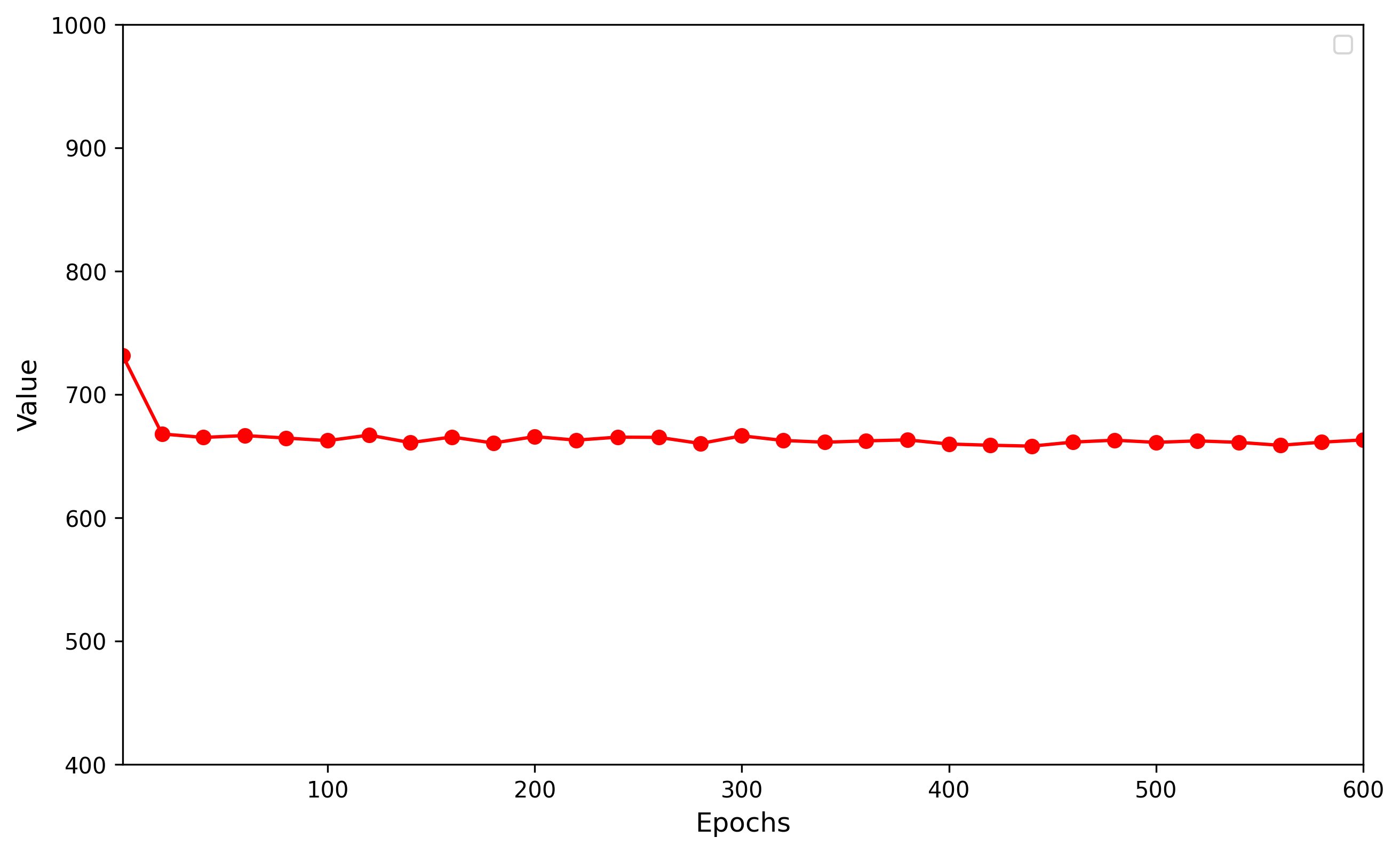}
    }
    \subfigure[KITTI]{
        \includegraphics[width=0.23\textwidth]{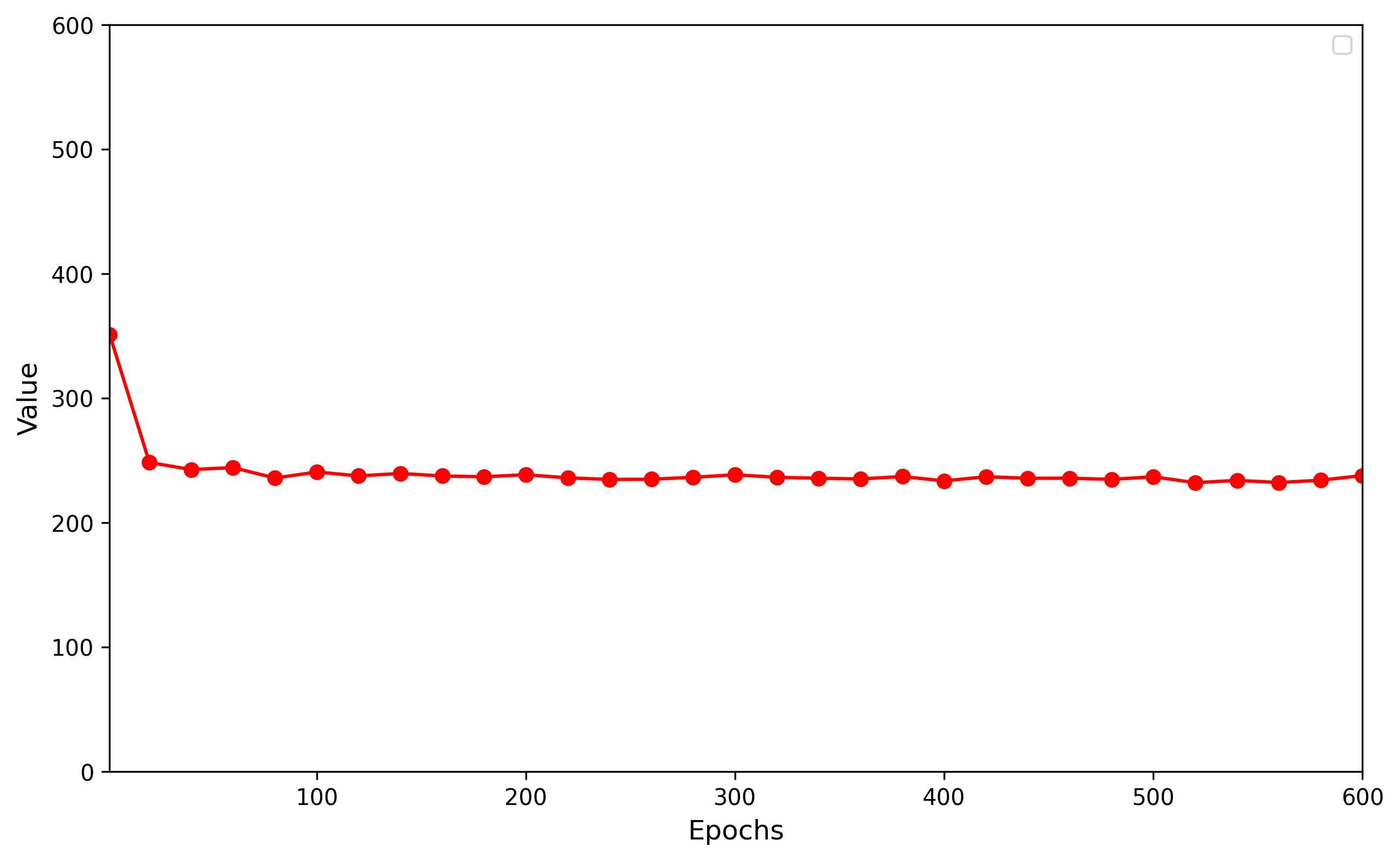}
    }
     \subfigure[Flowers]{
        \includegraphics[width=0.23\textwidth]{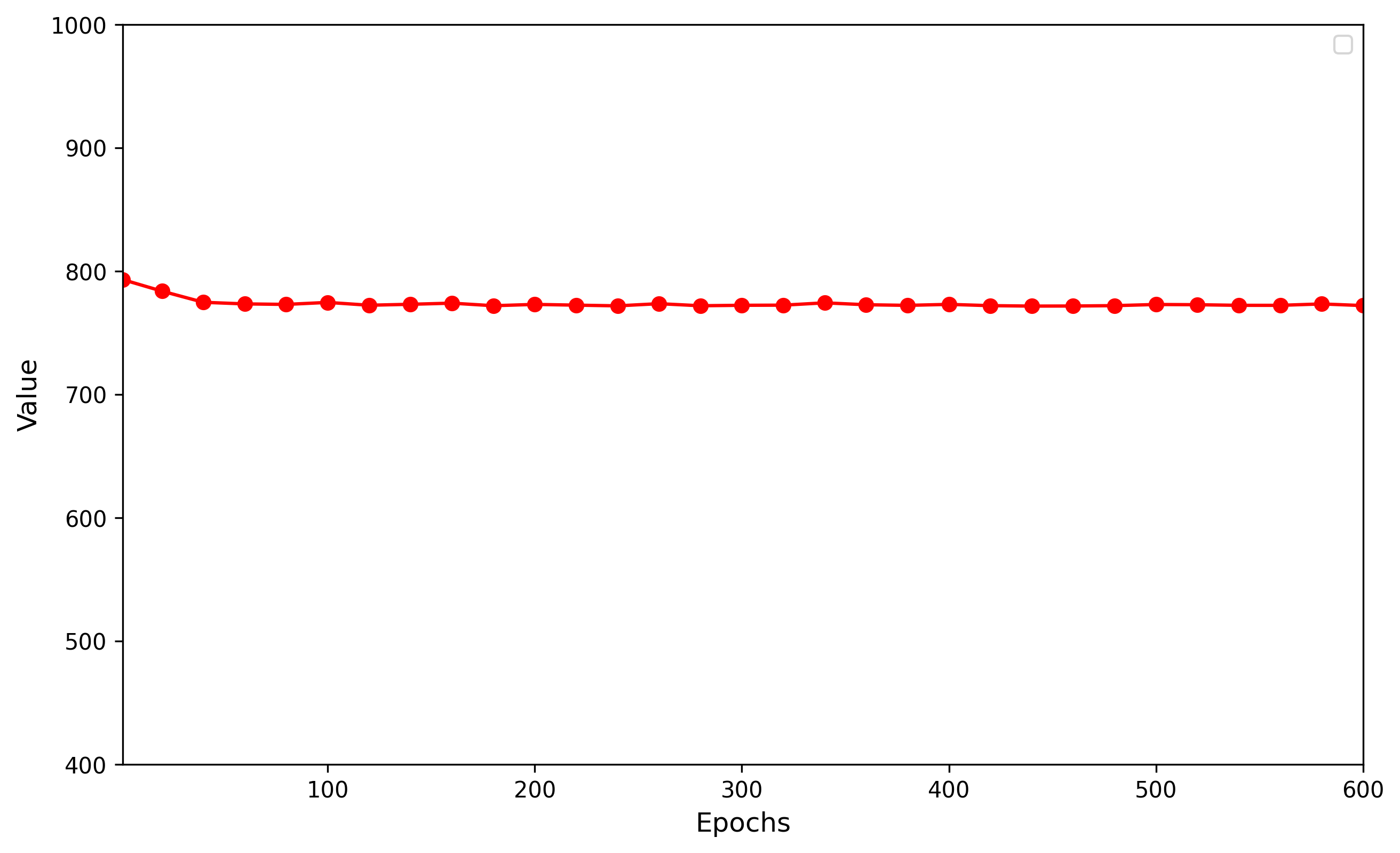}
    }
     \subfigure[Caltech101]{
        \includegraphics[width=0.23\textwidth]{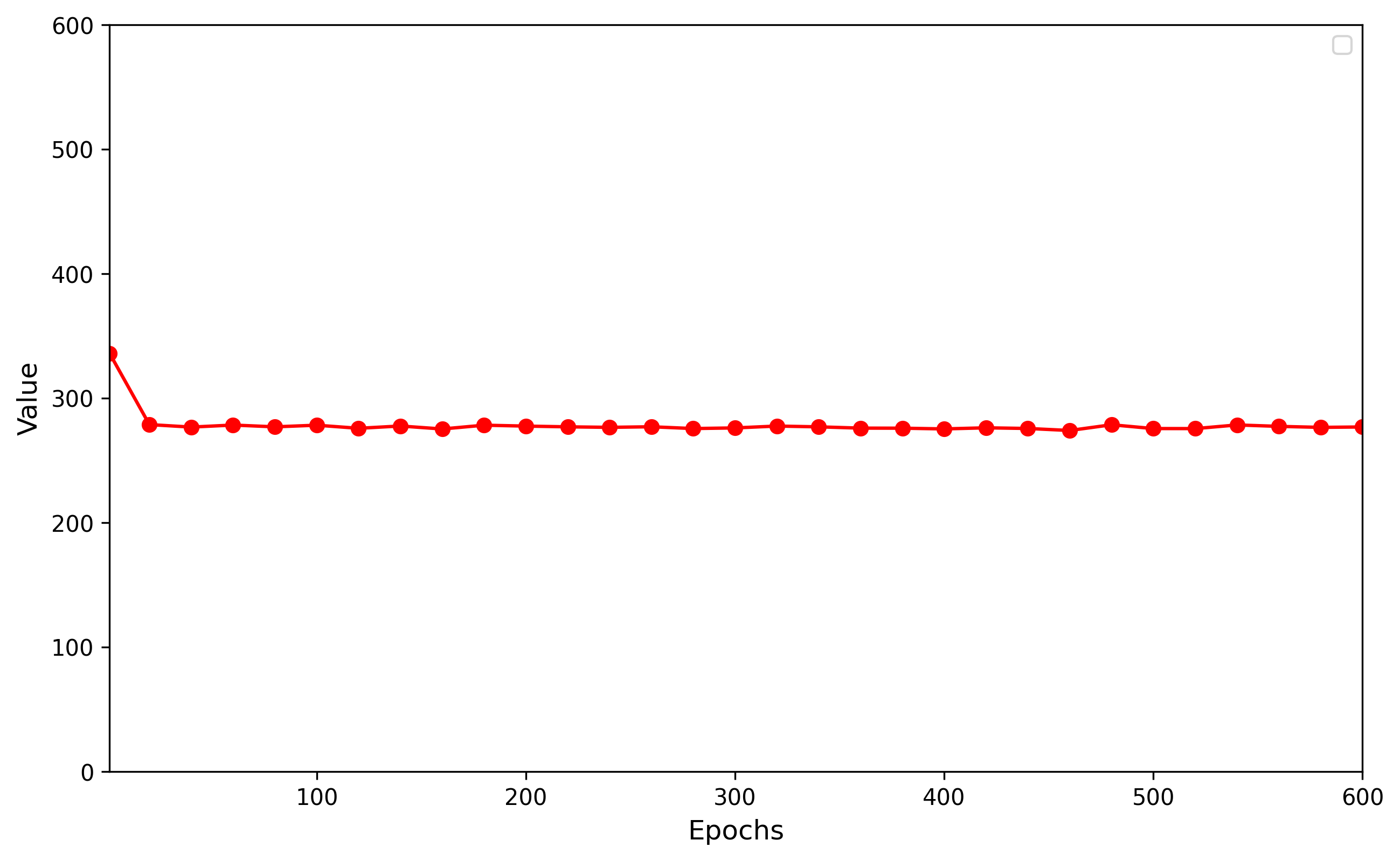}
    }

    \subfigure[EuroSAT]{
        \includegraphics[width=0.23\textwidth]{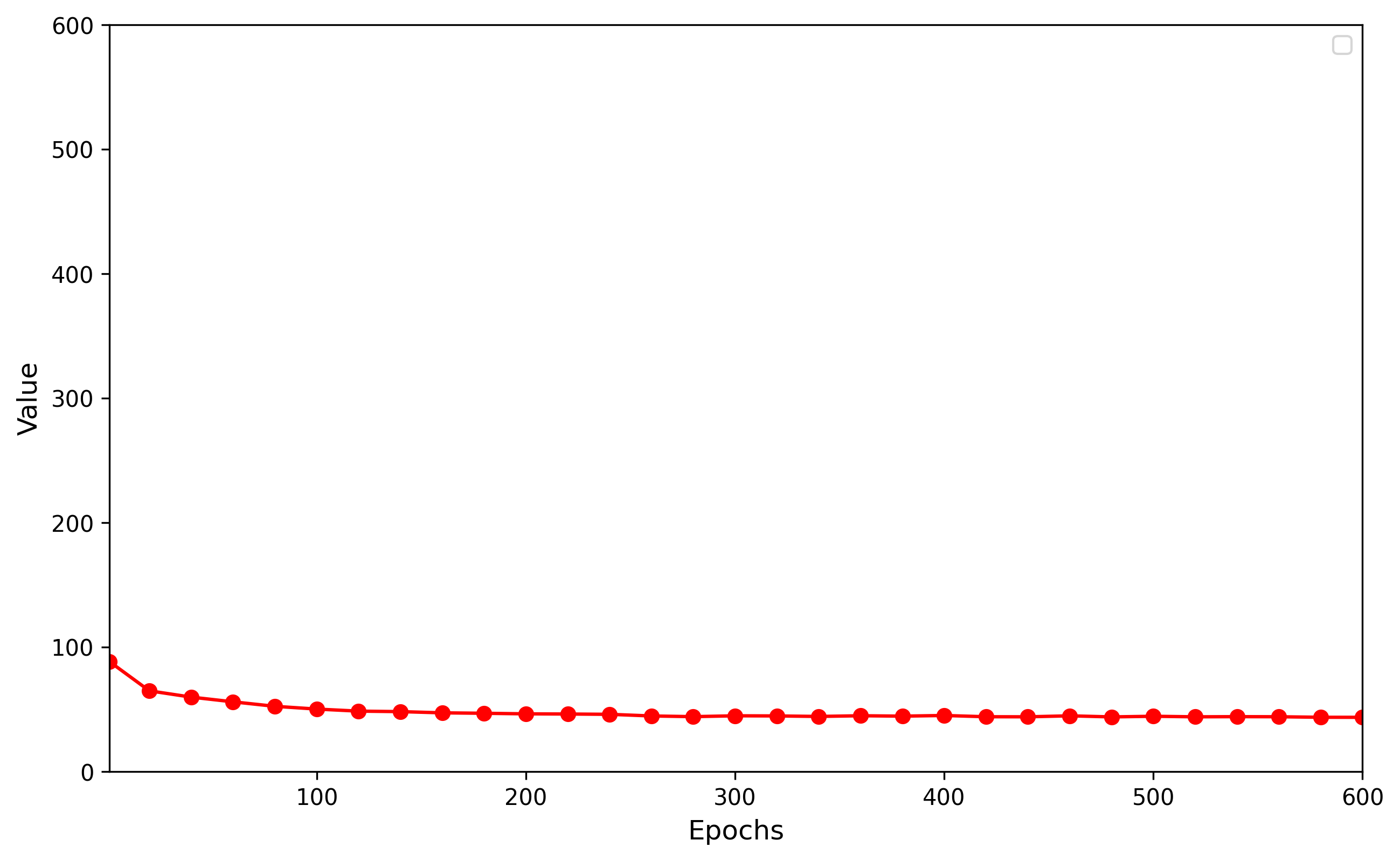}
    }
    \subfigure[SUN397]{
        \includegraphics[width=0.23\textwidth]{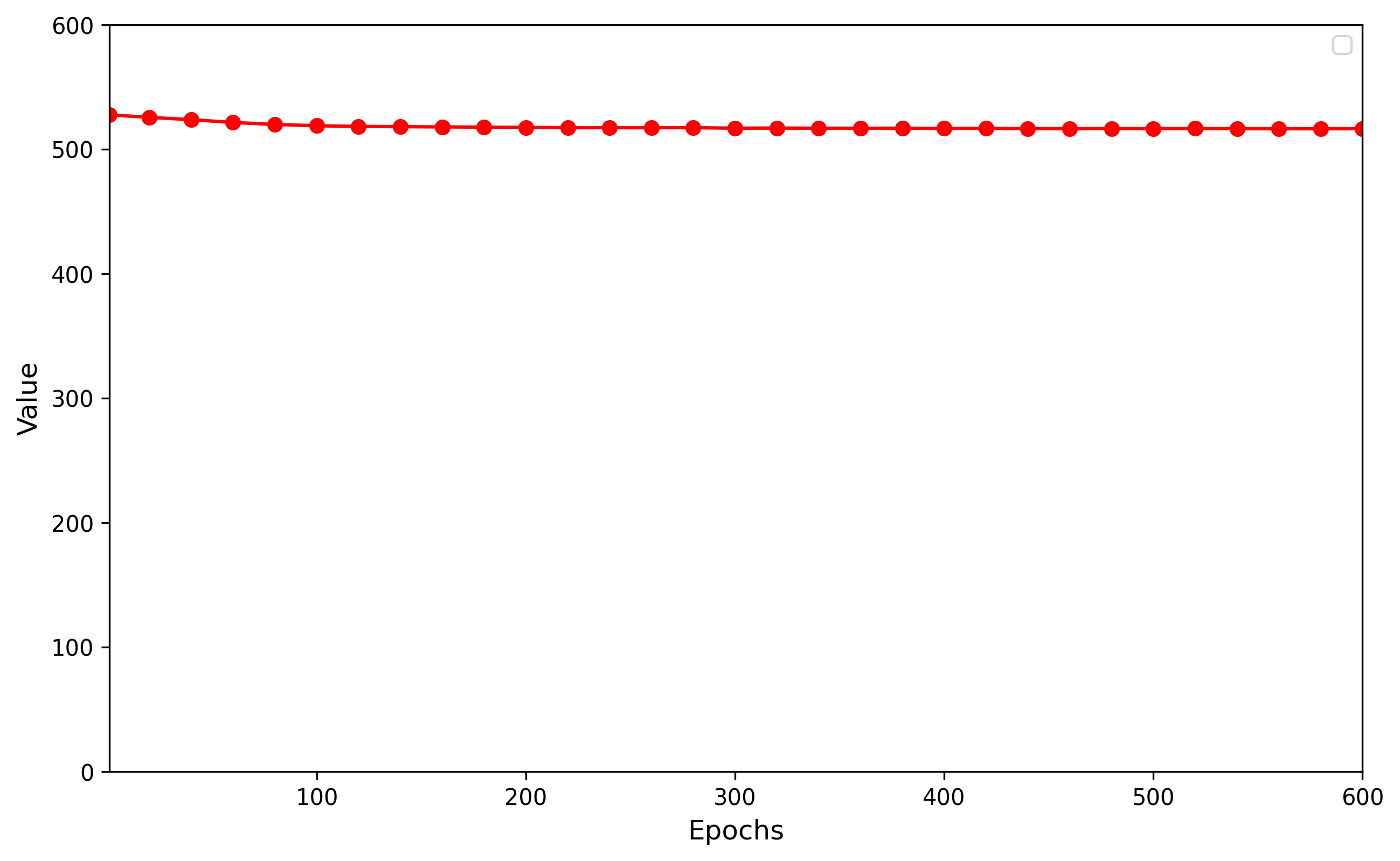}
    }
     \subfigure[Food101]{
        \includegraphics[width=0.23\textwidth]{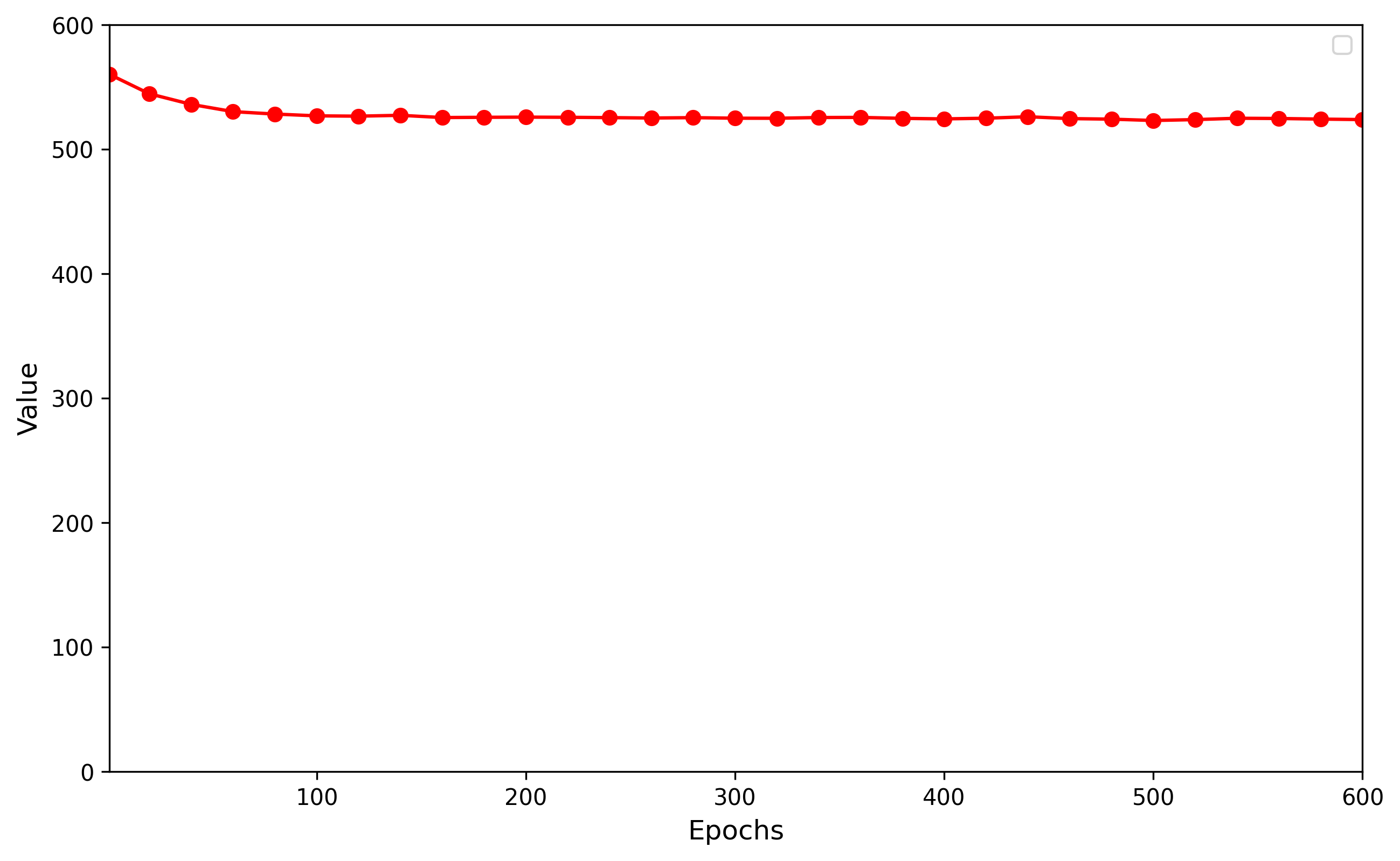}
    }
     \subfigure[ImageNet-1K]{
        \includegraphics[width=0.23\textwidth]{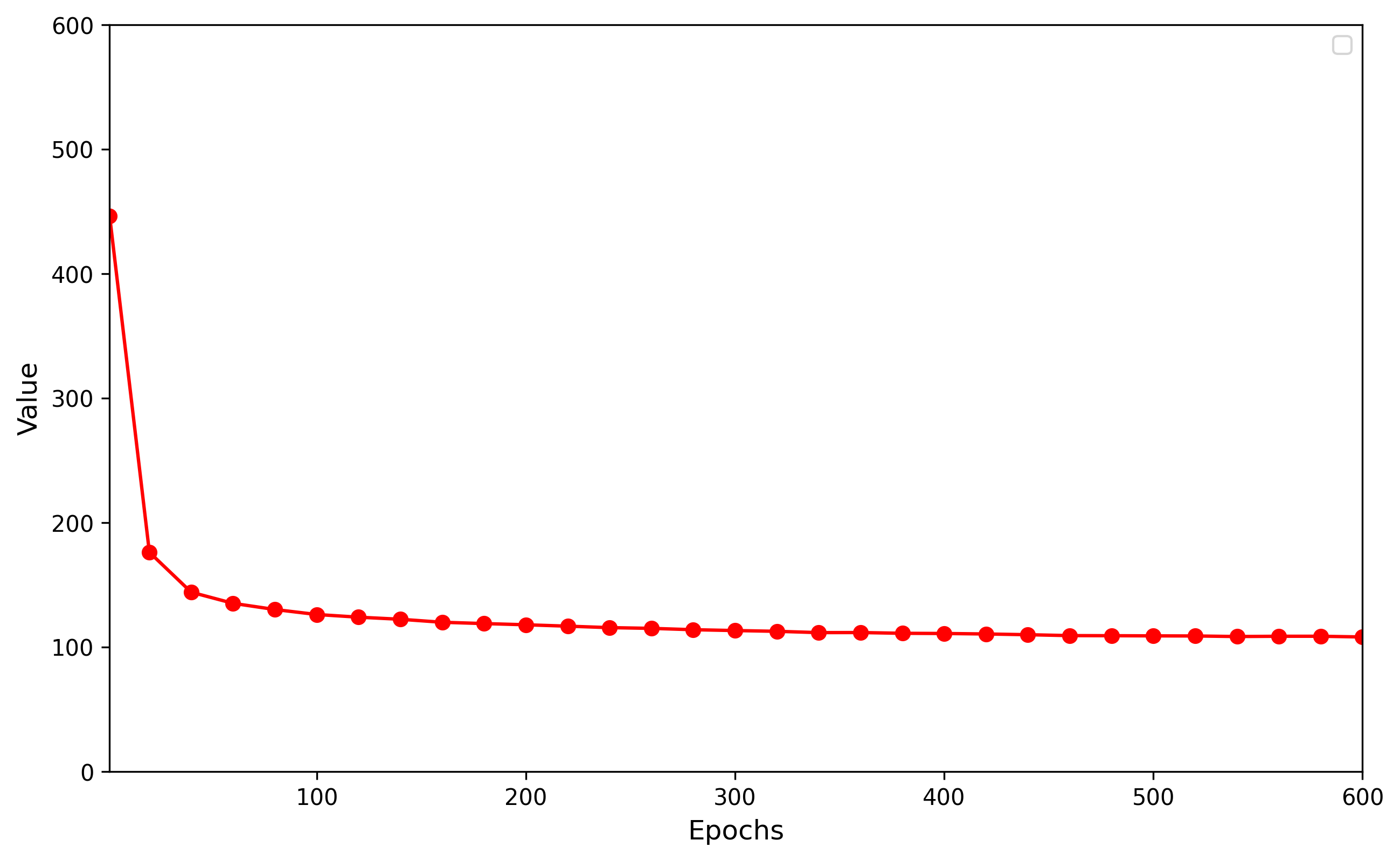}
    }
    \caption{Convergence results of Algorithm 1 on all the datasets.}
    \label{fig:convergence}
\end{figure*}

\subsubsection{Measuring Similarity Across Heterogeneous Views}
Different backbones exhibit different representation levels: transformer-based models (e.g., DINOv3, SigLIP 2 and CLIP ViT-L/14) typically produce global semantic representations, while convolutional models (e.g., ConvNeXt V2) capture more localized spatial features. We adopt centered kernel alignment (CKA) \cite{Kornblith2019SNNR} to measure the similarity between feature distributions produced by different pretrained model pairs. CKA summarizes the similarity between two feature distributions as a scalar value in $[0,1]$. A higher CKA score indicates that the feature distributions of a pair of sets tend to be identical. Table~\ref{tb:cka} shows the CKA scores for various pretrained model pairs. As expected, the CKA scores between transformer-based and convolutional architectures are relatively low, reflecting their fundamentally different feature distributions.

\subsubsection{Comparison of Training Times for Self-Supervised Learning}
To evaluate the training efficiency of the proposed SAGL method, we compare the computational costs of TURTLE, MSRL, and SAGL on the training sets of all eight datasets. For fair comparison, we report the computational cost of the competing methods that utilize two pretrained backbones. Fig.~\ref{fig:timecost} shows the training times of TURTLE, MSRL, and SAGL on eight datasets. MSRL runs faster than TURTLE and SAGL on most datasets. This is because it directly employs standard dense attention without considering intrinsic subspace structures embedded in the high-dimensional multiview data. As observed, SAGL achieves training efficiency compared to TURTLE on small and medium-scale datasets while maintaining consistently superior clustering performance. However, SAGL requires more training time on large-scale datasets such as ImageNet-1K. This is due to the structured sparse projection based on the $\alpha$-entmax transformation. This represents a reasonable computational trade-off for the significant performance gains achieved.

\subsubsection{Visualizations}
To evaluate the learned representations, we employ t-SNE \cite{Maaten2008TSNE} to visualize three levels of features on three representative datasets of varying scales: Pets, Caltech101, and Food101. Specifically, the three levels are: (1) the original features extracted from the two pretrained backbones, (2) the projected features after the linear transformation layers, and (3) the final representations constructed by SAGL via sparse information aggregation. Figs.~\ref{fig:sne:pets1}-\ref{fig:sne:food2} show the t-SNE visualizations across these three levels. We can observe that the learned representations exhibit significantly better class separability than those in the earlier stages.  These results indicate that SAGL produces more discriminative representations for improved clustering performance.

\subsubsection{Sparsity Analysis on Sparse Attention Graphs}
We first analyze the sparsity of the learned attention graphs during training. The sparsity ratio (SR) is defined as the number of nonzero elements in $\mathbf{A}^{(l)}$ divided by the total number of elements. Specifically, we examine the sparsity ratios of sparse attention graphs on the two representative datasets, Caltech101 and Food101, during training. Fig.~\ref{fig:sparsity:training:view} illustrates the evolution of sparsity ratios of the learned sparse attention graphs during training, where View 1 and View 2 correspond to features extracted from SigLIP 2 and DINOv3, respectively. We can observe that  the sparsity ratio drops rapidly in the first few epochs and fluctuates slightly during the remaining epochs.

To investigate the subspace-preserving property of the proposed SAGL model, we further examine subspace structures of the learned sparse attention graphs during testing. Fig.~\ref{fig:sparsity:testing:view} illustrates that the learned sparse attention graphs exhibit clear block-diagonal structures, where each diagonal block corresponds to a distinct latent semantic category. This block-diagonal pattern directly reflects the intrinsic subspace structures embedded in the high-dimensional multiview data. These results provide intuitive evidence that SAGL faithfully recovers the underlying subspace structures.

\subsubsection{Parameter Sensitivity Analysis}
By exploiting the sparse self-representation property of features, each representation is constructed as a sparse linear combination of spatially proximate neighbors. Consequently, the batch size plays an important role in determining the quality of the learned sparse attention graphs. To investigate the sensitivity of SAGL to batch size, we evaluate the clustering performance on three representative datasets of varying scales: Pets, SUN397, and Food101. Following the batch size selection strategy provided in Section ~\ref{sec:paramsetting}, the candidate batch sizes are set to $\{100, 500, 1{,}000\}$ for Pets and $\{2{,}000, 5{,}000, 10{,}000\}$  for SUN397 and Food101. The parameter $\gamma$ is selected from $\{5.0, 10.0, 20.0\}$, and $\beta$ is fixed at $1.0$. Figs.~\ref{fig:bs:pets}-\ref{fig:bs:food} show the clustering performance of SAGL under different batch sizes on the Pets, SUN397, and Food101 datasets, respectively. The clustering performance of SAGL remains relatively stable across different batch sizes, particularly on large-scale datasets such as SUN397 and Food101. This indicates that SAGL yields competitive clustering results across a wide range of batch sizes.

We further evaluate the sensitivity of the proposed SAGL model to hyperparameters $\gamma$ and $\beta$. Specifically, we investigate the clustering performance of SAGL under different combinations of hyperparameters $\gamma \in \{5.0, 10.0, 20.0\}$ and $\beta \in \{0.1, 0.2, 0.5, 1.0\}$. Figs.~\ref{fig:param:pets}-\ref{fig:param:food} illustrate the clustering performance of SAGL in terms of ACC, NMI, and ARI on the Pets, SUN397, and Food101 datasets. The proposed SAGL method performs steadily with relatively wide ranges of the parameters $\gamma$ and $\beta$. According to the sensitivity experiments, we empirically observe that setting $\gamma \in \{5.0, 10.0\}$ with $\beta = 1.0$ often yields satisfactory clustering results for SAGL.

\subsubsection{Convergence Analysis}
We empirically evaluate the convergence property of the proposed method across all eight datasets. Fig.~\ref{fig:convergence} shows the convergence curves of Algorithm 1, where the $x$-axis corresponds to iterations, and the $y$-axis represents the objective loss defined in Eq.~\eqref{eq:totalloss}. A positive constant is added to the $y$-axis values for better readability. As observed in the figures, the objective value drops dramatically within the first few dozen iterations and quickly converges to a steady state. These results empirically validate the rapid and stable convergence properties of the proposed method in practice.

%%%%%%%%%%%%%%%%%%%%%%%%%%%%%%%%%%%%%%%%%%%%%%%%%%%%%%%%%%%%

\end{document}